\documentclass[a4paper,12pt]{article}
\usepackage[normalem]{ulem}
\usepackage{pdfsync}
\usepackage{amsmath}
\usepackage{amsfonts}
\usepackage{amssymb}
\usepackage{amsthm} 
\usepackage{mathrsfs}
\newtheorem{proposition}{Proposition}[section]
\newtheorem{definition}[proposition]{Definition}

\newtheorem{theorem}[proposition]{Theorem}

\newtheorem{corollary}[proposition]{Corollary}

\newtheorem{lemma}[proposition]{Lemma}
\newtheorem*{lemma*}{Lemma}
\newtheorem*{proposition*}{Proposition}
\newtheorem*{theorem*}{Theorem}
\usepackage[normalem]{ulem}
\usepackage{indentfirst} 
\usepackage{extarrows}
\usepackage{xcolor}
\usepackage{caption}
\usepackage{subcaption}
\usepackage[inline]{enumitem}

\usepackage{tcolorbox}

\newcommand{\nleft}{}
\newcommand{\nright}{}

\DeclareMathOperator{\coef}{coef}

\newcommand{\Othres}{\mathcal{T}}

\newcommand{\lip}{H^1}

\newcommand{\Qlinfty}{R}
\newcommand{\Qapproximation}{A}
\newcommand{\Qradius}{\delta}
\newcommand{\Qnumeric}{\kappa}
\newcommand{\Qnc}{N}
\newcommand{\affine}{S}

\usepackage{mathtools}

\DeclareMathOperator*{\argmin}{argmin}

\newcommand{\nmathbf}[1]{#1}

\newcommand{\TODO}[1][0]{%
  \ifx#10
    $\square$
  \else
    $\boxtimes$
  \fi
}

\usepackage[linkcolor=blue,colorlinks=true,citecolor=blue,filecolor=black]{hyperref}

\usepackage{float}

\usepackage{tikz}
\tikzstyle{inarrow} = [<-,very thick]
\tikzstyle{bcircle} = [circle,draw = blue]

\usepackage{geometry}
\usepackage{setspace}

\usepackage[numbers, sort]{natbib}
\makeatletter
\renewcommand*\env@matrix[1][\arraystretch]{%
  \edef\arraystretch{#1}%
  \hskip -\arraycolsep
  \let\@ifnextchar\new@ifnextchar
  \array{*\c@MaxMatrixCols c}}
\makeatother

\usepackage{framed}
\newenvironment{customabstract}{
    \begin{center}
    \begin{minipage}{0.9\textwidth}
}{%
    \end{minipage}
    \end{center}
    \vspace{1cm}
}

\newcommand\blindfootnote[1]{%
  \begingroup
  \renewcommand\thefootnote{}\footnote{#1}%
  \addtocounter{footnote}{-1}%
  \endgroup
}

\title{Covering Numbers for Deep ReLU Networks with Applications to Function Approximation and Nonparametric Regression}
\author{Weigutian Ou \\ETH Z\"urich\\ wou@mins.ee.ethz.ch  \and Helmut Bölcskei \\ETH Z\"urich \\ hboelcskei@ethz.ch}
\date{}

\begin{document}

	\maketitle

\begin{customabstract}
\noindent \textbf{Abstract.}
Covering numbers of (deep) ReLU networks have been used to characterize
approximation-theoretic performance, to upper-bound prediction error in
nonparametric regression, and to quantify classification capacity. These
results rely on covering number upper bounds obtained via explicit
constructions of coverings. Lower bounds on covering numbers do not appear  
to be available in the literature. The present paper fills
this gap by deriving tight (up to multiplicative constants) lower and upper
bounds on the metric entropy (i.e., the logarithm of the covering numbers)
of fully connected networks with bounded weights, sparse networks with
bounded weights, and fully connected networks with quantized weights.
The tightness of these bounds yields a fundamental understanding of the
impact of sparsity, quantization, bounded versus unbounded weights, and
network output truncation. Moreover, the bounds allow one to characterize
fundamental limits of neural network transformation, including network
compression, and lead to sharp upper bounds on the prediction error in
nonparametric regression through deep networks. In particular, we remove
a $\log^6(n)$-factor from the best known sample complexity rate for
estimating Lipschitz functions via deep networks, thereby establishing
optimality. Finally, we identify a systematic relation between optimal
nonparametric regression and optimal approximation through deep networks,
unifying numerous results in the literature and revealing underlying
general principles.
\end{customabstract}

\noindent\small\textit{Communicated by Felix Voigtlaender}

\vspace{1em}


\vspace{1em}


\section{Introduction}
\label{sec:introduction}

    It\blindfootnote{\noindent H. B\"olcskei gratefully acknowledges support by the Lagrange Mathematics and
Computing Research Center, Paris, France.}
is well known that neural networks exhibit universal approximation properties \cite{Hornik1989, Cybenko1989, Funahashi1989, Barron1993}, but these results typically require infinitely large, specifically infinitely wide, networks. Neural networks employed in practice are, however, subject to constraints on width, depth, weight magnitude and precision, and connectivity (i.e., the number of nonzero weights). To characterize the performance limits
    of neural networks under such constraints, it is necessary to quantify the complexity of the function classes they realize. This is typically done through two widely used complexity notions, namely Vapnik-Chervonenkis (VC) dimension \cite{VAPNIK1971} and covering numbers \cite{wainwright2019high}.

    The VC dimension finds application in the characterization of (i) the approximation-theoretic limits of neural networks with the ReLU activation function, see e.g. \cite{shen2019deep}, hereafter referred to as ReLU networks, and (ii) the prediction error incurred in nonparametric regression through ReLU networks, see e.g. \cite{Kohler2019}.
    Nearly-tight bounds on the VC dimension of ReLU networks were reported in \cite{bartlett2019nearly}, specifically upper and lower bounds differing only by
    a multiplicative factor of order lower than that of the upper and the lower bound\footnote{For networks with piecewise polynomial activation functions of degree at least $2$, there is still a large gap between the VC dimension upper and lower bounds available in the literature.}.
 
Covering numbers
have  been used to characterize the approximation-theoretic limits of ReLU networks \cite{deep-it-2019, yarotsky2019phase, Caragea2022}, upper-bound the prediction error they incur in nonparametric regression \cite{Schmidt-Hieber2017, Chen2019}, and quantify their classification capacity \cite{Caragea2023, Anthony1999, Bartlett2017}.
These analyses typically construct coverings by quantizing the network weights to a precision commensurate with the desired covering ball radius. The cardinality of the resulting coverings then provides upper bounds on the covering number.
Corresponding explicit lower bounds are, to the best of our knowledge, not available in the literature.

	The contributions of the present paper can be organized along three main threads. The first one revolves around explicit lower bounds on the covering number of fully-connected ReLU networks with uniformly bounded weights. 
 In particular, these bounds are shown, by way of establishing matching upper bounds, to be tight in terms of scaling behavior. 
 The techniques we devise to derive the bounds are novel and partly rely on results recently reported by the authors of the present paper in \cite{FirstDraft2022}. 
	
    The second thread of contributions illustrates, by way of application scenarios, what is made possible by the tightness of the covering number bounds identified before.
    The first scenario is concerned with the fundamental limits of neural network transformation, concretely, the approximation of a given class of networks by another class that is subject to different constraints.  This includes applications such as network quantization \cite{elbrachter2018dnn} and network compression \cite{Huyen2022}. In the second scenario, we consider the fundamental limits of function approximation through ReLU networks.
    A novel minimax error upper bound, interesting in its own right, is shown to lead to sharp upper bounds on the prediction error in nonparametric regression through ReLU networks. This result also allows us to uncover a systematic relation between optimal nonparametric regression and optimal approximation through (deep) ReLU networks thereby unifying numerous corresponding results in the literature \cite{Schmidt-Hieber2017, Chen2019, nakada2020adaptive} and identifying general underlying principles. 
    In all cases considered, we either improve upon best known results in the literature or fill gaps in available theories.

    Our third objective is to establish tight covering number bounds for sparse (in terms of connectivity) networks with bounded weights and for fully-connected networks with quantized weights.
    We also provide an upper bound on the covering number of fully-connected networks with unbounded weights and truncated outputs.
    These three choices are motivated by their prevalence in theoretical analyses and practical applications, see e.g. \cite{YAROTSKY2017103, deep-approx-18, deep-it-2019, Schmidt-Hieber2017, Chen2019, nakada2020adaptive, GUHRING2021107, PETERSEN2018296, Huyen2022}. 

	The remainder of the paper is organized as follows. Frequently used definitions are provided at the end of this section, while basic notation and further definitions are listed in Appendix~\ref{sec:notation}. In Section~\ref{sec:fully_connected_relu_networks_with_bounded_weights}, we present our results on the covering number of fully-connected ReLU networks with uniformly bounded weights. Sections~\ref{sub:fundamental_limit_covering_number} and \ref{sub:empirical_risk_minimizaiton} discuss the application of our covering number bounds to neural network transformation, function approximation, and nonparametric regression.
	Sections~\ref{sec:sparsely_connected_relu_networks}-\ref{sec:covering_number_for_relu_networks_with_bounded_output} report the covering number bounds for sparse networks with uniformly bounded weights, fully-connected networks with quantized weights, and fully-connected networks with truncated outputs, respectively.

\subsection{Important Definitions}

    We start with the definition of ReLU networks.

	\begin{definition}
	\label{def:ReLU_networks}
		Let $L,N_0,N_1,\dots, N_L \in \mathbb{N}$. A network configuration $\Phi$ is a sequence of matrix-vector tuples 
				\begin{equation*}
					\Phi = (( A_i,\nmathbf{b}_i )  )_{i = 1}^{L},
				\end{equation*}
				with $A_i \in \mathbb{R}^{N_i \times N_{i - 1}}$, $b_i \in \mathbb{R}^{N_i}$, $i = 1,\dots, L$. We refer to $N_i$ as the width of the $i$-th layer, $i = 0,\dots, L$, and call the tuple $( N_0,\dots,N_L )$ the architecture of the network configuration. $\mathcal{N} ( d )$ denotes the set of all network configurations with input dimension $N_0 = d$ and output dimension $N_L = 1$. The depth of the configuration $\Phi$ is $\mathcal{L} ( \Phi )  := L$, its width $\mathcal{W}  ( \Phi ) := \max_{i = 0,\dots,L} N_i$, its  weight set $\coef ( \Phi ) := \bigcup_{i = 1,\dots,L} ( \coef ( A_i ) \bigcup \coef ( \nmathbf{b}_i )  )  $, with $\coef ( A )$ and $\coef ( \nmathbf{b} )$ denoting the value set of the entries of $A$ and $\nmathbf{b}$, respectively, its weight magnitude $\mathcal{B} ( \Phi )  := \max_{i = 1,\dots, L} \max \{ \| A_i \|_\infty, \| \nmathbf{b}_i \|_\infty \}$, and its connectivity $\mathcal{M} ( \Phi ) := \sum_{\ell = 1}^L  ( \nleft\| A_\ell \nright\|_0 + \nleft\| b_\ell \nright\|_0 )$.

				We define, recursively, the network realization $R (\Phi ): \mathbb{R}^{N_0} \rightarrow \mathbb{R}^{N_L}$,  associated with the network configuration $\Phi$, according to

				\begin{equation}
				\label{eq:realization_computation}
				R ( \Phi ) = \left\{
				\begin{aligned}
					& \affine ( A_L, b_L ), &&\text{if } L = 1,\\
					& \affine ( A_L, b_L ) \circ \rho \circ R ( (( A_i,\nmathbf{b}_i )  )_{i = 1}^{L - 1} ), &&\text{if } L \geq 2,
				\end{aligned}
				\right.
				\end{equation}
				where $S(A,b)$ is the affine mapping $S(A,b) (x) = Ax + b, x \in \mathbb{R}^{n_2}$, with $A \in \mathbb{R}^{n_1 \times n_2}$, $b \in \mathbb{R}^{n_1}$, and $\rho(x):= \max \{ x,0 \}$, for $x \in \mathbb{R}$, is the ReLU activation function, which, 
    when applied to vectors, acts elementwise.

				The family of network configurations with depth at most  $L$, width at most $W$, weight magnitude at most $B$, where $B \in \mathbb{R}_+ \cup \{ \infty \}$,  connectivity at most $s$, weights taking values in $\mathbb{A} \subseteq \mathbb{R}$, $d$-dimensional input, and $1$-dimensional output, for $d \in \mathbb{N}$, $W, L, s \in \mathbb{N} \cup \{ \infty \}$, with\footnote{The condition $W \geq d$ is formally stated here so as to prevent the trivial case of $\mathcal{N}_\mathbb{A} ( d, W,L,B,s )$ being an empty set. It will be a standing assumption throughout the paper.}
		        $W \geq d$, is denoted as
				\begin{equation*}
					\mathcal{N}_\mathbb{A} ( d, W,L,B, s )  =  \{  \Phi \in \mathcal{N} ( d ) :\mathcal{W} (  \Phi ) \leq W, \ \mathcal{L} ( \Phi ) \leq L,\ \mathcal{B} ( \Phi ) \leq B, \mathcal{M} ( \Phi ) \leq s, \coef ( \Phi ) \subseteq \mathbb{A} \},
				\end{equation*}
				with the family of associated network realizations
				\begin{equation}
					\mathcal{R}_\mathbb{A} ( d, W,L,B, s ) := \{ R ( \Phi ) \, : \Phi \in \, \mathcal{N}_\mathbb{A} ( d, W,L,B, s ) \}. \label{eqline:continuous_nn_realization}
				\end{equation}

					To simplify notation, for $\mathbb{A} = \mathbb{R}$, we allow omission of the argument $\mathbb{A}$ in $\mathcal{N}_\mathbb{A} ( d, W,L,B, s )$ and $\mathcal{R}_\mathbb{A} ( d, W,L,B, s )$. When $s = \infty$, we allow omission of the argument $s$ in $\mathcal{N}_\mathbb{A} ( d, W,L,B, s )$ and $\mathcal{R}_\mathbb{A} ( d, W,\allowbreak L,B, s )$. 
     Furthermore, we allow omission of both arguments $B,s$ in $\mathcal{N}_\mathbb{A} ( d, W,L,B, s )$ and $\mathcal{R}_\mathbb{A} ( d, W,L,B, s )$ when $B = s = \infty$. One specific, frequently used, incarnation of these policies is $\mathcal{N} ( d, W,L ) = \mathcal{N}_\mathbb{R} ( d, W,L, \infty , \infty)$ and $\mathcal{R} ( d, W,L ) = \mathcal{R}_\mathbb{R} ( d, W,L, \infty, \infty )$. 
            \end{definition}	
            
            We emphasize the importance of differentiating between network configurations and network realizations. Different network configurations may result in the same realization. Nevertheless, whenever there is no potential for confusion, we shall use the term network to collectively refer to both configurations and realizations.
		
Throughout the paper, we shall frequently use the covering number and the packing number, defined as follows.

    \begin{definition}
		[Covering number and packing number] \cite[Definitions 5.1 and 5.4]{wainwright2019high}
		Let $(\mathcal{Y}, \delta)$ be a metric space. An $\varepsilon$-covering of $\mathcal{X} \subseteq \mathcal{Y}$ is a subset $\{ x_1, \dots,x_n \}$ of $\mathcal{X}$ such that for all $x \in \mathcal{X}$, there exists an $i \in \{ 1,\dots,n \}$ so that $\delta ( x,x_i )\leq \varepsilon$.  The $\varepsilon$-covering number $N (\varepsilon,\mathcal{X}, \delta )$ is the cardinality of a smallest $\varepsilon$-covering of $\mathcal{X}$. An $\varepsilon$-packing of $\mathcal{X}$ is a subset $\{ x_1, \dots,x_n \}$ of $\mathcal{X}$ such that $\delta ( x_i,x_j ) > \varepsilon$, for all $i,j \in \{ 1,\dots, n \}$ with $i\neq j$. The $\varepsilon$-packing number $M ( \varepsilon, \mathcal{X}, \delta ) $ is the cardinality of a largest $\varepsilon$-packing of $\mathcal{X}$.
    \label{def:covering_packing}
    \end{definition}
    To simplify notation, when $\delta$ is the $L^p(\mathbb{X} )$-norm, with $\mathbb{X} \subseteq \mathbb{R}^d$ and $p \in [1,\infty]$, we may write  $N ( \varepsilon, \mathcal{F}, L^p ( \mathbb{X} ) ) :=  N ( \varepsilon, \mathcal{F}, \nleft\| \cdot \nright\|_{L^p ( \mathbb{X} ) }) $. Moreover, we shall use $N ( \varepsilon, \mathcal{F}, L^2 ( P ) ) :=  N ( \varepsilon, \mathcal{F}, \nleft\| \cdot \nright\|_{L^2 ( P )}) $, for $P$ a distribution on $\mathbb{X}$. The same conventions apply to the packing number.

\section{\raggedright Fully-connected ReLU Networks with Uniformly  Bounded Weights}
\label{sec:fully_connected_relu_networks_with_bounded_weights}

Our covering number bounds for fully-connected ReLU networks with uniformly bounded weights are as follows.

	\begin{theorem}
			\label{thm:covering_number_upper_bound_fully_connected_bounded_weight}
			Let $p \in [1,\infty], d,W,L \in \mathbb{N}$, $B,\varepsilon \in \mathbb{R}_+$ with $B \geq 1$ and $\varepsilon \in (0,1\slash 2)$. We have
			\begin{equation}
				\label{eq:upper_bound_fully_connected_bounded_output}
				\log ( N(\varepsilon,\mathcal{R}(d,W,L,B),{L^p ( [0,1]^d )} ) ) \leq C W^2 L \log \biggl( \frac{(W+1)^L B^{L}}{\varepsilon} \biggr),
			\end{equation}
			where $C \in \mathbb{R}_+$ is an absolute constant. Further, if, in addition, $W, L \geq 60$, then 
			\begin{equation}
                \label{eq:lower_bound_fully_connected_bounded_output_main}
				\log ( N(\varepsilon,\mathcal{R}(d,W,L,B),{L^p ( [0,1]^d )} ) ) \geq c \, W^2 L \log \biggl( \frac{(W+1)^L B^{L}}{\varepsilon} \biggr),
			\end{equation}
			where $c\in \mathbb{R}_+$ is an absolute constant.
			\begin{proof}
				The proofs of the upper bound and the lower bound are provided in Sections~\ref{sub:covering_number_upper_bound} and \ref{sub:covering_number_lower_bound}, respectively.
			\end{proof}
		\end{theorem}

			We remark that, for $W,L \geq 60$ and $\varepsilon \in (0, 1 \slash 2)$, the upper bound \eqref{eq:upper_bound_fully_connected_bounded_output} and the lower bound \eqref{eq:lower_bound_fully_connected_bounded_output_main} differ only by the multiplicative absolute constants $C,c$ to be specified in the proof.  
			These constants as well as the condition $W,L \geq 60$ are chosen for expositional convenience of the proof; improvements are possible, but will not be pursued here.

            	The covering number upper bound in \eqref{eq:upper_bound_fully_connected_bounded_output} is not conceptually new. Results of this type are  available in the literature for specific parameter choices, notably for $B=1$ in \cite[Lemma 5]{Schmidt-Hieber2017} and for $p = \infty$ in 
                \cite[Lemma 5.3]{Chen2019}. The extension to general $B \in [1,\infty)$ and $p \in [1,\infty]$ follows by standard scaling arguments and norm domination. We nevertheless include a full proof in Section~\ref{sub:covering_number_upper_bound} for self-containment and to keep all constants, normalizations, and notation explicit and consistent with the remainder of the paper. Retaining the full argument also allows the constants to be tracked explicitly; in particular, one may take $C=30$ in \eqref{eq:upper_bound_fully_connected_bounded_output}. No ideas beyond those contained in the cited references are introduced in this proof.

	\subsection{Proof of the Upper Bound  in Theorem~\ref{thm:covering_number_upper_bound_fully_connected_bounded_weight}}
	\label{sub:covering_number_upper_bound}

		The proof is effected by constructing an explicit $\varepsilon$-covering of $\mathcal{R}(d,W,L,B)$ with elements in $\mathcal{R}_{[-B,B] \cap 2^{-b} \mathbb{Z}}(d,W,L)$, where $b \in \mathbb{N}$ is a parameter suitably depending on $\varepsilon$. We start with three technical  lemmata, and then provide the proof of the upper bound at the end of the section. The first lemma quantifies the distance between the realizations of two networks sharing the same architecture.

		\begin{lemma}
			\label{lem:quantization}
                \cite[Lemma E.1]{FirstDraft2022}
			Let $d,W,L, \ell \in \mathbb{N}$ with $\ell \leq L$, $B \in \mathbb{R}_+$ with $B \geq 1$, and let 
			\begin{equation*}
				\Phi^i = ( ( A_j^i,\nmathbf{b}_j^i ) )_{j =1 }^{\ell} \in \mathcal{N} ( d, W, L, B ), \quad i = 1,2,
			\end{equation*}
			have the same architecture.
			Then,
			\begin{equation}
			\label{eq:quantization_error_general_dimension_000}
				\|  R ( \Phi^1 )  -  R ( \Phi^2 )  \|_{L^\infty ( [0,1]^d )} \leq L (W+1)^L B^{L-1} \| \Phi^1 - \Phi^2 \|,
			\end{equation}
			where 
			\begin{equation}
			\label{eq:weightwise_difference}
				\| \Phi^1 - \Phi^2 \| := \max_{j = 1,\dots, \ell} \max \bigl\{ \| A_j^1 - A_j^2\|_\infty,  \| \nmathbf{b}_j^1 - \nmathbf{b}_j^2\|_\infty  \bigr\} .
			\end{equation}
		\end{lemma}

		Based on Lemma~\ref{lem:quantization}, we now construct the announced $\varepsilon$-covering of
  $\mathcal{R}(d,W,L,B)$.

		\begin{lemma}
			\label{lem:define_covering_fully_connected}
			Let $p \in [1,\infty]$, $d,W,L,b \in \mathbb{N}$, and $B \in \mathbb{R}_+$ with $B \geq 1$. Then, the set $\mathcal{R}_{[-B,B] \cap 2^{-b} \mathbb{Z}}(d,W,L)$ is an $(L (W+1)^L B^{L-1} 2^{-b})$-covering of $\mathcal{R}(d,W,L,B)$ with respect to the $L^p ( [0,1]^d )$-norm.

			\begin{proof}
				Define the quantization mapping  $q_b: [-B,B] \rightarrow [-B,B] \cap 2^{-b} \mathbb{Z} $ as
                    {
                    \begin{equation*}
					q_{b} ( x ) = \left\{
					\begin{aligned}
						2^{-b} \lfloor 2^b x \rfloor, &\quad  \text{ for } x \in [0,B],\\
						2^{-b} \lceil 2^b x \rceil, & \quad  \text{ for } x \in [-B,0),
					\end{aligned}
                    \right.   
				\end{equation*}
                    }
				and note that $| x - q_b ( x ) | \leq 2^{-b} $, for all $x \in [-B , B]$. When applied to matrices or vectors, $q_b ( \cdot )$ acts elementwise.  Now, arbitrarily fix $R ( \Phi ) \in \mathcal{R}(d,W,L,B)$ with $ \Phi  =  (( A_\ell,\nmathbf{b}_\ell )  )_{\ell = 1}^{\tilde{L}} \in \mathcal{N}(d,W,L,B)$ and $\tilde{L} \leq L$, and quantize the weights of $\Phi$ according to
				\begin{equation*}
					Q_b ( \Phi ) = (( q_b ( A_\ell ),q_b ( b_\ell ) )  )_{\ell = 1}^{\tilde{L}} \in \mathcal{N}_{[-B,B] \cap 2^{-b} \mathbb{Z}}(d,W,L,B).
				\end{equation*}
				We then have 
				\begin{equation*}
					\| \Phi - Q_b ( \Phi ) \| = \max_{\ell = 1, \dots,\tilde{L} } \max \bigl\{ \| A_\ell - q_b ( A_\ell )\|_\infty,  \| b_\ell - q_b ( b_\ell )\|_\infty  \bigr\}  \leq 2^{-b},
				\end{equation*}
				which, together with Lemma~\ref{lem:quantization}, yields 
				\begin{equation}
				\label{eq:proof_bounded_fully_1}
					\|  R ( \Phi )  -  R ( Q_b ( \Phi ) )  \|_{L^\infty ( [0,1]^d )}  \leq L (W+1)^L B^{L-1} 2^{-b}.
    \end{equation}
				As 
    $$
    \|  R ( \Phi )  -  R ( Q_b ( \Phi ) )  \|_{L^p ( [0,1]^d )} \leq \sup_{x \in [0,1]^d} |  R ( \Phi ) (x) -  R ( Q_b ( \Phi ) ) (x)  | = \|  R ( \Phi )  -  R ( Q_b ( \Phi ) )  \|_{L^\infty ( [0,1]^d )},
    $$ 
    \eqref{eq:proof_bounded_fully_1} implies that
				\begin{equation}
				\label{eq:proof_bounded_fully_20}
					\|  R ( \Phi )  -  R ( Q_b ( \Phi ) )  \|_{L^p ( [0,1]^d )} \leq L (W+1)^L B^{L-1} 2^{-b}.
				\end{equation}
				We can therefore conclude that $\mathcal{R}_{[-B,B] \cap 2^{-b} \mathbb{Z}}(d,W,L)$ is an $(L (W+1)^L B^{L-1} 2^{-b})$-covering of $\mathcal{R}(d,W,L,B)$ in the $L^p([0,1]^d )$-norm. 
			\end{proof}
		\end{lemma}

		It remains to upper-bound the cardinality of the covering $\mathcal{R}_{[-B,B] \cap 2^{-b} \mathbb{Z}}(d,W,L)$ identified in Lemma~\ref{lem:define_covering_fully_connected}. To this end, we first state an auxiliary result from \cite{FirstDraft2022}.
		\begin{lemma}
			\cite[Proposition 2.4]{FirstDraft2022}
			\label{lem:counting_cardinality}
			For  $d ,W,L \in \mathbb{N}$ and a finite set $\mathbb{A} \subseteq \mathbb{R}$ with $| \mathbb{A} | \geq 2$, it holds that
		 	\begin{equation}
				\label{eq:cardinality_realization}
				\log  (| \mathcal{R}_\mathbb{A} ( d,W,L ) |) \leq \log  (| \mathcal{N}_\mathbb{A} ( d,W,L ) |)   \leq 5 W^2 L \log  (| \mathbb{A} |).
			\end{equation}
		\end{lemma}

		We next make the choice of $b$ explicit. Specifically, we set
			\begin{equation}
				 b := \biggl \lceil \log \biggl( \frac{L (W+1)^L B^{L-1}}{\varepsilon} \biggr) \biggr\rceil.
			\end{equation}
			Noting that $L (W+1)^L B^{L-1} 2^{-b} \leq \varepsilon$, it follows from Lemma~\ref{lem:define_covering_fully_connected} that $\mathcal{R}_{[-B,B] \cap 2^{-b} \mathbb{Z}}(d,W,L)$ is an $\varepsilon$-covering of $\mathcal{R}(d,W,L,B)$ with respect to the $L^p ( [0,1]^d )$-norm. By minimality of the covering number, we have
			\begin{equation}
			\label{eq:using_the_covering}
				 N ( \varepsilon, \mathcal{R}(d,W,L,B), L^p ( [0,1]^d ) )  \leq | \mathcal{R}_{[-B,B] \cap 2^{-b} \mathbb{Z}}(d,W,L) |.
			\end{equation}
			Application of Lemma~\ref{lem:counting_cardinality} yields an upper bound on the cardinality of the covering according to
			\begin{align}
				\log  (\nleft|  \mathcal{R}_{[-B,B] \cap 2^{-b} \mathbb{Z}}(d,W,L) |) \leq&\,   5 W^2 L \log (\nleft| [-B,B] \cap 2^{-b} \mathbb{Z} \nright| ).
			\end{align}
			The term $\log ( \nleft| [-B,B] \cap 2^{-b} \mathbb{Z} \nright|  )$ can now be bounded as follows
			\begin{align}
				\log ( \nleft| [-B,B] \cap 2^{-b} \mathbb{Z} \nright|  ) =  &\,\log ( \nleft| [-2^{b}B,2^{b}B] \cap  \mathbb{Z} \nright|  ) \label{eq:ccw_1}\\
				\leq &\,\log ( \lfloor 2 \cdot 2^{b}B + 1 \rfloor  ) \label{eq:ccw_2}\\
				\leq &\,\log ( 4 \cdot 2^{b}B   ) \label{eq:ccw_3}\\
				= &\, 2 + \biggl \lceil \log \biggl( \frac{L (W+1)^L B^{L-1}}{\varepsilon} \biggr) \biggr\rceil + \log (B) \label{eq:ccw_4}\\
				\leq &\, 3 + \log \biggl( \frac{L (W+1)^L B^{L-1}}{\varepsilon} \biggr)  + \log (B) \label{eq:ccw_5}\\
				\leq &\, 3 \log \biggl( \frac{L (W+1)^L B^{L}}{\varepsilon} \biggr), \label{eq:ccw_6}
			\end{align}
   where \eqref{eq:ccw_3} is by $1 \leq 2\cdot 2^b B$ and in \eqref{eq:ccw_6} we used $3 \leq 2 \log \bigl( \frac{L (W+1)^L B^{L}}{\varepsilon} \bigr)$ owing to $\varepsilon \in (0,1/2)$. 
   Putting \eqref{eq:using_the_covering}-\eqref{eq:ccw_6} together, yields
			\begin{align}
				\log ( N ( \varepsilon, \mathcal{R}(d,W,L,B), L^p ( [0,1]^d ) ) )
				\leq&\, 15 W^2 L \log \biggl( \frac{L (W+1)^L B^{L}}{\varepsilon} \biggr) \\
				\leq&\, 30 W^2 L \log \biggl( \frac{(W+1)^L B^{L}}{\varepsilon} \biggr) \label{eq:proof31_last_3}
			\end{align}
			where \eqref{eq:proof31_last_3} follows from $ \frac{L (W+1)^L B^{L}}{\varepsilon} \leq \frac{(W+1)^L \cdot (W+1)^L B^{L}}{\varepsilon} \leq ( \frac{(W+1)^L B^{L}}{\varepsilon} )^2$. The proof is concluded by taking $C := 30$.

	\subsection{Proof of the Lower Bound in Theorem~\ref{thm:covering_number_upper_bound_fully_connected_bounded_weight}}
	\label{sub:covering_number_lower_bound}
           We again start with a series of technical results, which will then be synthesized to the proof of the lower bound. The first of these results reduces the problem of lower-bounding the covering number of $\mathcal{R}(d,W,L,B)$ with respect to the $L^p ( [0,1]^d )$-norm to that of lower-bounding the packing number of $\mathcal{R}(1, W,L,B)$ with respect to the $L^1 ( [0,1] )$-norm.

		\begin{lemma}
		\label{lem:1dpinfinity}
			Let $p \in [1,\infty]$, $d,W,L \in \mathbb{N}$, $B, \varepsilon \in \mathbb{R}_+$ with $B \geq 1$ and $\varepsilon \in (0,1/2)$. 
   We have
			\begin{alignat}{2}
				N(\varepsilon,\mathcal{R}(d,W,L,B),{L^p ( [0,1]^d )} )  \geq &\, M(2\varepsilon,\mathcal{R}(d,W,L,B), {L^p ( [0,1]^d )} ) ) \label{eq:1dpinfinity_0} \\
				\geq &\, M (2\varepsilon,\mathcal{R}(1, W,L,B), {L^1 ( [0,1] )} ) ).\label{eq:1dpinfinity_1}
			\end{alignat}

			\begin{proof}
				The inequality \eqref{eq:1dpinfinity_0} follows from
    Lemma~\ref{lem:equivalence_covering_packing}. To establish \eqref{eq:1dpinfinity_1}, we show that a maximal $(2\varepsilon)$-packing of  $\mathcal{R}(1, W,L,B)$ with respect to the $L^1 ( [0,1] )$-norm induces a $(2\varepsilon)$-packing of $\mathcal{R}(d,W,L,B)$ with respect to the $L^p ( [0,1]^d )$-norm and of the same cardinality.  The proof of this statement is provided in  Appendix~\ref{sub:proof_of_lemma_lem:1dpinfinity}.	
\end{proof}
		\end{lemma}
We shall next make use of the fact that ReLU networks can efficiently realize one-dimensional bounded continuous piecewise linear functions, defined as follows. 
		\begin{definition}
		[One-dimensional bounded continuous piecewise linear functions]
		\cite[Definition B.2]{FirstDraft2022}
			Let $M \in \mathbb{N}$, with $M \geq 3$, $E \in \mathbb{R}_+ \cup \{ \infty \}$, and let $X = (x_i)_{i= 0}^{M-1}$ be a strictly increasing sequence taking values in $\mathbb{R}$. Define the set of functions
			\begin{alignat*}{2}
				\Sigma ( X, E) = \bigl\{f \in&&\, C ( \mathbb{R} ): \| f \|_{L^\infty ( \mathbb{R} )}  \leq E,  f \text{ is constant on } (-\infty, x_0] \text{ and } [x_{M-1}, \infty),\, \\
				&&\,f \text{ is affine on } [x_i, x_{i +1}], \, i = 0,\dots, M -2 \bigr\}.
			\end{alignat*}
			For a function $f \in \Sigma ( X, E)$, we call $X$ the set of its breakpoints, as the slope of $f$ can change only at these points. We  refer to the intervals $(-\infty, x_0], [x_i, x_{i +1}], \, i = 0,\dots, M -2, [x_{M-1}, \infty)$ as the piecewise linear regions of $f$.
		\end{definition}

            We only need to consider breakpoint sets of the form 
			\begin{equation*}
				X_{N} := ( i\slash N )_{i=0}^N, \quad N \in \mathbb{N},
			\end{equation*} 
            along with the associated function families $\Sigma (X_N, E)$, $ E \in \mathbb{R}_+$,  whose $L^1 ([0,1])$-covering number can be lower-bounded as follows.   

		\begin{lemma}
			\label{lem:packing_number_piecewise_linear_function}
			For $\Qnc \in \mathbb{N}$, $\varepsilon, E \in \mathbb{R}_+$, we have 
			\begin{equation}
				\log ( M(\varepsilon,\Sigma ( X_\Qnc, E ), {L^1 ( [0,1] )} ) ) \geq \Qnc \log    \biggl( \biggl \lceil \frac{E}{4 	\varepsilon \Qnc}  \biggr \rceil \biggr), \label{eq:lemma27}
			\end{equation}
			with $X_{N} = ( i\slash N )_{i=0}^N$. 
			\begin{proof}
				See Appendix~\ref{sub:proof_of_lemma_lem:packing_number_piecewise_linear_function}.
			\end{proof}
		\end{lemma}

		To realize functions in $\Sigma (X_N, E)$ efficiently by ReLU networks, we need two technical results from
  \cite{FirstDraft2022}, which we restate for convenience.

		\begin{lemma}
			\label{prop:piecewise_representation_constructive}
			\cite[Proposition C.1]{FirstDraft2022}
			Let $M \in \mathbb{N}$ with $M \geq 3$, $E \in \mathbb{R}_+$, and let $X = (x_i)_{i = 0}^{M-1}$ be a strictly increasing sequence taking values in $[0,1]$. Then, for all $u,v \in \mathbb{N}$ such that $u^2 v \geq M$, we have
			\begin{align*}
				\Sigma ( X, E) \subseteq \mathcal{R} ( 1, 20u, 30v, \max \{ 1, C M^6 (R_m ( X ))^4 E \} ),
			\end{align*}
			for an absolute constant $C \in \mathbb{R}$ satisfying $2\leq C \leq 10^5$, and where $R_m(X) := \max_{i =1,\dots, M} ( x_{i} - x_{i -1} )^{-1}$.
		\end{lemma}

    The second result is as follows.
    
		\begin{lemma}
			\label{prop:depth_weight_magnitude_tradeoff}
			\cite[Proposition H.4]{FirstDraft2022} 
			Let $W_1,L_1 \in \mathbb{N}$, with $W_1 \geq 2$, $L_2 \in \mathbb{N} \, \cup \, \{ 0 \}$, and $B_1,B_2 \in \mathbb{R}_+$, with $B_1,B_2\geq 1$. It holds that
			\begin{equation}
				\label{eq:depth_weight_magnitude_tradeoff}
				\frac{(B_2)^{L_1 + L_2}   \lfloor W_1\slash 2 \rfloor^{L_2}}{{B_1}^{L_1}}  \cdot \mathcal{R} (1, W_1, L_1, B_1 ) \subseteq \mathcal{R} ( 1, W_1, L_1 + L_2, B_2 ).
			\end{equation}
		\end{lemma}
		We are now ready to prove the lower bound in Theorem~\ref{thm:covering_number_upper_bound_fully_connected_bounded_weight} and start
  by noting that thanks to Lemma~\ref{lem:1dpinfinity}, 
 it suffices to lower-bound $M(2 \varepsilon,\mathcal{R}(1, W,L,B),{L^1 ( [0,1] )} )$. 
		We proceed to identify the family of bounded continuous piecewise linear functions corresponding to the set $\mathcal{R} (1,W,L,B)$. To this end, we first introduce notation, namely we set
		\begin{equation*}
			u := \biggl\lfloor \frac{W}{20} \biggr\rfloor, \quad v := \biggl\lfloor \frac{L}{60} \biggr\rfloor.
		\end{equation*}
		As $W,L \geq 60$, we have $u \geq 3$ and $v \geq 1$. Application of Lemma~\ref{prop:piecewise_representation_constructive} with $M = u^2 v $, $X = X_{u^2v - 1} = ( \frac{i}{u^2 v - 1} )_{i=0}^{u^2 v - 1}$, $R_m (X_{u^2v - 1}  ) =  u^2v - 1$, and $E = \frac{1}{C ( u^2 v )^{10}}$, with the absolute constant $C$ per Lemma~\ref{prop:piecewise_representation_constructive}, yields 
		\begin{equation}
			\Sigma \biggl( X_{u^2v - 1}, \frac{1}{C ( u^2 v )^{10}}\biggr) \subseteq \mathcal{R} ( 1, 20u, 30v, 1). \label{eq:rand_2}
   \end{equation}
		Next, application of Lemma~\ref{prop:depth_weight_magnitude_tradeoff} with $W_1 = 20u \geq 2$, $L_1 = 30v$, $B_1 = 1$, $L_2 =30v $, $B_2 = B \geq 1$, yields $( B^{60v}    (10u)^{30v} )  \cdot \mathcal{R} (1,  20u, 30v, 1 ) \subseteq \mathcal{R} ( 1, 20u, 60v, B )$, which together with $\mathcal{R} ( 1, 20u, 60v, B ) = \mathcal{R} ( 1, 20 \lfloor \frac{W}{20} \rfloor, 60 \lfloor \frac{L}{60} \rfloor, B ) \subseteq \mathcal{R} (1, W, L, B )$ establishes that
  		\begin{equation}
		\label{eq:2801}
			( B^{60v}    (10u)^{30v} )  \cdot \mathcal{R} (1,  20u, 30v, 1 ) \subseteq \mathcal{R} (1, W, L, B ).
		\end{equation}
		Moreover, as $a \cdot \Sigma (X_{u^2v - 1} , b) = \Sigma (X_{u^2v - 1} , ab)$, for all $a,b \in \mathbb{R}_+$, we have
		\begin{equation}
            \label{eq:2802}
			( B^{60v}    (10u)^{30v} )  \cdot  \Sigma \biggl( X_{u^2v - 1}, \frac{1}{C ( u^2 v )^{10}}\biggr)  = \Sigma \biggl( X_{u^2v - 1}, \frac{( B^{60v}    (10u)^{30v} )}{C ( u^2 v )^{10}}\biggr).
		\end{equation} 
		Combining \eqref{eq:2801}, \eqref{eq:2802}, and \eqref{eq:rand_2}, then yields
		\begin{equation}
		\label{eq:inclusion_complicated_cpwl}
			\Sigma \biggl( X_{u^2v - 1}, \frac{( B^{60v}    (10u)^{30v} )}{C ( u^2 v )^{10}}\biggr) \subseteq \mathcal{R} (1, W, L, B ).
		\end{equation}
		We have 
			\begin{align}
				&\log (M(2 \varepsilon,\mathcal{R}(1, W,L,B),{L^1 ( [0,1] )} ) ) \label{eq:MuvLower_000}\\
				&\geq \, \log \biggl(M\biggl(2 \varepsilon,\Sigma \biggl( X_{u^2v - 1}, \frac{ B^{60v}    (10u)^{30v} }{C ( u^2 v )^{10}}\biggr),{L^1} ( [0,1] ) \biggr) \label{eq:MuvLower_1}\biggr)  \\
				&\geq \, ( u^2 v - 1 ) \log \biggl(  \biggl \lceil \frac{ B^{60v}    (10u)^{30v} }{C ( u^2 v )^{10}} \cdot \frac{1}{4 \cdot 2\varepsilon ( u^2v - 1 )}  \biggr \rceil \biggr ) \label{eq:MuvLower_2} \\
				&\geq \, ( u^2 v - 1 ) \log \biggl(  \frac{B^{60v} ( 100 u^2  )^v }{\varepsilon }  \cdot \frac{10^6}{8C} \biggr) \label{eq:MuvLower_3}\\
				&\geq \, ( u^2 v - 1 ) \log \biggl(  \frac{ ( 100 B u  )^v }{\varepsilon }   \biggr)\label{eq:MuvLower_31}
			\end{align}
			where  in \eqref{eq:MuvLower_1} we used the inclusion relation \eqref{eq:inclusion_complicated_cpwl} together with Lemma~\ref{lem:equivalence_covering_packing}, \eqref{eq:MuvLower_2} is by Lemma~\ref{lem:packing_number_piecewise_linear_function}, \eqref{eq:MuvLower_3} follows from $ (10u)^{30v} \geq 10^{6} \cdot u^{22} \cdot ( 2^{v} )^{11} \cdot  ( 10u )^{2v}  \geq 10^6\cdot u^{22} \cdot v^{11} \cdot  ( 10u )^{2v}  \geq  10^6 ( u^2 v)^{10} ( u^2 v - 1 )  ( 100u^2 )^v $, and in \eqref{eq:MuvLower_31} we employed $u \geq 3$, $v \geq 1$, $B \geq 1$, and $8 C \leq 8 \cdot 10^5 < 10^6$. We next lower-bound $( u^2 v - 1 ) \log (  \frac{ ( 100 B u  )^v }{\varepsilon }   )$ in terms of $W,L,B$ according to 
			\begin{align}
				( u^2 v - 1 ) \log \biggl(  \frac{( 100Bu )^{v} }{\varepsilon }   \biggr ) \geq&\, \frac{1}{2} u^2 v \log \biggl(  \frac{( 100 B u )^{v} }{\varepsilon }   \biggr )  \label{eq:rd_20} \\
				\geq&\,  \frac{1}{2}  \biggl(\frac{W}{40}\biggr)^2 \frac{L}{120} \log \biggl(  \frac{( 100 B \frac{W}{40} )^{\frac{L}{120}} }{\varepsilon }   \biggr ) \label{eq:rd_201} \\
				\geq&\, \frac{1}{2 \cdot 40^2 \cdot 120} W^2 L  \log \biggl(  \frac{(  B W )^{\frac{L}{120}} }{\varepsilon^{\frac{1}{120}} }   \biggr ) \\
				= &\, \frac{1}{2 \cdot  40^2 \cdot 120^2} W^2 L  \log \biggl(  \frac{W^L B^L}{\varepsilon }   \biggr ). \label{eq:rd_2}\\
				\geq&\,  \frac{1}{4\cdot 40^2 \cdot 120^2} W^2 L  \log \biggl(  \frac{(W+1)^L B^L}{\varepsilon }   \biggr ), \label{eq:rd_3}
			\end{align}
			where \eqref{eq:rd_20} follows from $u^2v - 1 \geq u^2v - \frac{1}{9} u^2v \geq \frac{1}{2} u^2 v $ as $u \geq 3$ and $v \geq 1$, in \eqref{eq:rd_201} we used $u  = \lfloor \frac{W}{20} \rfloor \geq  \frac{W}{40}$ and $v  = \lfloor \frac{L}{60} \rfloor \geq  \frac{L}{120}$, and \eqref{eq:rd_3} is by $2 \log (  \frac{W^L B^L}{\varepsilon } )  = \log (  \frac{W^{2L} B^{2L}}{\varepsilon^2 } ) \geq \log (  \frac{(W+1)^L B^L}{\varepsilon }  )$  as $W \geq 60$ by assumption. The proof is concluded by 
   setting $c = \frac{1}{4 \cdot 40^2 \cdot 120^2}$.

\section{Neural Network Transformation and Function Approximation}
	\label{sub:fundamental_limit_covering_number}
		We now show how the precise characterization of ReLU network covering numbers obtained in the previous section can be put to work to characterize the fundamental limits of neural network transformation and function approximation. Before describing the specifics of these two problems, we need a general result which relates the covering numbers of sets $\mathcal{G}$ and $\mathcal{F}$ that are close in terms of minimax distance 
		\begin{equation*}
			\mathcal{A} ( \mathcal{G}, \mathcal{F}, \delta )  = \sup_{g \in \mathcal{G}} \inf_{f \in \mathcal{F}} \delta (f,g),
		\end{equation*}
		with respect to some metric $\delta$. 

 \begin{proposition}			\label{lem:cardinality_approximation_class}
			Let $( \mathcal{X}, \delta )$ be a metric space, $\mathcal{F},\mathcal{G} \subseteq \mathcal{X}$, and $\varepsilon \in \mathbb{R}_+$. Suppose that
			\begin{equation}
				\label{eq:improper_cover}
				\mathcal{A} ( \mathcal{G}, \mathcal{F}, \delta ) \leq \varepsilon.
			\end{equation}
			Then,
			\begin{equation}
			\label{eq:minimal_cardinality_G}
				N ( \varepsilon, \mathcal{F}, \delta ) \geq N ( 4\varepsilon, \mathcal{G}, \delta ).
			\end{equation}
			\begin{proof}
				See Appendix~\ref{sub:approximation_error_lower_bound_by_covering_number}.
			\end{proof}
		\end{proposition}

		\subsection{Neural Network Transformation}
		\label{sub:network_transformation}

			Generally speaking, neural network transformation is the practice of approximating or exactly realizing a given neural network with certain structural properties by another neural network satisfying different prescribed structural properties.
    This problem has a number of concrete incarnations. For example, in network compression the objective is to reduce the size of networks. In practice, this is often effected through techniques such as pruning \cite{Janowsky1989, Blalock2020} or knowledge distillation \cite{Gou2021}. 
    Another example is network quantization, where real-valued network weights are replaced by weights that are quantized to a predetermined level of precision, or  high-precision weights are substituted by lower-precision weights. This can be done either by rounding each individual weight to the nearest quantization point or by searching for the best set of quantized weights jointly through specific algorithms \cite{Maly2022}.
    The primary motivation for network compression and quantization stems from the necessity to store neural networks on microchips under prescribed memory constraints.
    Further examples of neural network transformation appear in \cite[Theorem 3.1]{Vardi2022} where a given network is transformed into one that is narrower and deeper, and in \cite[Lemma A.1]{deep-it-2019}, \cite[Theorem 5]{Schmidt-Hieber2017}, \cite[Corollary 3.2]{Vardi2022} which all 
    employ transformations into networks of smaller weight magnitude.

More formally, the problem of neural network transformation can be cast as follows. Considering the classes of networks $\mathcal{R}_1$ and $\mathcal{R}_2$,
one wants to approximate a given network $r_1 \in \mathcal{R}_1$ by a network $r_2 \in \mathcal{R}_2$ such that the distance $\delta ( r_1, r_2 )$, for some metric $\delta$, is minimized. The fundamental limit on the worst-case error incurred by
the transformation mapping $\mathfrak{C}:\mathcal{R}_1 \rightarrow \mathcal{R}_2$, under the metric $\delta$, 
is characterized by the minimax approximation error $\mathcal{A} ( \mathcal{R}_1, \mathcal{R}_2, \delta )$ according to
\begin{equation*}
				\sup_{r_1 \in \mathcal{R}_1} \delta ( r_1, \mathfrak{C}(r_1) ) \geq 
    \sup_{r_1 \in \mathcal{R}_1} \inf_{r_2 \in \mathcal{R}_2} \delta ( r_1, r_2 ) = 
    \mathcal{A} ( \mathcal{R}_1, \mathcal{R}_2, \delta ).
			\end{equation*}

In \cite[Theorem 1.1]{telgarsky2016benefits}, for example, a lower bound on $\mathcal{A} ( \mathcal{R}_1, \mathcal{R}_2, \delta )$ was provided in terms of the oscillation count of ReLU networks (as defined in \cite[Sec.~3]{telgarsky2016benefits}), in the case
where deep networks are replaced by shallow ones.
We next show how  $\mathcal{A} ( \mathcal{R}_1, \mathcal{R}_2, \delta )$ 
can be characterized for general $\mathcal{R}_{1}$ and $\mathcal{R}_{2}$ through covering numbers.
For concreteness, we consider $\mathcal{R}_{1}=\mathcal{R} ( d, W,L,B), \mathcal{R}_{2}=\mathcal{R} ( d, \widetilde{W},\widetilde{L},\widetilde{B})$ with $\delta=L^p([0,1]^d)$ for general $p \in [1,\infty]$.
 				\begin{corollary}
				\label{cor:fundamental_limit_representation}	
				Let $p \in [1,\infty]$,  $d  \in \mathbb{N}$,  $ L,\widetilde{L}, W, \widetilde{W}\in \mathbb{N}$, with $W, L \geq 60$, $B, \widetilde{B} \in \mathbb{R}_+$, with  $B, \widetilde{B}  \geq 1$. Assume that there exists an $\varepsilon \in (0, 1\slash 8)$ such that 
				\begin{equation}
				\label{eq:definition_compression}
					\mathcal{A} (\mathcal{R} ( d, W,L,B),\mathcal{R} ( d, \widetilde{W},\widetilde{L},\widetilde{B}), \nleft\| \cdot \nright\|_{L^p ( [0,1]^d )}   ) \leq \varepsilon.
				\end{equation}
				Then,
				\begin{equation}
				\label{eq:necessary_condition_compression}
					c\, W^2 L \log \biggl( \frac{(W+1)^L B^{L}}{4\varepsilon} \biggr) \leq C\, \widetilde{W}^2 \widetilde{L} \log \biggl( \frac{(\widetilde{W}+1)^{\widetilde{L}} \widetilde{B}^{\widetilde{L}}}{\varepsilon} \biggr), 	
				\end{equation} 
				where $C$ and $c$ are the absolute constants in Theorem~\ref{thm:covering_number_upper_bound_fully_connected_bounded_weight} corresponding to the parameters $\widetilde{W},\widetilde{L},\widetilde{B}$ and $W,L,B$, respectively. 
        In particular, if  $\mathcal{R} ( d, W,L,B ) \subseteq \mathcal{R} ( d, \widetilde{W},\widetilde{L},\widetilde{B} )$, we have
				\begin{equation}
				\label{eq:lossy_compression}
					C\, \widetilde{W}^2 \widetilde{L} \geq c\, W^2 L.
				\end{equation}
				\begin{proof}

				Application of Proposition~\ref{lem:cardinality_approximation_class} with $\delta = \nleft\| \cdot \nright\|_{L^p ( [0,1]^d )} $, $\mathcal{G} = \mathcal{R} ( d, W,L,B)$, and $\mathcal{F} = \mathcal{R} ( d, \widetilde{W},\widetilde{L},\widetilde{B})$, and the prerequisite \eqref{eq:improper_cover} satisfied thanks to  \eqref{eq:definition_compression}, yields 
				\begin{equation*}
				 	N ( \varepsilon, \mathcal{R} ( d, \widetilde{W},\widetilde{L},\widetilde{B}), L^p ( [0,1]^d )  ) \geq N ( 4\varepsilon, \mathcal{R} ( d, W,L,B), L^p ( [0,1]^d ) ),
				 \end{equation*}
				 which together with Theorem~\ref{thm:covering_number_upper_bound_fully_connected_bounded_weight} establishes \eqref{eq:necessary_condition_compression}. 
				If $\mathcal{R} ( d, W,L,B ) \subseteq \mathcal{R} ( d, \widetilde{W},\widetilde{L},\widetilde{B} )$, then \eqref{eq:definition_compression} holds for all $\varepsilon \in \bigl(0,\frac{1}{8}\bigr)$ and consequently so does \eqref{eq:necessary_condition_compression}. Dividing \eqref{eq:necessary_condition_compression} by $\log (\frac{1}{\varepsilon})$ and letting $\varepsilon \to 0 $ results in \eqref{eq:lossy_compression}.
				\end{proof}
			\end{corollary}

 Corollary~\ref{cor:fundamental_limit_representation} allows us to answer the following question on network size reduction: Is it possible to approximate a network in $\mathcal{N}(d,W,L,B)$ by one in $\mathcal{N}(d,\widetilde{W},\widetilde{L},\widetilde{B})$, with prescribed error $\varepsilon$, while
 having the maximum number of nonzero weights of the approximating network, $\widetilde{W}^2 \widetilde{L}$, be of order smaller than
 that of the original network? When $\varepsilon = 0$, i.e., $\mathcal{R}(d,W,L,B) \subseteq \mathcal{R}(d,\widetilde{W},\widetilde{L},\widetilde{B})$,    \eqref{eq:lossy_compression} shows that the answer is negative. For $\varepsilon \in (0,1/8)$, we can conclude from (\ref{eq:necessary_condition_compression}) that this would require that the weight magnitude $\widetilde{B}$ compensate for the reduction in an exponential manner. Note that Corollary~\ref{cor:fundamental_limit_representation} is a conditional statement, in that it quantifies
 constraints on the network parameters that must hold whenever a uniform approximation guarantee of the form \eqref{eq:definition_compression} is available.

			We next consider network transformation through weight quantization. 

			\begin{corollary}
				\label{eq:fundamental_limit_quantization}
				Let $p \in [1,\infty]$, $d,W,L  \in \mathbb{N}$, $B \in \mathbb{R}_+$, with $B  \geq 1$ and $W, L \geq 60$. Let $\mathbb{A} \subseteq \mathbb{R}$ be a finite set such that $\nleft| \mathbb{A} \nright| \geq 2 $. Then, we have 
				\begin{equation}
				\label{eq:fundamental_limit_quantization_1}
					\mathcal{A} ( \mathcal{R} ( d, W,L,B), \mathcal{R}_{\mathbb{A}} ( d, W,L) , \nleft\| \cdot \nright\|_{L^p ( [0,1]^d )}) \geq \min \{ 1\slash 8, (W+1)^L B^{L} 2^{-c \log ( \nleft| \mathbb{A} \nright| )}  \},
				\end{equation}
				for some absolute constant $c \in \mathbb{R}_+$.

			\end{corollary}

			\begin{proof}
				Let 
				\begin{equation}
				\label{eq:fundamental_limit_quantization_2}
				\kappa:= \mathcal{A} ( \mathcal{R} ( d, W,L,B ), \mathcal{R}_{\mathbb{A}} ( d, W,L) , \nleft\| \cdot \nright\|_{L^p ( [0,1]^d )}).
				\end{equation}
				When $\kappa \geq 1\slash 8$, the desired inequality \eqref{eq:fundamental_limit_quantization_1} holds trivially.  
				For $\kappa < 1\slash 8$, it follows from Proposition~\ref{lem:cardinality_approximation_class} with $\varepsilon = \kappa$, $\delta = \nleft\| \cdot \nright\|_{L^p ( [0,1]^d )}$, $\mathcal{G} = \mathcal{R} ( d, W,L,B )$, and $\mathcal{F} = \mathcal{R}_{\mathbb{A}} ( d, W,L) $
    that 
				\begin{equation}
				\label{eq:quantizatio_0} 
					N ( \kappa, \mathcal{R}_{\mathbb{A}} ( d, W,L ), L^p ( [0,1]^d ) ) \geq N ( 4\kappa, \mathcal{R} ( d, W,L,B ), L^p ( [0,1]^d ) ).
				\end{equation}
				We next note that
				\begin{align}
					\log ( N ( \kappa, \mathcal{R}_{\mathbb{A}} ( d, W,L), L^p ( [0,1]^d ) ))  \leq&\,\log ( \nleft| \mathcal{R}_{\mathbb{A}} ( d, W,L) \nright|) \label{eq:quantizatio_1} \\
					\leq &\, 5 W^2 L \log ( \nleft| \mathbb{A} \nright|  ),\label{eq:quantizatio_2}
				\end{align}
				where \eqref{eq:quantizatio_1} follows from the fact that every set is a covering of itself and \eqref{eq:quantizatio_2} is by Lemma~\ref{lem:counting_cardinality}. Further, it follows from \eqref{eq:lower_bound_fully_connected_bounded_output_main} in  Theorem~\ref{thm:covering_number_upper_bound_fully_connected_bounded_weight} with $\varepsilon = 4 \kappa$, that 
				\begin{align}
					\log ( N ( 4\kappa, \mathcal{R} ( d, W,L,B), L^p ( [0,1]^d ) ))  \geq c_1 W^2 L \log \biggl( \frac{(W+1)^L B^{L}}{4 \kappa} \biggr), \label{eq:quantizatio_3}
				\end{align}
				with $c_{1} \in \mathbb{R}_+$ an absolute constant.
    Using \eqref{eq:quantizatio_1}-\eqref{eq:quantizatio_2} and \eqref{eq:quantizatio_3} in \eqref{eq:quantizatio_0}, yields
				\begin{equation*}
					5 W^2 L \log ( \nleft| \mathbb{A} \nright|  ) \geq c_1 W^2 L \log \biggl( \frac{(W+1)^L B^{L}}{4 \kappa} \biggr), 
				\end{equation*}
				which implies  $\kappa \geq \frac{1}{4} (W+1)^L B^{L} 2^{-\frac{5}{c_1} \log ( \nleft| \mathbb{A} \nright|  )} \geq (W+1)^L B^{L} 2^{- \bigl( \frac{5}{c_1} + 2 \bigr) \log ( \nleft| \mathbb{A} \nright|  )}$, thanks to $\log ( \nleft| \mathbb{A} \nright|  ) \geq 1$. The proof is concluded by setting $c:= \frac{5}{c_1} + 2$.
			\end{proof}

			Corollary~\ref{eq:fundamental_limit_quantization} allows us to conclude that the worst-case quantization error $\mathcal{A} ( \mathcal{R} ( d, W,L,B), \\ \mathcal{R}_{\mathbb{A}} ( d, W,L) , \nleft\| \cdot \nright\|_{L^p ( [0,1]^d )})$ decreases no faster than exponential in the number of bits $\log(|\mathbb{A}|)$ required to store the elements of $\mathbb{A}$. Moreover, as $W,L$, and $B$ grow, the network weight resolution has to increase in order to compensate for the growth in the factor $( W+1 )^L B^L$. Specifically, if we require that 
   $ \mathcal{A} ( \mathcal{R} ( d, W,L,B), \mathcal{R}_{\mathbb{A}} ( d, W,L) , \nleft\| \cdot \nright\|_{L^p ( [0,1]^d )}) \leq \kappa$, we must have $\log ( \nleft| \mathbb{A} \nright| ) \geq \frac{1}{c}\log \bigl( \frac{( W+1 )^L B^L}{\kappa} \bigr)$.
This lower bound
   can be achieved, within a multiplicative constant, by taking $\mathbb{A}$ to be an equidistant set contained in the interval $[-B,B]$. To see this, we set
      $\mathbb{A} = [-B,B] \cap 2^{-b} \mathbb{Z}$ with $b = \lceil \log ( \frac{L ( W+1 )^L B^{L-1}}{\kappa} )\rceil$ and note that
			\begin{align}
				&\,\mathcal{A} ( \mathcal{R}(d,W,L,B), \mathcal{R}_{[-B,B] \cap 2^{-b} \mathbb{Z}}(d,W,L), \nleft\| \cdot \nright\|_{L^p ( [0,1]^d )}  )\\
				&\leq\, L (W+1)^L B^{L-1} 2^{-b} \label{eq:achievability_quantize_1} \\
				&\leq\, \kappa,
			\end{align}
			where in \eqref{eq:achievability_quantize_1} we applied Lemma~\ref{lem:define_covering_fully_connected}. The argument is concluded upon realizing that $ \log ( \nleft| \mathbb{A} \nright|  ) = \log ( \nleft| [-B,B] \cap 2^{-b} \mathbb{Z} \nright|  ) \leq 3 \log ( \frac{L ( W+1 )^L B^L}{\kappa}) \leq 6\log ( \frac{( W+1 )^L B^L}{\kappa})$, where the first inequality follows from \eqref{eq:ccw_1}-\eqref{eq:ccw_6} and the second is by
      $\frac{L (W+1)^L B^{L}}{\kappa} \leq \frac{(W+1)^L \cdot (W+1)^L B^{L}}{\kappa} \leq ( \frac{(W+1)^L B^{L}}{\kappa} )^2$.

			We finally emphasize that the results on the fundamental limits on neural network transformation presented in this section are made possible by the tight covering number lower bound \eqref{eq:lower_bound_fully_connected_bounded_output_main} in Theorem~\ref{thm:covering_number_upper_bound_fully_connected_bounded_weight}.

		\subsection{Deep Neural Network Function Approximation}
		\label{sub:fundamental_limit_for_approximaiton}

We next show how the covering number bounds in Theorem~\ref{thm:covering_number_upper_bound_fully_connected_bounded_weight}, in combination with 
Proposition~\ref{lem:cardinality_approximation_class}, can be used to establish a tight characterization of the minimax error in the ReLU network approximation of the class of $1$-Lipschitz functions
			\begin{equation}
			\label{eq:def_1lipschitz}
				\lip  ( [0,1] ):= 	\{ f \in C ( [0,1] ): | f(x) | \leq 1, | f(x) - f(y) | \leq | x - y |, \,\forall x,y \in [0,1]    \}.
			\end{equation}

To this end, we start with the following upper bound on the minimax error.

			\begin{lemma}
				\label{lemma:approximation_lip}
				There exist absolute constants $C,D \in \mathbb{R}_+$ such that, for all $W,L \in \mathbb{N}$, with $W,L \geq D$, and $p \in [1,\infty]$,
				\begin{equation}
				\label{eq:approximation_lip_001}
					\begin{aligned}
						 \mathcal{A}	( \lip  ( [0,1] ), \mathcal{R} ( 1, W, L, 1 ), \nleft\| \cdot \nright\|_{L^p ( [0,1] )}  )\leq   C  ( W^2 L^2 \log (W)  )^{-1}.
					\end{aligned}
				\end{equation}
				\begin{proof}
					By \cite[Theorem 3.1]{FirstDraft2022}, there exist absolute constants $C,D \in \mathbb{R}_+$ such that, for all $W,L \in \mathbb{N}$, with $W,L \geq D$, $\mathcal{A}	( \lip  ( [0,1] ), \mathcal{R} ( 1,W, L, 1 ), \nleft\| \cdot \nright\|_{L^\infty ( [0,1] ) } )\leq   C  ( W^2 L^2 \log (W)  )^{-1}$.  Noting that, for all $p \in [1,\infty]$, the $L^p([0,1])$-norm is dominated by the $L^\infty ([0,1])$-norm, i.e.,  $\nleft\| f \nright\|_{L^p ( [0,1] )} \leq \nleft\| f \nright\|_{L^\infty ( [0,1] )}, \forall f \in L^\infty ( [0,1] )$, we have, for all $W,L \in \mathbb{N}$, with $W,L \geq D$, and $p \in [1,\infty]$, $\mathcal{A}	( \lip  ( [0,1] ), \mathcal{R} ( 1, W, L, 1 ), \nleft\| \cdot \nright\|_{L^p ( [0,1] )}  )  \leq \mathcal{A}	( \lip  ( [0,1] ), \mathcal{R} ( 1, W, L, 1 ), \nleft\| \cdot \nright\|_{L^\infty ( [0,1] )}  ) \leq   C  ( W^2 L^2 \log (W)  )^{-1}.$
				\end{proof}
			\end{lemma}
		A corresponding lower bound, for $p = \infty$ and $L \geq 2$, obtained through arguments involving  VC dimension is given by \cite[Proposition 2.11]{FirstDraft2022}\footnote{The lower bound \cite[Proposition 2.11]{FirstDraft2022} is actually stated for 
        neural networks with unbounded weights, i.e., with $\mathcal{R} ( 1, W, L,1 ) $ replaced by $\mathcal{R} ( 1, W, L,\infty )$.}
		\begin{equation}
		\label{eq:eq:approximation_error_from_VC_2}
		 	\mathcal{A}	( H^1 ( [0,1] ), \mathcal{R} ( 1, W, L,1 )  , \nleft\| \cdot \nright\|_{L^\infty ( [0,1] )} )\geq   c_v  ( W^2 L^2 ( \log (W) + \log(L) ) )^{-1},
		\end{equation} 
where $c_v$ is an absolute constant.
  A comparison of the upper bound \eqref{eq:approximation_lip_001} and the lower bound \eqref{eq:eq:approximation_error_from_VC_2} reveals a gap owing to the additive term $\log(L)$ in the lower bound. The question is now whether the lower or the upper bound would need to be refined to close this gap. The tight
  covering number bounds in Theorem~\ref{thm:covering_number_upper_bound_fully_connected_bounded_weight} allow to answer this question. Concretely, it turns out that it is the lower bound that can be strengthened. The resulting improvement pertains to all $p \in [1, \infty]$.
  To see all this, we start with a lower bound on the covering number of $\lip([0,1])$.

		\begin{lemma}
		\label{lem:packing_lower_bound_lipschitz}
			There exists an absolute constant $C \in \mathbb{R}_+ $ such that, for all $p \in [1,\infty]$ and $\varepsilon \in (0,1\slash 2)$,
			\begin{equation}
			\label{eq:packing_lower_bound_lipschitz}
				\log ( N ( \varepsilon, \lip  ( [0,1] ), L^p ( [0,1])  ) ) \geq  C \varepsilon^{-1}.
			\end{equation}

		\end{lemma}
For $p = \infty$, the statement of Lemma~\ref{lem:packing_lower_bound_lipschitz} is \cite[Example 5.10]{wainwright2019high}. An asymptotic version of Lemma~\ref{lem:packing_lower_bound_lipschitz} was provided in \cite[Theorem 1.7]{Sh1980}. Inspection of the proof of \cite[Theorem 1.7]{Sh1980} reveals quite directly that the result holds in nonasymptotic form as stated here and does so even for a more general class of functions. The proof of Lemma~\ref{lem:packing_lower_bound_lipschitz} is hence omitted.

We are now ready to present the strengthened lower bound.

		\begin{corollary}
			\label{cor:fundamental_limit_approximation}
			Let $p \in [1, \infty]$,  $W, L \in \mathbb{N}$. It holds that
			\begin{equation}
			\label{eq:cor_fundamental_limit_approximation}
				\mathcal{A}	( H^1 ( [0,1] ), \mathcal{R} ( 1, W, L, 1 ), \nleft\| \cdot \nright\|_{L^p ( [0,1] )}  )  \geq  \min \biggl\{ \frac{1}{8}, c \,  ( W^2 L^2 \log (W)  )^{-1}  \biggr\},
			\end{equation}
			where $c \in \mathbb{R}_+$ is an absolute constant.
			\begin{proof}	
				Let
				\begin{equation}
				\label{eq:proof_fundamental_limit_approximation_1}
					\kappa := \mathcal{A}	( H^1 ( [0,1] ), \mathcal{R} ( 1, W, L, 1 ), \nleft\| \cdot \nright\|_{L^p ( [0,1] )}  ) .
				\end{equation} 
    When $\kappa \geq 1\slash 8$, the desired inequality \eqref{eq:cor_fundamental_limit_approximation} holds trivially.  
        For $\kappa < 1 \slash 8$, it follows from Proposition~\ref{lem:cardinality_approximation_class} with $\varepsilon = \kappa$, $\delta = \nleft\| \cdot \nright\|_{L^p ( [0,1] )} $, $\mathcal{F} = \mathcal{R} ( 1, W, L, 1 )$, and $\mathcal{G} = H^1 ( [0,1] )$
        that 
				\begin{equation}
				\label{eq:proof_cor_25}
				 	N (\kappa,\mathcal{R} ( 1, W, L, 1 )  ,   L^p ( [0,1] )) \geq N ( 4 \kappa ,\lip ( [0,1] ), L^p ( [0,1] )).
				 \end{equation} 
				The left-hand-side of \eqref{eq:proof_cor_25} can now be upper-bounded by \eqref{eq:upper_bound_fully_connected_bounded_output} in Theorem~\ref{thm:covering_number_upper_bound_fully_connected_bounded_weight} according to 
				\begin{equation}
				\label{eq:proof_cor_25_2}
					\log ( N (\kappa,\mathcal{R} ( 1, W, L, 1 )  ,   L^p ( [0,1] )) ) \leq  C_1 W^2 L \log \biggl( \frac{(W+1)^L}{\kappa} \biggr),
				\end{equation}
				with $C_1$ an absolute constant. Application of Lemma~\ref{lem:packing_lower_bound_lipschitz} with $\varepsilon = 4\kappa$, yields 
				\begin{equation}
				\label{eq:proof_cor_25_20}
					\log ( N ( 4\kappa, \lip  ( [0,1] ), L^p ( [0,1])  ) ) \geq  C_{2} \kappa^{-1},
				\end{equation}
				with $C_2$ an absolute constant. Using \eqref{eq:proof_cor_25_2} and \eqref{eq:proof_cor_25_20} in \eqref{eq:proof_cor_25} leads to
				\begin{equation}
				\label{eq:eq:proof_cor_25_4}
					C_1 W^2 L^2 \log ( W+1 ) + C_1 W^2 L \log(\kappa^{-1}) - C_2 \kappa^{-1} \geq 0 .
				\end{equation}
    Next, we define $f : \mathbb{R} \rightarrow \mathbb{R}$ as
				\begin{equation}
					f(x) = C_1 W^2 L^2 \log ( W+1 ) + C_1 W^2 L \log(x) - C_2 x,
				\end{equation}
				which allows us to rewrite \eqref{eq:eq:proof_cor_25_4} as 
				\begin{equation}
				\label{eq:eq:proof_cor_25_5}
					f(\kappa^{-1}) \geq 0.
				\end{equation}
				We proceed to characterize the feasible set $\{ x: f(x) \geq 0\}$ for $\kappa^{-1}$. First, note that the set $\{ x \in [100, \infty): 2^x -  \frac{2 C_1}{C_2} \cdot ( x + 10 ) \geq 0 \}$ is nonempty, and let $\nu := \inf \{ x \in [100, \infty): 2^x -  \frac{2 C_1}{C_2} \cdot ( x + 10 ) \geq 0 \} \in [100, \infty)$. Thanks to the continuity of the mapping $x \in \mathbb{R} \mapsto 2^x -  \frac{2 C_1}{C_2} \cdot ( x + 10 )$, 
    we have $2^\nu -  \frac{2 C_1}{C_2} \cdot ( \nu + 10 ) \geq 0$. Moreover, $\nu$ depends on $\frac{C_1}{C_2}$ only.  Let 
				\begin{equation*}
					b := \frac{2 C_1}{C_2} \cdot ( \nu + 10 ).
				\end{equation*}
				Then, we have  $2^\nu \geq b$, and 
				\begin{align}
					&\,f\biggl(b W^2 L^2 \log(W+1)\biggr) \label{eq:proof_linfty_lower_0}\\
					&=\, C_1 W^2 L^2 \log ( W+1 ) + C_1 W^2 L \log (b W^2 L^2 \log(W+1)) - C_2 \, b \, W^2 L^2 \log(W+1) \\
					&=\, \biggl(C_1 - \frac{C_2 b}{2}\biggr) W^2 L^2 \log ( W+1 )\nonumber\\
					&\, + C_1  W^2 L \biggl(  \log (b W^2 L^2 \log(W+1)) - \frac{C_2 b}{2 C_1} L \log(W+1) \biggr) \\
					& < \,  C_1 W^2 L \biggl(\log \biggl(b W^2 L^2 \log(W+1)\biggr) -   \log\biggl(( W+1 )^{\frac{C_2 b}{2 C_1 } L }\biggr)  \biggr)\label{eq:proof_linfty_lower_1} \\
					&\leq \, 0, \label{eq:proof_linfty_lower_100} 
				\end{align}
				where \eqref{eq:proof_linfty_lower_1} follows from $C_1 - \frac{C_2 b}{2} = C_1 - \frac{C_2}{2} \cdot \frac{2 C_1}{C_2} \cdot ( \nu +10 ) < C_1 - \frac{C_2}{2} \cdot  \frac{2 C_1}{C_2}   = 0$, and  in  \eqref{eq:proof_linfty_lower_100} we used 
				\begin{align}
					( W+1 )^{\frac{C_2 b}{2 C_1 } L }  = &\,  ( W+1 )^{( \frac{C_2 b}{2 C_1 } - 10 ) L } ( W+1 )^{10L} \label{eq:w110l1}  \\
					\geq &\,  2^{ \frac{C_2 b}{2 C_1 } - 10  } ( W+1 )^2 \cdot (( W+1 )^L )^2 \cdot ( W+1 ) \\
					\geq &\,  2^{ \frac{C_2 b}{2 C_1 } - 10  } W^2 L^2 \log(W+1) \label{eq:w110l2} \\
					=&\, 2^{\nu} W^2 L^2 \log(W+1)\\
					\geq &\, b W^2 L^2 \log(W+1).
				\end{align} 
				We next note that $f$ is strictly decreasing on $[b W^2 L^2 \log(W+1) , \infty) $ as $f' (x) = \frac{C_1 W^2 L}{x \ln(2)} - C_2 \leq \frac{C_1 W^2 L}{b W^2 L^2 \log(W+1) \ln(2)} - C_2 \leq \frac{C_1}{b \ln(2)} - C_2 = \frac{C_1}{\frac{2 C_1}{C_2} \cdot ( \nu + 10 ) \ln(2)} - C_2 < 0$, for all $x \in [b W^2 L^2 \log(W+1) , \infty) $.  It hence follows from \eqref{eq:proof_linfty_lower_0}-\eqref{eq:proof_linfty_lower_100} that $f(x) < 0$, for all $ x > b W^2 L^2 \log(W+1)$, and therefore
				\begin{equation}
				\label{eq:eq:proof_cor_25_6}
					\{ x: f(x) \geq 0\} \subseteq (-\infty, b W^2 L^2 \log(W+1)].
				\end{equation}
				Putting \eqref{eq:eq:proof_cor_25_5} and \eqref{eq:eq:proof_cor_25_6} together, we obtain
				\begin{equation}
					\kappa^{-1} \leq b W^2 L^2 \log(W+1),
				\end{equation}
				which, in turn, implies 
				\begin{equation*}
					\kappa \geq b^{-1} ( W^2 L^2 \log(W+1) )^{-1} \geq\min \biggl\{ \frac{1}{8}, b^{-1} ( W^2 L^2 \log(W+1) )^{-1}  \biggr\}.
				\end{equation*}
				The proof is concluded by setting $c = b^{-1}$. 
			\end{proof}
		\end{corollary}

	\section{Optimal Rates in Nonparametric Regression}
	\label{sub:empirical_risk_minimizaiton}

		In this section, we show how the minimax error upper bound in Lemma~\ref{lemma:approximation_lip}
    leads to a sharp characterization of the prediction error in nonparametric regression through ReLU networks. The general results we obtain allow to infer, inter alia, 
    that nonparametric regression with very deep\footnote{Here, ``very deep'' refers to networks whose depth increases at least linearly 
    in network width, which is in contrast to networks commonly considered in the literature \cite{Schmidt-Hieber2017, Chen2019, nakada2020adaptive, deep-it-2019} that have depth increasing at most logarithmically in width.} fully-connected ReLU networks achieves optimal sample complexity rate in the estimation of $1$-Lipschitz functions; this improves significantly upon \cite[Theorem 1(b)]{Kohler2019} in the special case of $1$-Lipschitz functions by removing the $(\log(n))^6$-factor. 
    The section concludes with insights on a systematic relation between optimal nonparametric regression and optimal approximation through (deep) ReLU networks unifying numerous corresponding results in the literature \cite{Schmidt-Hieber2017, Chen2019, nakada2020adaptive} and identifying general underlying principles.

  The goal of nonparametric regression is to estimate the unknown function $g: \mathbb{X} \rightarrow \mathbb{R}$, with $\mathbb{X} \subseteq \mathbb{R}^d$, $d \in \mathbb{N}$, referred to as the regression function, from the $n \in \mathbb{N}$ (random) samples
		\begin{equation}
		\label{eq:setting_nonparametric}
			( x_i,y_i )_{i = 1}^n = ( x_i, g(x_i) + \sigma \xi_i )_{i =1}^n,
		\end{equation}
 		where $\sigma \in \mathbb{R}_+$, $( x_i )_{i = 1}^n$ are i.i.d. random variables of distribution $P$ supported on $\mathbb{X}$, 
   $(\xi_i)_{i=1}^n$ are i.i.d. standard (i.e., zero mean and unit variance) Gaussian random variables,
   and $( x_i )_{i = 1}^n$ and $(\xi_i)_{i=1}^n$ are statistically independent. 
\subsection{Nonparametric Regression through ReLU Networks}

		Nonparametric regression through ReLU networks  was considered in \cite{Schmidt-Hieber2017, Chen2019, nakada2020adaptive,  Kohler2019}, with
  $g$ estimated by fitting a network $\hat{f}_n$ from a given class $\mathcal{F}_n$ of networks through minimization of the empirical risk $\frac{1}{n}\sum_{i =1}^n (\hat{f}_n(x_i) - y_i )^2$. For example, \cite{Schmidt-Hieber2017} considers regression functions $g$ that can be written as the composition of bounded H\"older functions and $\mathcal{F}_n$ is a family of sparse ReLU networks with bounded output. The quality of the estimator is generally measured by the so-called prediction error
		\begin{equation*}
			\nleft\| \hat{f}_n - g \nright\|^2_{L^2 ( P )}   = \int \nleft| \hat{f}_n (x) - g(x) \nright|^2\, d\,P(x).
		\end{equation*}
		The references \cite{Schmidt-Hieber2017, Chen2019, nakada2020adaptive,  Kohler2019} report upper bounds on the prediction error. Notably, the bounds in \cite{Kohler2019} are derived employing arguments based on VC-dimension, while those in \cite{Schmidt-Hieber2017, Chen2019, nakada2020adaptive} are obtained from covering number upper bounds for ReLU networks. The following Theorem~\ref{thm:expected_L2_error} summarizes the results 
    \cite[Lemma~4]{Schmidt-Hieber2017}, \cite[Lemma~4 and Lemma~5]{Chen2019}
        and reformulates them so as to highlight the individual effects of the approximation error and the covering number of $\mathcal{F}_n$. 
        Application of the minimax error upper bound in Lemma~\ref{lemma:approximation_lip} then 
        results in the removal of the $(\log(n))^6$-factor in the special case of Lipschitz functions in \cite[Theorem 1(b)]{Kohler2019}.
            Moreover, our reformulation sets the stage for the development of a fundamental relation between optimal approximation and optimal regression through ReLU networks provided at the end of this section. We emphasize that most of the techniques and ideas used in the proof of Theorem~\ref{thm:expected_L2_error} follow \cite{Schmidt-Hieber2017}.

		\begin{theorem}
		\label{thm:expected_L2_error}
			Let $\mathbb{X} \subseteq \mathbb{R}^d$ and consider the regression function $g: \mathbb{X} \rightarrow \mathbb{R}$. Let $n \in \mathbb{N}$ and $\sigma \in \mathbb{R}_+$. Let $P$ be a distribution on $\mathbb{X}$,  with the associated samples  $( x_i,y_i )_{i = 1}^n = ( x_i, g(x_i) + \sigma \xi_i )_{i =1}^n$, 
			where $( x_i )_{i = 1}^n$ are i.i.d. random variables of distribution $P$, $(\xi_i)_{i=1}^n$ are i.i.d. standard Gaussian random variables, and $( x_i )_{i = 1}^n$ and $(\xi_i)_{i=1}^n$ are statistically independent.

			Let $\varepsilon \in ( 0, 1\slash 2 )$, and consider a class of functions $\mathcal{F}_n \subseteq L^\infty ( \mathbb{X} )$ such that   
			\begin{equation}
			\label{eq:lsr_approximation_simplified}
				\inf_{f \in \mathcal{F}_n} \nleft\| g - f \nright\|_{L^2 ( P )} \leq \varepsilon
			\end{equation}
			and an $\mathcal{F}_n$-valued random variable $\hat{f}_n$ satisfying 
			\begin{equation}
			\label{eq:numerical_assumption_simplified}
				\frac{1}{n} \sum_{i = 1}^n ( \hat{f}_n (x_i)  - y_i )^2 \leq \inf_{f \in \mathcal{F}_n} \biggl( \frac{1}{n} \sum_{i = 1}^n ( f (x_i) - y_i )^2 \biggr) +  \varepsilon^2,  \quad \text{a.s.}
			\end{equation}
			It holds that
				\begin{equation}
			\label{eq:generalization_error}	
				E ( \nleft\| \hat{f}_n - g  \nright\|^2_{L^2 ( P )}  ) \leq C ( 1 +  \sigma^2 + ( \Qlinfty ( g, \mathcal{F}_n ) )^2 ) \biggl( \varepsilon^2 +  \frac{\log  (N (\varepsilon^2  , \mathcal{F}_n, L^\infty ( \mathbb{X}  )  ))  + 1}{n} \biggr),
			\end{equation}
			where $C \in \mathbb{R}_+$ is an absolute constant and $\Qlinfty ( g, \mathcal{F}_n ):  =  \max \{ \nleft\| g \nright\|_{L^\infty ( \mathbb{X} )}, \sup_{f \in \mathcal{F}_n} \nleft\| f \nright\|_{L^\infty ( \mathbb{X}  )}   \}$.

			\begin{proof}
				See Appendix~\ref{sec:complexity_base_prediction_error_bound}.
			\end{proof}
		\end{theorem}

			The prediction error upper bound in Theorem~\ref{thm:expected_L2_error} relies on two assumptions.
   The first one is the approximation assumption \eqref{eq:lsr_approximation_simplified}, which states that the regression function $g$ can be approximated well by functions in $\mathcal{F}_n$. The second is the empirical risk minimization assumption \eqref{eq:numerical_assumption_simplified}, which requires $
			\hat{f}_n$, almost surely, to nearly achieve the minimal empirical risk among $\mathcal{F}_n$. 
		
		We proceed to apply Theorem~\ref{thm:expected_L2_error} to the estimation of $1$-Lipschitz functions $g \in \lip ( [0,1] )$ 
  from the associated (random) samples $( x_i, g(x_i) + \sigma \xi_i)_{i = 1}^n$ using nonparametric least squares with very deep fully-connected ReLU networks of fixed width and output truncated to the interval $[-1,1]$. Formally, the truncation is effected by applying
  the operator $\Othres_E: \mathbb{R} \rightarrow [-E, E]$, $E \in \mathbb{R}_+$,
		\begin{equation}
		\label{eq:def_thresholded}
			\Othres_E ( x ) = \max \{ -E, \min \{ E, x \} \},
		\end{equation}
        with $E = 1$ to the neural network output.
	This truncation is commonly adopted in the literature
  \cite{Schmidt-Hieber2017, Chen2019, nakada2020adaptive, Kohler2019} and it is
 quite natural given that the regression function $g$ to be estimated satisfies $\nleft\| g \nright\|_{L^\infty ( [0,1] )} \leq 1$.	
The formal result can now be stated as follows.
		\begin{corollary}
		\label{cor:estimation_example} 
			Consider the regression function $g \in \lip ( [0,1] )$. Let $n \in \mathbb{N}$ and $\sigma \in \mathbb{R}_+$. Let $P$ be a distribution on $[0,1]$,  with the associated samples  $( x_i,y_i )_{i = 1}^n = ( x_i, g(x_i) + \sigma \xi_i )_{i =1}^n$, where $( x_i )_{i = 1}^n$ are i.i.d. random variables of distribution $P$, $(\xi_i)_{i=1}^n$ are i.i.d. standard Gaussian random variables, and $( x_i )_{i = 1}^n$ and $(\xi_i)_{i=1}^n$ are statistically independent.

			Let $C$ and $D$ be the constants specified in Lemma~\ref{lemma:approximation_lip}, and set  
			\begin{equation}
			\label{eq:def_Ln}
			 	L(n) :=  \lceil 2 (D+1) (C+1)^{1\slash 2} n^{1\slash 6} \rceil, \quad  \mathcal{F}_n := \Othres_1 \circ \mathcal{R} ( 1, \lceil D + 1 \rceil, L(n), 1).
			\end{equation}
			Let $\hat{f}_n$  be the empirical risk minimizer\footnote{
   The existence of the minimizer is argued in Section~\ref{subs:existence_empirical_risk}. For simplicity of exposition, we assume that the minimizer can be identified exactly, thereby ignoring the impact of suboptimality of the optimization algorithm employed. This simplification is common in the literature, see e.g. \cite{Kohler2019, nakada2020adaptive}. We note, however, that Theorem~\ref{thm:expected_L2_error} can accommodate cases where minimization is accomplished only approximately.}  in $\mathcal{F}_n$, i.e., 
			\begin{equation}
			\label{eq:assumption_H1_minimum}
				\frac{1}{n} \sum_{i = 1}^n ( \hat{f}_n (x_i)  - y_i )^2 = \inf_{f \in \mathcal{F}_n} \biggl( \frac{1}{n} \sum_{i = 1}^n ( f (x_i) - y_i )^2 \biggr), \quad \text{a.s.}
			\end{equation}
			Then,  
			\begin{equation}
			\label{eq:estimation_upper_bound_neural_network}
				E ( \nleft\| \hat{f}_n - g \nright\|^2_{L^2 ( P ) }  ) \leq K(\sigma) n^{ - 2/3},
			\end{equation}
			where $K(\sigma)$ is a constant depending on $\sigma$ only.
			\begin{proof}
				We apply Theorem~\ref{thm:expected_L2_error} with $\mathbb{X} = [0,1]$ and choose $\varepsilon \in (0,1/2)$ such that the prerequisites \eqref{eq:lsr_approximation_simplified} and \eqref{eq:numerical_assumption_simplified} are satisfied. To this end, we first apply Lemma~\ref{lemma:approximation_lip} with $W = \lceil D + 1 \rceil \geq D$, $L = L(n) = \lceil 2 (D+1) (C+1)^{1\slash 2} n^{1\slash 6} \rceil \geq D$, and $p=\infty$ to obtain
				\begin{equation}
				\label{eq:32_1}
				\begin{aligned}
					\mathcal{A}	( H^1 ( [0,1] ), \mathcal{R} ( 1, \lceil D + 1 \rceil, L(n), 1), \nleft\| \cdot \nright\|_{L^\infty ( [0,1] )}  ) \leq&\,   C  ( \lceil D + 1 \rceil^2 (L(n))^2 \log (\lceil D+1 \rceil)  )^{-1}\\
					\leq &\,  C (L(n))^{-2}  \\
					= &\,  C ( \lceil 2 (D+1) (C+1)^{1\slash 2} n^{1\slash 6} \rceil  )^{-2} \\
					\leq & \, \frac{1}{4} n^{- 1 \slash 3}.
				\end{aligned}
				\end{equation}
				It then follows that
				\begin{align}
					\inf_{f \in \mathcal{F}_n} \nleft\| f - g \nright\|_{L^2 ( P )} = &\, \inf_{f \in \mathcal{R} ( 1, \lceil D + 1 \rceil, L(n), 1)} \nleft\| \Othres_1 \circ f - g \nright\|_{L^2 ( P )} \label{eq:32_2_1}\\
					\leq &\, \inf_{f \in \mathcal{R} ( 1, \lceil D + 1 \rceil, L(n), 1)} \nleft\| \Othres_1 \circ f - g \nright\|_{L^\infty ([0,1])}\label{eq:32_2_2} \\
					=& \, \inf_{f \in \mathcal{R} ( 1, \lceil D + 1 \rceil, L(n), 1)} \nleft\| \Othres_1 \circ f - \Othres_1 \circ g \nright\|_{L^\infty ( [0,1] )} \label{eq:32_2_3}\\
					\leq&\, \inf_{f \in \mathcal{R} ( 1, \lceil D + 1 \rceil, L(n), 1)} \nleft\|  f -  g \nright\|_{L^\infty ( [0,1] )}\label{eq:32_2_4}\\
					\leq &\, \mathcal{A}	( \lip ( [0,1] ), \mathcal{R} ( 1, \lceil D + 1 \rceil, L(n), 1), \nleft\| \cdot \nright\|_{L^\infty ( [0,1] )}  ) \label{eq:32_2_5} \\
					\leq&\,  \frac{1}{4} n^{-1 \slash 3}, \label{eq:32_2_6} 
				\end{align}
				where in \eqref{eq:32_2_2} we used that $P$ is a distribution on $[0,1]$, \eqref{eq:32_2_3} follows from the fact that $g \in \lip ( [0,1] )$ takes values in $[-1,1]$, \eqref{eq:32_2_4} is a consequence of $\Othres_1$ being $1$-Lipschitz, and in \eqref{eq:32_2_6} we employed \eqref{eq:32_1}. We have therefore verified \eqref{eq:lsr_approximation_simplified} with 
				\begin{equation*}
					\varepsilon := \frac{1}{4} n^{-1 \slash 3}.
				\end{equation*}
				Prerequisite \eqref{eq:numerical_assumption_simplified} holds with the same $\varepsilon = \frac{1}{4} n^{-1 \slash 3}$ owing to assumption \eqref{eq:assumption_H1_minimum}.  We are now in a position to apply Theorem~\ref{thm:expected_L2_error} resulting in 
					\begin{align}
					E ( \nleft\| \hat{f}_n - g  \nright\|^2_{L^2 ( P )}  ) \leq &\,  C_1 ( 1+  \sigma^2 + ( R ( g, \mathcal{F}_n ) )^2 ) \biggl( \varepsilon^2 +  \frac{\log  (N (\varepsilon^2  , \mathcal{F}_n, {L^\infty ( [0,1]  ) }  ))  + 1}{n} \biggr) \label{eq:applying_general_theorem_1} \\
					\leq &\, C_1 ( 2+ \sigma^2) \biggl(\frac{1}{16} n^{-2\slash 3} +  \frac{\log  (N ( \frac{1}{16} n^{-2\slash 3}  , \mathcal{F}_n, {L^\infty ( [0,1]  ) }  ))  + 1}{n} \biggr), \label{eq:applying_general_theorem_2} 
				\end{align}
				where $C_1$ is the absolute constant $C$ from Theorem~\ref{thm:expected_L2_error}, and in \eqref{eq:applying_general_theorem_2} we used
    $\Qlinfty ( g, \mathcal{F}_n )  \leq 1$ which follows from $g \in \lip ( [0,1] )$ and the fact that $\mathcal{F}_n$ consists of functions that take values in $[-1,1]$.

				We next upper-bound the term  $\log  (N ( \frac{1}{16} n^{-2\slash 3}  , \mathcal{F}_n, {L^\infty ( [0,1]  ) }  ))$.  As  $\mathcal{F}_n = \Othres_1 \circ \mathcal{R} ( 1, \lceil D + 1 \rceil, L(n), 1) $ and $\Othres_1$ is $1$-Lipschitz, every $\varepsilon$-covering $\{ x_i \}_{i = 1}^N$ of $\mathcal{R} ( 1, \lceil D + 1 \rceil, L(n), 1) $ with respect to the $L^\infty ( [0,1] )$-norm induces an $\varepsilon$-covering $\{ \Othres_1 \circ x_i \}_{i = 1}^N$ of $\mathcal{F}_n$ with respect to the $L^\infty ( [0,1] )$-norm. It therefore holds that
				\begin{equation}
				\label{eq:32_3_1}
					N \biggl( \frac{1}{16} n^{-2\slash 3}  , \mathcal{F}_n, {L^\infty ( [0,1]  ) }  \biggr) \leq N \biggl( \frac{1}{16} n^{-2\slash 3}  , \mathcal{R} ( 1, \lceil D + 1 \rceil, L(n), 1), {L^\infty ( [0,1]  ) }  \biggr).
				\end{equation}
				The right-hand-side of \eqref{eq:32_3_1}  can now be upper-bounded according to 
				\begin{align}
					&\, \log \biggl(N \biggl( \frac{1}{16}\, n^{-2\slash 3}  , \mathcal{R} ( 1, \lceil D + 1 \rceil, L(n), 1), {L^\infty ( [0,1]  ) }  \biggr)\biggr) \label{eq:32_4_1}\\
					&\leq  \, C_2 \lceil D + 1 \rceil^2 L(n) \log \biggl( \frac{(\lceil D + 1 \rceil +1)^{L(n)}}{\frac{1}{16} n^{-2\slash 3}} \biggr) \label{eq:32_4_10}\\
					&=\, C_2 \lceil D + 1 \rceil^2 L(n)\biggl(L(n) \log (\lceil D + 1 \rceil +1) +\log(16\, n^{2\slash 3} )\biggr) \label{eq:32_4_2} \\
					&\leq \,  C_2 \lceil D + 1 \rceil^2 L(n)\biggl(L(n) \log (\lceil D + 1 \rceil +1) + C_3 n^{1\slash 6} \biggr) \label{eq:32_4_3} \\
					&\leq \, C_2 \lceil D + 1 \rceil^2 C_4 n^{1\slash 6} (C_4  n^{1\slash 6}  \log (\lceil D + 1 \rceil +1) + C_3 n^{1\slash 6} )   \label{eq:32_4_4}  \\
					&= \, C_5 n^{1/3}, \label{eq:32_4_5} 
				\end{align}
				where \eqref{eq:32_4_10} follows by application of \eqref{eq:upper_bound_fully_connected_bounded_output} in Theorem~\ref{thm:covering_number_upper_bound_fully_connected_bounded_weight} with $ p = \infty$, $\varepsilon = \frac{1}{16} n^{-2 \slash 3} $, $d = 1$, $W = \lceil D+1 \rceil$, $L = L ( n )$, and $B = 1$, and $C_2$ is the absolute constant $C$ from Theorem~\ref{thm:covering_number_upper_bound_fully_connected_bounded_weight}, in \eqref{eq:32_4_3} we set $C_3 := \sup_{x \in [1,\infty)} \frac{\log (16 x^{2\slash 3})}{x^{1\slash 6}} < \infty$, which is an absolute constant, in \eqref{eq:32_4_4} we used $L(n) = \lceil 2 (D+1) (C+1)^{1\slash 2} n^{1\slash 6} \rceil \leq  2 (D+1) (C+1)^{1\slash 2} n^{1\slash 6} + 1 \leq  4 (D+1) (C+1)^{1\slash 2} n^{1\slash 6}$ and let $C_4: = 4 (D+1) (C+1)^{1\slash 2}$, and in \eqref{eq:32_4_5} we set $C_5 = C_2 \lceil D + 1 \rceil^2 C_4 ( C_4  \log (\lceil D + 1 \rceil +1) + C_3 )$. Using \eqref{eq:32_3_1} and \eqref{eq:32_4_1}-\eqref{eq:32_4_5} in \eqref{eq:applying_general_theorem_1}-\eqref{eq:applying_general_theorem_2}, finally yields 
					\begin{align*}
					E ( \nleft\| \hat{f}_n - g  \nright\|^2_{L^2 ( P )}  ) \leq&\,  C_1 ( 2 +  \sigma^2) \biggl(\frac{1}{16} n^{-2\slash 3} +  \frac{C_5 n^{1\slash 3}  + 1}{n} \biggr)\\
					\leq &\, C_1 ( 2 +  \sigma^2) \biggl( \frac{1}{16} n^{-2\slash 3}  +  \frac{( C_5 + 1 ) n^{1\slash 3} }{n}\biggr) \\
					=&\, C_1 ( 2 +  \sigma^2) \biggl( \frac{1}{16} + C_5 + 1 \biggr) n^{-2 \slash 3}.
				\end{align*}
				The proof is finalized by taking $K(\sigma) = C_1 ( 2 + \sigma^2) ( \frac{1}{16} + C_5 + 1 )$.
			\end{proof}
		\end{corollary}

		The rate $n^{-2/3}$ in Corollary \ref{cor:estimation_example} is optimal \cite[Theorem 1]{Stone1982}, see also \cite[Theorem 3.2]{Gyoerfi2002} and \cite[Theorem 3]{Schmidt-Hieber2017}. In particular, compared to the corresponding best known result in the literature given by an
  upper bound of rate $(\log(n))^6 n^{- 2 \slash 3}$ \cite[Theorem 1(b)]{Kohler2019}, Corollary \ref{cor:estimation_example} disposes of the 
  $(\log(n))^6$-factor. 
    We note that \cite[Theorem 1(b)]{Kohler2019} applies to a more general class of smooth functions mapping $\mathbb{R}^d$ to $\mathbb{R}$. 
    The removal of the $(\log(n))^6$-factor carries through to Lipschitz functions on $\mathbb{R}^d$ with general $d \in \mathbb{N}$, but we do not present the details here.
  The improvement we obtain stems from the approximation result Lemma~\ref{lemma:approximation_lip} and the use of the covering number instead of VC-dimension as in \cite{Kohler2019}. 
More specifically, in the approximation of functions in $\lip ( [0,1] )$ through very deep fully-connected ReLU networks of fixed width and depth $L$, both our Lemma~\ref{lemma:approximation_lip} and \cite[Theorem 2(b)]{Kohler2019} achieve guaranteed error decay of $L^{-2}$. However, 
  \cite[Theorem 2(b)]{Kohler2019} requires networks with arbitrarily large weight-magnitude, corresponding to unbounded sets, whereas Lemma~\ref{lemma:approximation_lip} needs networks
  of weight magnitude bounded by $1$ only. This significant reduction in the size of the model class, to compact sets, makes it possible to upper-bound the prediction error through covering numbers.

			\subsection{Optimal Regression and Optimal Approximation}
			\label{ssub:optimal_estimation_and_kolmogorov_optimal_approximation}

			The optimality established in Corollary \ref{cor:estimation_example} in the previous section is a consequence of a deeper set of ideas, which we now bring to the fore and develop in a more general context. 
   This discussion will demonstrate how optimality in function approximation through ReLU networks along with their covering number behavior plays a fundamental role in attaining optimal regression. As a byproduct, we will be able to
      shed light on the specific choices for $L ( n)$ and $\mathcal{F}_n$ in Corollary~\ref{cor:estimation_example}.

			We build on the information-theoretic characterization of optimal sample complexity rates developed by Yang and Barron \cite{Yang1999}. Concretely, it is shown in \cite[Section 3.2]{Yang1999} that, for a uniformly bounded function class $\mathcal{G}$, the optimal sample complexity rate in the estimation of regression functions $g \in \mathcal{G}$ is  determined by the covering number of $\mathcal{G}$. To make these results more concrete, consider the general nonparametric regression setup introduced at the beginning of Section~\ref{sub:empirical_risk_minimizaiton}.
   The estimation of $g \in \mathcal{G}$ from random samples can now be described as the application of a mapping $
			\mathfrak{F}_n: ( \mathbb{X} \times \mathbb{R} )^n \rightarrow L^\infty ( \mathbb{X})$ that takes the samples $( x_i, y_i )_{i = 1}^n$ to the estimate $\hat{f}_n \in \mathcal{G}$. For example, in Corollary~\ref{cor:estimation_example}, $\mathfrak{F}_n ( ( x_i, y_i )_{i=1}^n )$ would be the mapping induced by the empirical risk minimizer defined according to \eqref{eq:assumption_H1_minimum}. By  \cite[Theorem 6]{Yang1999}, under a weak technical condition\footnote{Concretely, \cite[Condition 2]{Yang1999} requires that there exist an $\alpha \in (0,1)$ such that $\liminf_{\varepsilon \to 0} \log(M(\alpha \varepsilon, \mathcal{G}, L^2 ( P )) ) \slash \log(M(\varepsilon,\mathcal{G}, L^2 ( P ))) > 1$.} on the packing number of $\mathcal{G}$, it holds that, for all $
			\mathfrak{F}_n: ( \mathbb{X} \times \mathbb{R} )^n \rightarrow L^\infty ( \mathbb{X})$,
			\begin{equation}
			\label{eq:lower_bound_yang}
				\sup_{ g \in \mathcal{G} } E ( \nleft\| \mathfrak{F}_n ( ( x_i, g(x_i) + \sigma \xi_i )_{i=1}^n ) - g \nright\|^2_{L^2 ( P )}   ) \geq c(\mathcal{G}, \sigma, P) \kappa_n^2,
			\end{equation}
			where  $\kappa_n$ is  the solution to the equation 
			\begin{equation}
			\label{eq:definition_varepsilon_n}
				\kappa_n^2 =  \frac{\log ( M ( \kappa_n,\mathcal{G}, L^2 ( P )  ) )}{n},
			\end{equation} 
			and $c(\mathcal{G}, \sigma, P) \in \mathbb{R}_+$ is a constant depending on $\mathcal{G}, \sigma$, and $P$ only.
The lower bound \eqref{eq:lower_bound_yang} is achievable in the sense of the existence of an
			$\mathfrak{F}_n: ( \mathbb{X} \times \mathbb{R} )^n \rightarrow L^\infty ( \mathbb{X})$ such that
			\begin{equation}
				\sup_{ g \in \mathcal{G} } E ( \nleft\| \mathfrak{F}_n ( ( x_i, g(x_i) + \sigma \xi_i )_{i=1}^n ) - g \nright\|^2_{L^2 ( P )}   ) \leq C(\mathcal{G}, \sigma, P) \kappa_n^2,
			\end{equation}
            where $C(\mathcal{G}, \sigma, P) \in \mathbb{R}_+$ is a constant depending on $\mathcal{G}, \sigma$, and $P$ only. In summary, the optimal sample complexity rate can be characterized by the sequence $( \kappa^2_n )_{n = 1}^\infty$. We now particularize the Yang-Barron framework to
            $\mathcal{G} = \lip ( [0,1] )$ with $P$ the uniform distribution on $[0,1]$. First, note that
			\begin{equation}
			\label{eq:l2covering}
				c_1 \varepsilon^{-1}  \leq \log ( N ( \varepsilon, \lip ( [0,1] ), L^2 ( [0,1] )  ) )  \leq  \log ( M ( \varepsilon, \lip ( [0,1] ), L^2 ( [0,1] )  ) ) \leq  C_1 \varepsilon^{-1}, \, \varepsilon \in (0, \varepsilon_0),
			\end{equation}
            where $ \varepsilon_0, c_1, C_1 \in \mathbb{R}_+$ are  absolute constants.
			Here, the first inequality follows from Lemma~\ref{lem:packing_lower_bound_lipschitz} and in the second inequality we used 
   Lemma~\ref{lem:equivalence_covering_packing}.  The last inequality in \eqref{eq:l2covering} is thanks to 
   \begin{eqnarray*}
   \log ( M ( \varepsilon, \lip ( [0,1] ), L^2 ( [0,1] )  ) ) & \leq & \log ( N ( \varepsilon/2 , \lip ( [0,1] ), L^2 ( [0,1] )  ) ) \\
   & \leq & \log ( N ( \varepsilon/2, \lip ( [0,1] ),L^\infty ( [0,1] ) )) \\
   & \leq & C_1 \varepsilon^{-1},
   \end{eqnarray*} 
   where in the first inequality we again used Lemma~\ref{lem:equivalence_covering_packing}, the second inequality follows from the fact that coverings with respect to the $L^\infty ( [0,1] )$-norm are also coverings with respect to the $L^2 ( [0,1] )$-norm, and in the last inequality we used \cite[Eq. 5.12]{wainwright2019high}. 
   Next, take $n \in \mathbb{N}$ large enough\footnote{It suffices to take $n \geq \lceil \frac{8 C_1}{\varepsilon_0^3}\rceil + 1$, which we show leads to $\kappa_n < \frac{\varepsilon_0}{2}$. Suppose, for the sake of contradiction, that for such $n$, it holds that $\kappa_n \geq \frac{\varepsilon_0}{2}$. It would then follow from the monotonicity of the packing number
   that $\log ( M ( \kappa_n, \lip ( [0,1] ), L^2 ( [0,1] )  ) ) \leq \log ( M ( \frac{\varepsilon_0}{2}, \lip ( [0,1] ), L^2 ( [0,1] )  ) ) \leq 2 C_1 \varepsilon_0^{-1}$, which together with $\kappa_n^2 =  \frac{\log ( M ( \kappa_n,\lip ( [0,1] ), L^2 ( [0,1] )  ) )}{n}$ implies $\kappa_n^2 \leq \frac{2 C_1 \varepsilon_0^{-1}}{n} \leq \frac{2 C_1 \varepsilon_0^{-1}}{ \lceil 8 C_1/\varepsilon_0^3\rceil + 1} < \frac{\varepsilon_0^2}{4}  $. Hence, $\kappa_n < \frac{\varepsilon_0}{2}$, which establishes the contradiction.}  for the solution $\kappa_n$ of the equation $\kappa_n^2 =  \frac{\log ( M ( \kappa_n,\lip ( [0,1] ), L^2 ( [0,1] )  ) )}{n}$ to satisfy $\kappa_n < \varepsilon_0$. We then have $c_1 \kappa_n^{-1} \leq \log ( M ( \kappa_n,\lip ( [0,1] ), L^2 ( [0,1] )  ) ) \leq C_1 \kappa_n^{-1}$, which implies   $ \frac{c_1 \kappa_n^{-1}}{n} \leq \kappa_n^2 \leq \frac{C_1 \kappa_n^{-1}}{n}$, and hence 
			\begin{equation}
			\label{eq:kn_scale}
				c_1^{2\slash 3}  n^{-2 \slash 3} \leq \kappa_n^2 \leq C_1^{2\slash 3}  n^{-2 \slash 3},
			\end{equation}
			thereby recovering the optimal rate $n^{-2\slash 3}$ mentioned in Corollary \ref{cor:estimation_example}. 

			We proceed to derive sufficient conditions for a sequence of estimators to achieve optimal sample complexity rate. These conditions are general in the sense of the estimators not having to be neural networks and include, e.g., sparse dictionary approximation \cite{Donoho1993,Donoho1996,Grohs2015}. 

			\begin{corollary}
			\label{coro:optimal}
   Let $\mathbb{X} \subseteq \mathbb{R}^d$ and 
				consider the class $\mathcal{G}$ of regression functions mapping $\mathbb{X}$ to $\mathbb{R}$. Let $g \in \mathcal{G}$, $n \in \mathbb{N}$, and $\sigma \in \mathbb{R}_+$. Let $P$ be a distribution on $\mathbb{X}$,  with the associated samples  $( x_i,y_i )_{i = 1}^n = ( x_i, g(x_i) + \sigma \xi_i )_{i =1}^n$, where $( x_i )_{i = 1}^n$ are i.i.d. random variables of distribution $P$, $(\xi_i)_{i=1}^n$ are i.i.d. standard Gaussian random variables, and $( x_i )_{i = 1}^n$ and $(\xi_i)_{i=1}^n$ are statistically independent. 

				Let $\varepsilon_n \in ( 0, 1\slash 2 )$ and consider a class of functions $\mathcal{F}_n \subseteq L^\infty ( \mathbb{X} )$ such that  
				\begin{equation}
					\mathcal{A} ( \mathcal{G}, \mathcal{F}_n, \nleft\| \cdot \nright\|_{L^2 ( P )} ) \leq\, \varepsilon_n. \label{eq:sufficient_condition_2} 
				\end{equation}
				Let $\hat{f}_n$ be the empirical risk minimizer in $\mathcal{F}_n$, i.e.,
				\begin{equation}
				\label{eq:sufficient_condition_4}
					\frac{1}{n} \sum_{i = 1}^n ( \hat{f}_n (x_i)  - y_i )^2 = \inf_{f \in \mathcal{F}_n} \biggl( \frac{1}{n} \sum_{i = 1}^n ( f (x_i) - y_i )^2 \biggr), \quad \text{a.s.} 
				\end{equation}
				Then, 
					\begin{equation}
					\label{eq:resulting_optimal_upper_bounds}
						E ( \nleft\| \hat{f}_n - g  \nright\|^2_{L^2 ( P )}  ) \leq C ( 1 + \sigma^2 + ( \Qlinfty ( \mathcal{G}, \mathcal{F}_n ) )^2 ) \biggl( \frac{\varepsilon_n^2}{\kappa_n^2} + K ( \mathcal{G}, \mathcal{F}_n,  \varepsilon_n, \kappa_n, P ) \biggr) \kappa_n^2,
					\end{equation}
				where $C \in \mathbb{R}_+$ is an absolute constant, $\Qlinfty ( \mathcal{G}, \mathcal{F}_n):  =  \max \{ \sup_{h \in \mathcal{G}} \nleft\| h \nright\|_{L^\infty ( \mathbb{X} )}, \sup_{f \in \mathcal{F}_n} \nleft\| f \nright\|_{L^\infty ( \mathbb{X}  )}   \}$, $\kappa_n$ is the solution to \eqref{eq:definition_varepsilon_n}, and 
				\begin{equation}
				\label{eq:resulting_upper_bound}
					K ( \mathcal{G}, \mathcal{F}_n,  \varepsilon_n, \kappa_n, P ) = \frac{\log  (N (\varepsilon_n^2  , \mathcal{F}_n, L^\infty ( \mathbb{X}  )  ))  + 1}{\log ( M ( \kappa_n,\mathcal{G}, L^2 ( P )  ) )}.
				\end{equation}

				 \begin{proof}
				 	We have 
						\begin{align}	
						&\,E ( \nleft\| \hat{f}_n - g  \nright\|^2_{L^2 ( P )}  ) \label{eq:minimax_sample_1}\\
						&\leq\, C ( 1+ \sigma^2 + ( \Qlinfty ( g, \mathcal{F}_n ) )^2 ) \biggl( \varepsilon_n^2 +  \frac{\log  (N (\varepsilon_n^2  , \mathcal{F}_n, L^\infty ( \mathbb{X}  )  ))  + 1}{n} \biggr) \label{eq:minimax_sample_2}\\
						&\leq\,  C ( 1 + \sigma^2 + ( \Qlinfty ( \mathcal{G}, \mathcal{F}_n ) )^2 ) \biggl( \varepsilon_n^2 +  \frac{\log  (N (\varepsilon_n^2  , \mathcal{F}_n, L^\infty ( \mathbb{X}  )  ))  + 1}{n} \biggr)\label{eq:minimax_sample_3}\\
						&=\,  C ( 1 + \sigma^2 + ( \Qlinfty ( \mathcal{G}, \mathcal{F}_n ) )^2 ) \biggl( \varepsilon_n^2 +  K ( \mathcal{G}, \mathcal{F}_n,  \varepsilon_n, \kappa_n, P ) \cdot \frac{\log ( M ( \kappa_n,\mathcal{G}, L^2 ( P )  ) )}{n} \biggr)\label{eq:minimax_sample_4}\\
						&=\, C ( 1 + \sigma^2 + ( \Qlinfty ( \mathcal{G}, \mathcal{F}_n ) )^2 ) \biggl( \frac{\varepsilon_n^2}{\kappa_n^2} + K ( \mathcal{G}, \mathcal{F}_n,  \varepsilon_n, \kappa_n, P ) \biggr) \kappa_n^2, \label{eq:minimax_sample_5}
					\end{align}
					where \eqref{eq:minimax_sample_2} follows from  Theorem~\ref{thm:expected_L2_error} with $\varepsilon = \varepsilon_n$ and the prerequisites \eqref{eq:lsr_approximation_simplified} and \eqref{eq:numerical_assumption_simplified} satisfied thanks to \eqref{eq:sufficient_condition_2} and \eqref{eq:sufficient_condition_4}, respectively,
     in \eqref{eq:minimax_sample_3} we used $\Qlinfty ( g, \mathcal{F}_n ) \leq  \Qlinfty ( \mathcal{G}, \mathcal{F}_n )$,
     and \eqref{eq:minimax_sample_5} is by \eqref{eq:definition_varepsilon_n}.
				 \end{proof}
			\end{corollary}
			Assumption \eqref{eq:sufficient_condition_2} in Corollary~\ref{coro:optimal} is similar to the approximation assumption \eqref{eq:lsr_approximation_simplified} in Theorem~\ref{thm:expected_L2_error}, only here we need to control the worst-case error 
   over the entire class $\mathcal{G}$ of regression functions while Theorem~\ref{thm:expected_L2_error} pertains to a fixed $g \in \mathcal{G}$.
   Assumption \eqref{eq:sufficient_condition_4} corresponds to the empirical risk minimization condition
   \eqref{eq:numerical_assumption_simplified} in Theorem~\ref{thm:expected_L2_error}, with the qualification that here we require exact empirical risk minimization while Theorem~\ref{thm:expected_L2_error} allows for an additive slack term, given by $\varepsilon^2$.

Thanks to the Yang-Barron lower bound \eqref{eq:lower_bound_yang}, we can now conclude from Corollary~\ref{coro:optimal}, specifically from \eqref{eq:resulting_optimal_upper_bounds}, that for a sequence $\hat{f}_n$, $n \in \mathbb{N}$, of empirical risk minimizers to be optimal 
(up to constant factors), it suffices to meet the conditions \eqref{eq:sufficient_condition_2} and \eqref{eq:sufficient_condition_4} and have the quantity
			\begin{equation}
			\label{eq:ratio_optimal}
				C ( 1+  \sigma^2 + ( \Qlinfty ( \mathcal{G}, \mathcal{F}_n ) )^2 ) \biggl( \frac{\varepsilon_n^2}{\kappa_n^2} + K ( \mathcal{G}, \mathcal{F}_n,  \varepsilon_n, \kappa_n, P ) \biggr)
			\end{equation}
be upper-bounded by a constant not depending on $n$. 

We next illustrate how the choices made in Corollary~\ref{cor:estimation_example}, specifically for $L(n)$ and $\mathcal{F}_{n}$ in \eqref{eq:def_Ln}, 
meet all these conditions
thereby proving that the estimation of $1$-Lipschitz functions can be accomplished through very deep ReLU networks in
an information-theoretically optimal manner.
Accordingly, we take $\mathcal{G}=\lip ( [0,1] )$, let $L(n)$ and $\mathcal{F}_{n}$ be as in \eqref{eq:def_Ln}, and, for concreteness,
take $P$ to be the uniform distribution on $[0,1]$.
First, we verify that the assumptions
in Corollary~\ref{cor:estimation_example} imply \eqref{eq:sufficient_condition_2} and \eqref{eq:sufficient_condition_4}. 
    Condition \eqref{eq:sufficient_condition_4} is identical to \eqref{eq:assumption_H1_minimum} in Corollary~\ref{cor:estimation_example}. The approximation assumption \eqref{eq:sufficient_condition_2} is satisfied with 
			\begin{equation}
			\label{eq:varepsilon_scale}
				\varepsilon_n =  \frac{1}{4} n^{- 1 \slash 3}
			\end{equation}
			as \eqref{eq:32_2_1}-\eqref{eq:32_2_6} holds for all $g \in \lip ( [0,1] )$. 

			We proceed to upper-bound the individual terms in \eqref{eq:ratio_optimal} and start by noting that
			\begin{equation}
				\Qlinfty ( \mathcal{G},\allowbreak \mathcal{F}_n )  = \Qlinfty ( \lip ( [0,1] ), \mathcal{F}_n )  = \max \biggl\{ \sup_{h \in \lip ( [0,1] )} \nleft\| h \nright\|_{L^\infty ( \mathbb{X} )}, \sup_{f \in \mathcal{F}_n} \nleft\| f \nright\|_{L^\infty ( \mathbb{X}  )}   \biggr\} \leq 1,
			\end{equation}
			thanks to the truncation operation $\mathcal{T}_1$ in the definition of $\mathcal{F}_n$. Second, we need to verify that $\mathcal{F}_{n}$ is such that $\varepsilon_n^2$ is balanced with $\kappa_n^2$ in the sense of $\frac{\varepsilon_n^2}{\kappa_n^2}$ being upper-bounded by a constant independent of $n$.
            This follows from
			\begin{equation}
				\frac{\varepsilon_n^2}{\kappa_n^2} \leq \frac{\frac{1}{16}n^{-2\slash 3}}{c_1^{2\slash 3} n^{-2\slash 3}} = \frac{1}{16 c_1^{2\slash 3}},
			\end{equation}
			where we used
   \eqref{eq:kn_scale} and \eqref{eq:varepsilon_scale}.

			To upper-bound $K ( \lip ( [0,1] ), \mathcal{F}_n,  \varepsilon_n, \kappa_n, P )$, we first factorize according to 
			\begin{align}
				K ( \lip ( [0,1] ), \mathcal{F}_n,  \varepsilon_n, \kappa_n, P ) &=\frac{\log  (N (\varepsilon_n^2  , \mathcal{F}_n, L^\infty ( [0,1]  )  ))  + 1}{\log ( M ( \kappa_n,\lip ( [0,1] ), L^2 ( [0,1] )  ) )} \label{eq:factorization_1}\\
				&= \frac{\log  (N (\varepsilon_n^2  , \mathcal{F}_n, L^\infty ( [0,1]  )  ))  + 1}{\log  (N (\varepsilon_n  , \mathcal{R} ( 1, \lceil D + 1 \rceil, L(n), 1), L^2 ( [0,1]  )  ))} \nonumber\\
				&\quad \cdot  \frac{\log  (N (\varepsilon_n  , \mathcal{R} ( 1, \lceil D + 1 \rceil, L(n), 1), L^2 ( [0,1]  )  ))}{\log ( N ( 4 \varepsilon_n,\lip ( [0,1] ), L^2 ( [0,1] )  ) )} \nonumber\\
				&\quad  \cdot  \frac{\log ( N ( 4 \varepsilon_n,\lip ( [0,1] ), L^2 ( [0,1] )  ) )}{\log ( M ( \kappa_n,\lip ( [0,1] ), L^2 ( [0,1] )  ) )}\label{eq:factorization_2}
			\end{align}
			and then treat the three factors in \eqref{eq:factorization_2} individually.
      
			\begin{enumerate}[label=(\roman*)]
				\item For the numerator of the first factor in \eqref{eq:factorization_2}, we have 
				\begin{align}
					&\,  \log  (N (\varepsilon_n^2  , \mathcal{F}_n, L^\infty ( [0,1]  )  )) + 1 \label{eq:first_factor_0}\\
					&=\,  \log  (N (\varepsilon_n^2  , \Othres_1 \circ \mathcal{R} ( 1, \lceil D + 1 \rceil, L(n), 1), L^\infty ( [0,1]  )  )) + 1\\
					&\, \leq  \log  (N (\varepsilon_n^2  , \mathcal{R} ( 1, \lceil D + 1 \rceil, L(n), 1), L^\infty ( [0,1]  )  )) + 1 \label{eq:first_factor_1}\\ 
					&\, \leq   C_2 \lceil D + 1 \rceil^2 L(n) \log \biggl( \frac{(\lceil D + 1 \rceil+1)^{L(n)}}{\varepsilon_n^2} \biggr) + 1 \label{eq:first_factor_2} \\
					&\, \leq   ( C_2+1 ) \lceil D + 1 \rceil^2 L(n) \log \biggl( \frac{(\lceil D + 1 \rceil+1)^{L(n)}}{\varepsilon_n^2} \biggr), \label{eq:first_factor_22}
				\end{align}
				where \eqref{eq:first_factor_1} follows from the same argument as used to arrive at \eqref{eq:32_3_1}, and in \eqref{eq:first_factor_2} we employed the covering number upper bound in Theorem~\ref{thm:covering_number_upper_bound_fully_connected_bounded_weight}. For the denominator $\log  (N (\varepsilon_n  , \mathcal{R} ( 1, \lceil D + 1 \rceil, L(n), 1), L^2 ( [0,1]  )  ))$, we apply the covering number lower bound\footnote{The application of the lower bound requires $W,L \geq 60$, which holds thanks to $L(n) = \lceil 2 (D+1) (C+1)^{1\slash 2} n^{1\slash 6} \rceil$, $W = \lceil D + 1 \rceil$, with $D$ the constant specified in Lemma~\ref{lemma:approximation_lip} satisfying $D \geq 60$ owing to the specifics in the proof of \cite[Theorem 3.1]{FirstDraft2022}.
    } in Theorem~\ref{thm:covering_number_upper_bound_fully_connected_bounded_weight}, resulting in 
				\begin{align}
					&\log  (N (\varepsilon_n  , \mathcal{R} ( 1, \lceil D + 1 \rceil, L(n), 1), L^2 ( [0,1]  )  )) \label{eq:first_factor_3}\\
					&\geq c_2 \lceil D + 1 \rceil^2 L(n) \log \biggl( \frac{(\lceil D + 1 \rceil+1)^{L(n)}}{\varepsilon_n} \biggr). \label{eq:first_factor_4}
				\end{align}
Dividing \eqref{eq:first_factor_22} by \eqref{eq:first_factor_4}, it follows that
				\begin{equation}
				\label{eq:first_factor_5}
					\frac{C_2+1}{c_2} \log \biggl( \frac{(\lceil D + 1 \rceil+1)^{L(n)}}{\varepsilon_n^2} \biggr) \Bigg\slash \log \biggl( \frac{(\lceil D + 1 \rceil+1)^{L(n)}}{\varepsilon_n} \biggr) \leq 2 \frac{C_2 + 1}{c_2}, 
				\end{equation}
				where we used $(\lceil D + 1 \rceil +1)^{L(n)} < (\lceil D + 1 \rceil +1)^{2 L(n)}$.

				\item For the second factor in \eqref{eq:factorization_2}, the numerator
    can be upper-bounded according to 
				\begin{align}
					&\log  (N (\varepsilon_n  , \mathcal{R} ( 1, \lceil D + 1 \rceil, L(n), 1), L^2 ( [0,1]  )  )) \label{eq:dom_1}\\
					&\leq \log  (N ( \varepsilon_n , \mathcal{R} ( 1, \lceil D + 1 \rceil, L(n), 1), L^\infty ( [0,1]  )  )) \label{eq:dom_2} \\
					& \leq \log  \biggl(N \biggl(\frac{1}{16}\, n^{-2\slash 3}  , \mathcal{R} ( 1, \lceil D + 1 \rceil, L(n), 1), L^\infty ( [0,1]  )  \biggr)\biggr) \label{eq:dom_3}\\
					& \leq C_5 n^{1/3},  \label{eq:dom_4}
				\end{align}
				where \eqref{eq:dom_2} follows from the fact that coverings with respect to the $L^\infty ( [0,1] )$-norm are also coverings 
        with respect to the $L^2 ( [0,1] )$-norm, \eqref{eq:dom_3} is by $\frac{1}{16}\, n^{-2\slash 3} < \varepsilon_n $, and
        \eqref{eq:dom_4} follows from \eqref{eq:32_4_1}-\eqref{eq:32_4_5}.
        The denominator in the second factor in \eqref{eq:factorization_2}
can be lower-bounded by \eqref{eq:l2covering} according to
				\begin{equation}
				\label{eq:asdfqwaeoi}	
					\log ( N ( 4 \varepsilon_n,\lip ( [0,1] ), L^2 ( [0,1] )  ) ) \geq c_1 ( 4 \varepsilon_n )^{-1} = c_1 n^{1\slash 3}.
				\end{equation}

				Using \eqref{eq:dom_4} and \eqref{eq:asdfqwaeoi}, we finally obtain 
				\begin{equation}
				\label{eq:optimal_reverse}
					\frac{\log  (N (\varepsilon_n  , \mathcal{R} ( 1, \lceil D + 1 \rceil, L(n), 1), L^2 ( [0,1]  )  ))}{\log ( N ( 4 \varepsilon_n,\lip ( [0,1] ), L^2 ( [0,1] )  ) )} \leq \frac{C_5 n^{1/3}}{ c_1 n^{1\slash 3}} \leq \frac{C_5}{c_1}.
				\end{equation}

				\item For the third factor in \eqref{eq:factorization_2}, we have
				\begin{equation}
				\label{eq:scaling_target_family}	
				 	\frac{\log ( N ( 4 \varepsilon_n,\lip ( [0,1] ), L^2 ( [0,1] )  ) )}{\log ( M ( \kappa_n,\lip ( [0,1] ), L^2 ( [0,1] )  ) )} \leq \frac{C_1 ( 4\varepsilon_n  )^{-1}}{c_1 \kappa_n^{-1}} \leq
      \frac{C_1^{4 \slash 3}}{c_1},
				 \end{equation} 
				where the first inequality follows from \eqref{eq:l2covering}
    and in the second inequality we used
    \eqref{eq:kn_scale}. Note that here we exploited the fact that $\varepsilon_{n}$ is of the same order as $\kappa_n$, namely $n^{-1/3}$.
			\end{enumerate}
            Putting (i)-(iii) together, we have shown that \eqref{eq:ratio_optimal} can, indeed, be upper-bounded by a constant not depending on $n$, thereby establishing the information-theoretic optimality of the sequence of neural network estimators $\hat{f}_n$ in Corollary~\ref{cor:estimation_example}.
This was accomplished by exploiting                        
three key properties. The first one is the $\log(1/\varepsilon_{n})$-scaling behavior of the metric entropy of the set of approximants $\mathcal{R} ( 1, \lceil D + 1 \rceil, L(n), 1)$ used to establish \eqref{eq:first_factor_0}-\eqref{eq:first_factor_5}. Specifically, this scaling behavior is crucial in the last step \eqref{eq:first_factor_5}.
The second property is the $\varepsilon^{-1}$-scaling behavior of the metric entropy of the set of regression functions $\lip ( [0,1] )$ used to arrive at  \eqref{eq:scaling_target_family}. Such a scaling behavior is common for unit balls in function spaces, see, e.g., \cite[Table 1]{deep-it-2019}. The third property, leading to \eqref{eq:optimal_reverse}, states that the metric entropy of the set of regression functions $\lip ( [0,1] )$ has to be balanced with that of the set of approximants $\mathcal{R} ( 1, \lceil D + 1 \rceil, L(n), 1)$. 
We note that this balancing property can be relaxed to
\begin{equation}
\label{eq:weak_sufficiently}
                \frac{\log( N(\varepsilon, \mathcal{F}_\varepsilon, L^2 (\mathbb{X})))}{\log( N(4 \varepsilon, \mathcal{G}, L^2 (\mathbb{X})))} \leq  r_\mathcal{G}(\log( N(4 \varepsilon, \mathcal{G}, L^2 (\mathbb{X})))),
            \end{equation} 
where $\mathcal{G}$ is the class of regression functions under consideration, $\mathcal{F}_\varepsilon$ denotes a set of approximants satisfying $\mathcal{A} (\mathcal{G}, \mathcal{F}_\varepsilon, \|\cdot\|) \leq \varepsilon$, and
$r_\mathcal{G}: \mathbb{R}_+ \rightarrow \mathbb{R}_+$ is such that $\limsup_{x \to + \infty} \frac{\log(r_\mathcal{G} (x))}{\log(x)} = 0$. 
In this case the upper bound \eqref{eq:resulting_optimal_upper_bounds} in Corollary~\ref{coro:optimal} would still guarantee optimal sample
complexity rate, but would exhibit an additional logarithmic factor.
An example of such a behavior, albeit in the context of optimal approximation (through neural networks) rather than regression, can be found in  
   \cite{deep-it-2019} where it is referred to as 
      Kolmogorov-Donoho optimal approximation, defined through $\mathcal{A} (\mathcal{G}, \mathcal{F}_\varepsilon, \|\cdot\|_{L^2 (\mathbb{X})}) \leq \varepsilon$ and \eqref{eq:weak_sufficiently} holding concurrently. Note that $\mathcal{A} (\mathcal{G}, \mathcal{F}_\varepsilon, \|\cdot\|_{L^2 (\mathbb{X})}) \leq \varepsilon$ implies $\frac{\log( N(\varepsilon, \mathcal{F}_\varepsilon, L^2 (\mathbb{X})))}{\log( N(4 \varepsilon, \mathcal{G}, L^2 (\mathbb{X})))} \geq 1$.
               Although not explicitly mentioned and formally established, Kolmogorov-Donoho optimal neural network approximation is what leads to optimal sample complexity rates up to a logarithmic factor in \cite{Schmidt-Hieber2017, Chen2019, nakada2020adaptive}.
We remark that Kolmogorov-Donoho optimality can also be achieved through sparse dictionary approximation \cite{Donoho1993,Donoho1996,Grohs2015}. Thanks to its generality, Corollary~\ref{coro:optimal} allows to conclude that taking the set of approximants $\mathcal{F}_{n}$ to be obtained by sparse dictionary approximation, sample complexity rate-optimal regression of $\mathcal{G}$ up to a logarithmic factor is guaranteed whenever $\mathcal{F}_{n}$ achieves Kolmogorov-Donoho optimal approximation of $\mathcal{G}$.

\section{Sparse Networks with Uniformly Bounded Weights}
\label{sec:sparsely_connected_relu_networks} 

Sparse neural networks exhibit a connectivity $s$ that is (typically much) smaller than the total number of weights $L(W^2+W)$ in the network. 
In practical applications sparsity is often enforced with the goal of minimizing the amount of memory needed to store the network.
The approximation-theoretic limits of sparse neural networks have been studied widely in the literature, starting with 
\cite{YAROTSKY2017103, deep-approx-18,PETERSEN2018296} and further considered, both in the context of function approximation and regression, in \cite{deep-it-2019, Schmidt-Hieber2017, Chen2019, nakada2020adaptive, GUHRING2021107}. 
In the spirit of Theorem~\ref{thm:covering_number_upper_bound_fully_connected_bounded_weight}, we next characterize the covering numbers of sparse ReLU networks with uniformly bounded weights.

	\begin{theorem}
		\label{thm:covering_number_sparse}
		Let $p \in [1,\infty]$, $d,W,L,s \in \mathbb{N}$, $B, \varepsilon \in \mathbb{R}_+$, with $s \geq \max \{ W,L \}$ and $B \geq 1$. Then, for
 all $\varepsilon \in (0,1\slash 2)$, we have
		\begin{equation}
		\label{eq:upper_bound_sparse}
			\log ( N(\varepsilon,\mathcal{R}(d,W,L,B,s),L^p ( [0,1]^d ) ) ) \leq C \min \{ s, W^2 L \}\log \biggl( \frac{(W+1)^L B^{L}}{\varepsilon} \biggr),
		\end{equation}
		where $C \in \mathbb{R}_+$ is an absolute constant. Moreover, there exist absolute constants $c,D \in \mathbb{R}_+$ such that, if, in addition, $W,L \geq 60$ and $s \geq D d^2 L$, then, for all $\varepsilon \in (0,1\slash 4)$, it holds that
		\begin{equation}
		\label{eq:lower_bound_sparse}
			\log ( N(\varepsilon,\mathcal{R}(d,W,L,B,s),L^p ( [0,1]^d ) ) ) \geq c \min \{ s, W^2 L \} \log \biggl( \frac{( \widetilde{W} + 1 )^{L} B^L}{\varepsilon} \biggr), 
		\end{equation}
		where $\widetilde{W} := \min \{ \lceil \sqrt{\frac{s}{L} }\,\rceil, W \}$.

		\begin{proof}
			The proof is provided in Appendix~\ref{sec:proof_of_results_in_section_sec:sparsely_connected_relu_networks}. 
		\end{proof}
	\end{theorem}

 We first note that the condition $s \geq \max \{ W,L \}$ comes without loss of generality, for the following reasons. A network with $s < L$ necessarily has one or more layers with weights equal to $0$
 and hence realizes a constant function, which could equivalently be obtained by a single-layer network. Likewise, for $s < W$, there would be nodes that can
 be removed without affecting the network's input-output relation.

 We proceed to discuss the effect of the connectivity parameter $s$ on the covering number bounds \eqref{eq:upper_bound_sparse} and \eqref{eq:lower_bound_sparse}. First, recall that networks in $\mathcal{N} ( d,W,L, B)$ have no more than $L(W^2+W) \leq 2 W^2 L$ weights. Comparing the factors in front of the logarithms in \eqref{eq:upper_bound_sparse} and \eqref{eq:lower_bound_sparse} to those in 
  the corresponding bounds \eqref{eq:upper_bound_fully_connected_bounded_output} and \eqref{eq:lower_bound_fully_connected_bounded_output_main} for the fully-connected case, hence suggests an interpretation of $\min \{ s, W^2 L \}$ as the effective connectivity. 
 An important difference between the bounds for the fully-connected case in Theorem~\ref{thm:covering_number_upper_bound_fully_connected_bounded_weight} and those in Theorem~\ref{thm:covering_number_sparse} is the appearance of the quantity 
 $\widetilde{W} = \min \{ \lceil \sqrt{\frac{s}{L} }\, \rceil, W \}$ inside the logarithm in the lower bound \eqref{eq:lower_bound_sparse}. For $\lceil \sqrt{\frac{s}{L}}\rceil < W$, we will hence have a loss of tightness, albeit only of logarithmic order, between the bounds \eqref{eq:upper_bound_sparse} and \eqref{eq:lower_bound_sparse}. The term $\frac{s}{L}$ can be interpreted as the average connectivity per layer, a quantity also appearing in the VC-dimension lower bound \cite[Equation (2)]{bartlett2019nearly} for ReLU networks. 

 We finally note that the fundamental limits of sparse ReLU networks when used in neural network transformation, function approximation, and optimal regression, can be inferred by following the playbooks in Sections \ref{sub:fundamental_limit_covering_number} and \ref{sub:empirical_risk_minimizaiton}, but with the covering number behavior as quantified by Theorem~\ref{thm:covering_number_sparse}.

\section{\raggedright Fully-connected Networks with Base-$2$ Quantized Weights}
\label{sec:covering_number_and_degeneration_phenomenon_of_relu_networks_with_quantized_weights}

	In this section, we characterize the covering number of ReLU networks with base-$2$ quantized weights, i.e., we consider the set $\mathcal{R}_{\mathbb{Q}^a_b} ( d,W, L)$ with  $\mathbb{Q}^a_b := (-2^{a+1},2^{a+1}) \cap 2^{-b} \mathbb{Z}$, where $a,b \in \mathbb{N}$.  
 The motivation for analyzing this setting stems from the fact that neural networks stored on electronic devices necessarily have their weights encoded into
finite-length bitstrings. 
	For ease of presentation, we simplify notation according to
	\begin{align*}
		\mathcal{N}_b^a (d,W,L) :=&\, \mathcal{N}_{\mathbb{Q}^a_b} ( d,W, L) \\
		=&\, \{  \Phi \in \mathcal{N} ( d ) :\mathcal{W} (  \Phi ) \leq W, \ \mathcal{L} ( \Phi ) \leq L, \coef ( \Phi ) \subseteq \mathbb{Q}_b^a \},\\
		\mathcal{R}_b^a (d,W,L)  :=&\, \{ R(\Phi): \Phi \in \mathcal{N}_b^a (d,W,L) \}.
	\end{align*}
	To the best of our knowledge, there are no results in the literature on covering numbers of ReLU networks with base-2 quantized weights.
 Here, we report covering number lower and upper bounds that are tight.

		\begin{theorem}
			\label{thm:networks_quantized_weights_covering_number_bound}
			Let $p \in [1,\infty]$, $d,W,L,a,b \in \mathbb{N}$. For all $\varepsilon \in ( 0,1\slash 2 )$, it holds that
			\begin{equation}
			\label{eq:networks_quantized_weights_covering_number_upper_bound}
				\log ( N(\varepsilon,\mathcal{R}^a_b ( d,W,L ),{L^p ( [0,1]^d )} ) ) \leq C W^2 L \cdot  \min \biggl\{  (a+b), \log \biggl(\frac{(W+1)^L 2^{aL}}{\varepsilon}\biggr) \biggr\}, 
			\end{equation}
			with $C \in \mathbb{R}_+$ an absolute constant. Moreover, there exist absolute constants $c,D,E \in \mathbb{R}_+$ such that, for $W,L \geq D$ with $L(a+b) \geq E \,\log(W)$, and all $\varepsilon \in ( 0,\frac{1}{100})$,
			\begin{equation}
			\label{eq:44_results_quantized}
				\log (N(\varepsilon,\mathcal{R}^a_b ( d,W,L ),L^p ( [0,1]^d ) )) \geq  
					c   W^2 L\cdot \min \biggl\{(a+b),\log \biggl(\frac{( W+1 )^L 2^{aL}}{\varepsilon}\biggr) \biggr\}. 
			\end{equation}
		\end{theorem}

		\begin{proof}
			We start by proving the upper bound. Arbitrarily fix $\varepsilon \in ( 0, 1\slash 2 )$. As $\mathcal{R}^a_b ( d,W,L )$ is an $\varepsilon$-covering of itself, we have $N(\varepsilon,\mathcal{R}^a_b ( d,W,L ),L^p ( [0,1]^d ) ) \leq \nleft|  \mathcal{R}^a_b ( d,W,L )  \nright|  $ and hence
			\begin{align}
				\log ( N(\varepsilon,\mathcal{R}^a_b ( d,W,L ),L^p ( [0,1]^d ) )) \leq&\, \log ( \nleft|  \mathcal{R}^a_b ( d,W,L )  \nright| )  \label{eq:t42_u_10}\\ 
				 \leq&\, \log ( \nleft|  \mathcal{N}^a_b ( d,W,L )  \nright| ) \\
				\leq&\, 5 W^2 L \log (\nleft| \mathbb{Q}_b^a\nright| )  \label{eq:t42_u_11} \\
				<&\,  10 W^2 L ( a+ b ), \label{eq:t42_u_12} 
			\end{align}
			where \eqref{eq:t42_u_11} follows from Lemma~\ref{lem:counting_cardinality} with $\mathbb{A}  = \mathbb{Q}_b^a$, and \eqref{eq:t42_u_12} is by $\nleft| \mathbb{Q}_b^a \nright| = \nleft| (-2^{a+1},2^{a+1}) \cap 2^{-b} \mathbb{Z}  \nright| \leq 2^{a+2} \cdot 2^{b} \leq 2^{2 ( a+b )}$. Moreover, as $\mathcal{R}_b^a ( d,W,L ) \subseteq \mathcal{R} ( d,W,L, 2^{a+1} )$, it holds that 
			\begin{equation}
			\label{eq:61_prerequisite_0}
				\mathcal{A} ( \mathcal{R}_b^a ( d,W,L ),  \mathcal{R} ( d,W,L, 2^{a+1} ), \nleft\| \cdot \nright\|_{L^p ( [0,1]^d )}   ) =0.
			\end{equation}
			Application of Proposition~\ref{lem:cardinality_approximation_class} with $\mathcal{G} = \mathcal{R}^a_b ( d,W,L )$, $\mathcal{F} = \mathcal{R} ( d,W,L, 2^{a+1} )$, $\delta  = \nleft\|\cdot\nright\|_{L^p ( [0,1]^d )}$, $\varepsilon$ replaced by $\varepsilon \slash 4$, and the prerequisite \eqref{eq:improper_cover} satisfied thanks to \eqref{eq:61_prerequisite_0}, now yields 
			\begin{equation}
			\label{eq:quantized_to_bounded}
				N ( \varepsilon, \mathcal{R}^a_b (d ,W,L ), {L^p ( [0,1]^d )})  \leq   N ( \varepsilon \slash 4 , \mathcal{R} ( d,W,L, 2^{a+1} ),{L^p ( [0,1]^d )} ).
			\end{equation}
			The logarithm of the right-hand-side of \eqref{eq:quantized_to_bounded} can be upper-bounded according to 
			\begin{align}
				\log ( N ( \varepsilon \slash 4 , \mathcal{R} ( d,W,L, 2^{a+1} ),{L^p ( [0,1]^d )} ) ) \leq&\,  C_1 W^2 L \log \biggl( \frac{(W+1)^L 2^{(a+1)L}}{\varepsilon \slash 4} \biggr) \label{eq:t42up_1} \\
				< &\, 3 C_1 W^2 L \log \biggl( \frac{(W+1)^L 2^{aL}}{\varepsilon} \biggr), \label{eq:t42up_2}
			\end{align}
			where in \eqref{eq:t42up_1} we applied Theorem~\ref{thm:covering_number_upper_bound_fully_connected_bounded_weight}, \eqref{eq:t42up_2} follows from $\frac{(W+1)^L 2^{(a+1)L}}{\varepsilon \slash 4} < \frac{(W+1)^L 2^{3aL}}{\varepsilon \cdot \varepsilon^2} < ( \frac{(W+1)^{L} 2^{a L}}{\varepsilon} )^3$, and $C_1 \in \mathbb{R}_{+}$ is an absolute constant. Combining \eqref{eq:quantized_to_bounded} and \eqref{eq:t42up_1}-\eqref{eq:t42up_2}, establishes 
			\begin{equation}
			\label{eq:t42_u_2}
				\log ( N(\varepsilon,\mathcal{R}^a_b ( d,W,L ),L^p ( [0,1]^d ) ) ) <  3 C_1 W^2 L \log{\biggl(\frac{(W+1)^L 2^{aL}}{\varepsilon}\biggr)}.
			\end{equation}
			Putting \eqref{eq:t42_u_10}-\eqref{eq:t42_u_12} and \eqref{eq:t42_u_2} together finally yields 
			\begin{align*}
				&\log ( N(\varepsilon,\mathcal{R}^a_b ( d,W,L ),L^p ( [0,1]^d ) ) ) \\
				&< \,  \min \biggl\{10 W^2 L ( a+ b ),  3 C_1 W^2 L \log{\biggl(\frac{(W+1)^L 2^{aL}}{\varepsilon}\biggr)}  \biggr\}\\
				&< \, (10 + 3C_1  ) W^2 L \min \biggl\{ ( a+ b ), \log{\biggl(\frac{(W+1)^L 2^{aL}}{\varepsilon}\biggr)}  \biggr\}.
			\end{align*}
			Upon setting $C = 10 + 3C_1$, this concludes the proof of the upper bound. The proof of the lower bound is lengthy and hence relegated to  Appendix~\ref{sec:proof_of_proposition_prop:covering_number_lower_bound_quantized_weights}.
		\end{proof}	

	 First, we  note that the upper bound \eqref{eq:networks_quantized_weights_covering_number_upper_bound} and the lower bound \eqref{eq:44_results_quantized} are tight up to an absolute multiplicative constant. This allows us to conclude that
   the covering number of ReLU networks with base-2 quantized weights exhibits
   two regimes as a function of $\varepsilon$. Specifically,
 for $\varepsilon \geq \frac{( W+1 )^L 2^{aL}}{2^{a+b}}$, the log-term in the bounds \eqref{eq:networks_quantized_weights_covering_number_upper_bound} and \eqref{eq:44_results_quantized} is active, which renders them structurally identical to the bounds for networks with unquantized weights, as stated in Theorem~\ref{thm:covering_number_upper_bound_fully_connected_bounded_weight}.
 In this regime quantized neural networks, in terms of their covering numbers, hence behave like unquantized networks.
 On the other hand, for $\varepsilon < \frac{( W+1 )^L 2^{aL}}{2^{a+b}}$, the covering number can be sandwiched by quantities
 that are independent of $\varepsilon$ and solely determined by the parameters $W,L,a,b$ according to $ cW^{2}L(a+b) \leq  \log ( N(\varepsilon,\mathcal{R}^a_b ( d,W,L ),L^p ( [0,1]^d ) )) \leq CW^{2}L(a+b)$.
In this regime, the covering ball radius $\varepsilon$ is small enough to reveal the quantized nature of the network weights.
In summary, we have a phase-transition behavior, in terms of $\varepsilon$, between a regime
where $\mathcal{N}^a_b ( d,W,L )$ behaves like networks with weights in $\mathbb{R}$ and a regime where the quantized nature of the weights
limits the approximation capacity of $\mathcal{N}^a_b ( d,W,L )$.

  We finally note that the fundamental limits of ReLU networks with base-2 quantized weights when used in neural network transformation, function approximation, and optimal regression, can be inferred by following the playbooks in Sections \ref{sub:fundamental_limit_covering_number} and \ref{sub:empirical_risk_minimizaiton}, but with the covering number behavior as quantified by Theorem~\ref{thm:networks_quantized_weights_covering_number_bound}.

\section{Fully-connected Networks with Truncated Output}
\label{sec:covering_number_for_relu_networks_with_bounded_output}

	Fully-connected ReLU networks with unconstrained weight magnitude are prevalent in the literature \cite{Kohler2019, yarotsky2019phase, shen2019deep, shen2019nonlinear}. As the covering number of the function class $\mathcal{R}(d, W,L)$ realized by such networks is infinite, their 
  performance limits are typically characterized through the VC-dimension. It turns out, however, that when dealing with bounded functions such as the set $\lip ([0,1])$, it suffices to consider ReLU networks with truncated outputs. This allows to develop a more refined picture, namely 
    by arguing as follows.
First, note that $\mathcal{R} ( 1, W, L, 1) \subseteq \mathcal{R} ( 1, W, L)$ together with Lemma~\ref{lemma:approximation_lip}, yields
	\begin{align}
	 	\mathcal{A}	( \lip  ( [0,1] ), \mathcal{R} ( 1, W, L), \nleft\| \cdot \nright\|_{L^p ( [0,1] )}  ) \leq&\, \mathcal{A}	( \lip  ( [0,1] ), \mathcal{R} ( 1, W, L, 1 ), \nleft\| \cdot \nright\|_{L^p ( [0,1] )}  ) \label{eq:uhasdflk_1}\\
	 	\leq& \,  C  ( W^2 L^2 \log (W)  )^{-1},\label{eq:uhasdflk_2}
	\end{align} 
	for $p \in [1, \infty]$, $W,L \in \mathbb{N}$, with $W,L \geq D$, where
  $C,D \in \mathbb{R}_+$ are the absolute constants specified in Lemma~\ref{lemma:approximation_lip}.

	A corresponding lower bound for $p=\infty$, obtained through arguments based on VC-dimension, is
	\cite[Proposition 2.11]{FirstDraft2022}, \cite[Theorem 2.3]{shen2019deep}
	\begin{equation*}
		\mathcal{A}	( H^1 ( [0,1] ), \mathcal{R} ( 1, W, L )  ,\allowbreak \nleft\| \cdot \nright\|_{L^\infty ( [0,1]}  ) \geq   c \,  ( W^2 L^2 ( \log (W L) ) )^{-1},
	\end{equation*}
	with $c \in \mathbb{R}_+ $ an absolute constant. A lower bound for $p \in [1, \infty)$ 
    follows by application of \cite[Corollary 1]{achour2022general}, which implies that, for $L \geq 1$ and sufficiently large $W \in \mathbb{N}$, 
	\begin{equation*}
		\mathcal{A}	( H^1 ( [0,1] ), \mathcal{R} ( 1, W, L )  ,\allowbreak \nleft\| \cdot \nright\|_{L^{p} ( [0,1]}  ) \geq   c(p)  (W^2L^2 (\log ( W L ) )^3  )^{-1},
	\end{equation*}
	with the constant $c(p) \in \mathbb{R}_+$ depending on $p$. This lower bound is derived by upper-bounding the covering number of ReLU networks with truncated output and, as already noted in \cite{achour2022general}, can be improved in the case $p=2$. The goal of this section is to present this improvement, which will require, inter alia, an upper bound on the $L^2 (P)$-covering number, with $P$ a distribution on $\mathbb{R}^d$, of ReLU networks with unconstrained weight magnitude and truncated output. We hasten to add that the main ideas underlying this improvement have been laid out in \cite{achour2022general}.

We proceed as follows. First, note that 
	\begin{align}
		&\, \mathcal{A} ( \lip ( [0,1] ), \mathcal{R}(1,W,L), \nleft\| \cdot \nright\|_{L^2 ( [0,1] )})\label{eq:why_thresholding_class_0} \\
		&= \, \sup_{f \in \lip ( [0,1] ) } \inf_{g \in \mathcal{R}(1,W,L)} \nleft\| f - g \nright\|_{L^2 ( [0,1] )}  \\
		&\geq \, \sup_{f \in \lip ( [0,1] ) } \inf_{g \in \mathcal{R}(1,W,L)} \nleft\| \Othres_1 \circ f - \Othres_1 \circ g \nright\|_{L^2 ( [0,1] )} \label{eq:why_thresholding_class_01} \\
		&= \, \sup_{f \in \lip ( [0,1] ) } \inf_{g \in \mathcal{R}(1,W,L)} \nleft\| f - \Othres_1 \circ g \nright\|_{L^2 ( [0,1] )}\label{eq:why_thresholding_class} \\
		&= \, \mathcal{A} ( \lip ( [0,1] ), \Othres_1 \circ \mathcal{R}(1,W,L), \nleft\| \cdot \nright\|_{L^2 ( [0,1] )}), \label{eq:why_thresholding_class_1}
	\end{align}
	where \eqref{eq:why_thresholding_class_01} follows from the fact that $\Othres_1$ is $1$-Lipschitz, and in \eqref{eq:why_thresholding_class} we used that $f$ is uniformly bounded by $1$ on $[0,1]$.  Therefore, to lower-bound $\mathcal{A} ( \lip ( [0,1] ), \mathcal{R}(1,W,L), \nleft\| \cdot \nright\|_{L^2 ( [0,1] )})$, it suffices to lower-bound $ \mathcal{A} ( \lip ( [0,1] ), \Othres_1 \circ \mathcal{R}(1,W,L), \nleft\| \cdot \nright\|_{L^2 ( [0,1] )})$, which will be effected through the technique developed to prove the minimax approximation error lower bound \eqref{eq:fundamental_limit_quantization_1} in Corollary~\ref{eq:fundamental_limit_quantization} combined with a new covering number upper bound for $\Othres_1 \circ \mathcal{R}(1,W,L)$.
 We note that all our arguments can be extended, with minor effort, to the approximation of
 function classes that are uniformly bounded by arbitrary constants $E \in \mathbb{R}_{+}$, but we will stick to $\lip ([0,1])$ for brevity of exposition.

		We first present the upper bound on the covering number of $\Othres_1 \circ\mathcal{R}(d,W,L)$ and then show how it can be used to lower-bound $\mathcal{A} (\lip ( [0,1] ), \Othres_1 \circ \mathcal{R}(1,W,L), \nleft\| \cdot \nright\|_{L^2 ( [0,1] )} )$. 

		\begin{theorem}
			\label{thm:lower_bound_infinite_weight}
			Let  $d, W, L \in \mathbb{N}$, with $W, L \geq 2$, and let $P$ be a distribution on $\mathbb{R}^d$. For all $\varepsilon \in (0, 1\slash 2)$, it holds that
			\begin{equation*}
				\log ( N(\varepsilon,\Othres_1 \circ\mathcal{R}(d,W,L), L^2 ( P )) ) \leq C W^2L^2 \log ( W  L ) \log ( \varepsilon^{-1} ),
			\end{equation*}
			with $C \in \mathbb{R}_+$ an absolute constant.
		\end{theorem}
        A similar upper bound for general $p \in [1,\infty)$ was established implicitly in the proof of \cite[Theorem 1]{achour2022general}, namely
		\begin{equation}
			\label{eq:upper_old_bound_literature}
			\log ( N(\varepsilon,\Othres_1 \circ\mathcal{R}(d,W,L), L^p ( P )) ) \leq C(p)\, W^2L^2 (\log ( W  L ))^2 \log ( \varepsilon^{-1} ),
		\end{equation}
		with the constant $C(p)$ depending on $p$. While the bound \eqref{eq:upper_old_bound_literature} applies to $p \in [1,\infty)$, when particularized to $p=2$, it is weaker than that in Theorem~\ref{thm:lower_bound_infinite_weight}.
		
		The proof of Theorem~\ref{thm:lower_bound_infinite_weight} is based on a relation in \cite{Mendelson2003} between the covering number with respect to the $L^2(P)$-norm, for arbitrary distributions $P$, and the fat-shattering dimension of uniformly-bounded function classes combined with bounds on the fat-shattering dimension of ReLU networks \cite{bartlett2019nearly}. We first prepare the technical ingredients of the proof and start by recalling the definition of fat-shattering dimension. 
		\begin{definition}
              \cite{Mendelson2003}
			Let $\mathcal{X}$ be a set, $\mathcal{F}$ a class of functions from $\mathcal{X}$ to $\mathbb{R}$, and $\gamma \in \mathbb{R}_+$. The fat-shattering dimension of $\mathcal{F}$, written as $\text{fat} \, ( \mathcal{F}, \gamma )$, is the largest  $m \in \mathbb{N}$ for which there exists $( x_1,\dots, x_m, y_1,\dots, y_m ) \in \mathcal{X}^m \times \mathbb{R}^m$ such that for every $( b_1,\dots, b_m ) \in \{ 0,1 \}^m$, there is an $f \in \mathcal{F}$ so  that,      for all $i \in \{ 1,\dots, m \}$,
                {
                \begin{equation}
			\label{eq:require_fat_shattering_dimension}
				f(x_i) \left\{ \begin{aligned}
					&\geq  y_i + \gamma ,  \quad &&\text{if } b_i = 1,\\
					&\leq y_i,  \quad &&\text{if } b_i = 0.  
				\end{aligned}
                    \right.
			\end{equation}
                }
		\end{definition}

		Next, we upper-bound the fat-shattering dimension of $\Othres_1 \circ\mathcal{R}(d,W,L)$.
		\begin{lemma}
		\label{lemma:vc_dimension_upper_bounds}
			For $d, W,L \in \mathbb{N}$, with $W,L \geq 2$, it holds that
			\begin{equation}
				\label{eq:fat_dimension_upper_bounds}
				\text{fat}\,(\Othres_1 \circ\mathcal{R}(d,W,L), \gamma  ) \leq C_h W^2L^2 (\log ( W L )), \quad \text{for all } \gamma \in \mathbb{R}_+,
			\end{equation}
			with $C_h \in \mathbb{R}_+$ an absolute constant.

			\begin{proof}
				The result is essentially an implication of \cite[Eq.~(2)]{bartlett2019nearly}, with minor additional observations. We provide the details in Appendix~\ref{sub:proof_of_lemma_lemma:vc_dimension_upper_bounds}.
			\end{proof}
		\end{lemma}

		The  Mendelson-Vershynin upper bound \cite{Mendelson2003} on the covering number in terms of fat-shattering dimension is as follows.
		\begin{theorem}
		\cite[Theorem 1]{Mendelson2003}
		\label{thm:covering_number_}
			Let $\mathcal{F}$ be a class of functions defined on a set $\mathcal{X}$ with $\sup_{x \in \mathcal{X}, f \in \mathcal{F}} \nleft| f(x) \nright| \leq 1$. Then, for every distribution $P$ on $\mathcal{X}$, and all $\varepsilon \in (0,1)$,
			\begin{equation*}
				N ( \varepsilon, \mathcal{F}, L^2 ( P ) ) \leq \biggl(  \frac{2}{\varepsilon}\biggr)^{K\cdot \text{fat} ( \mathcal{F}, c\varepsilon )},
			\end{equation*}
			where $K, c \in \mathbb{R}_+$ are absolute constants.
		\end{theorem}

			In \cite{achour2022general}, a result similar to Theorem~\ref{thm:covering_number_}, albeit slightly weaker, but applicable to general $p \in [1,\infty$], namely \cite[Corollary 3.12]{mendelson2002rademacher}, was applied instead.

		We are now ready to prove Theorem~\ref{thm:lower_bound_infinite_weight}.

		\begin{proof}
		[Proof of Theorem~\ref{thm:lower_bound_infinite_weight}]
			Application of Theorem~\ref{thm:covering_number_} with $\mathcal{F} = \Othres_1 \circ\mathcal{R}(d,W,L)$, yields 
			\begin{equation}
			\label{eq:75_1}
				N ( \varepsilon, \Othres_1 \circ\mathcal{R}(d,W,L), L^2 ( P ) ) \leq \biggl(  \frac{2}{\varepsilon}\biggr)^{K\cdot \text{fat} ( \Othres_1 \circ \mathcal{R}(d,W,L),  c\varepsilon )}. 
			\end{equation}
			Taking logarithms in \eqref{eq:75_1} and applying Lemma~\ref{lemma:vc_dimension_upper_bounds} with  $\gamma = c \varepsilon$, results in 
			\begin{align*}
				\log ( N ( \varepsilon, \Othres_1 \circ\mathcal{R}(d,W,L), L^2 ( P ) ) )  \leq &\, K \cdot C_h W^2L^2 (\log ( W  L )) \log ( 2 \varepsilon^{-1}).
			\end{align*}
			The proof is concluded upon noting that $\log ( 2 \varepsilon^{-1} ) \leq \log ( \varepsilon^{-2} ) = 2 \log (\varepsilon^{-1} )$, for $\varepsilon \in (0,1\slash 2)$, and letting $C := 2 K \cdot C_h$. 
		\end{proof}

		We are now ready to put Theorem~\ref{thm:lower_bound_infinite_weight} to work in 
  deriving the sought lower bound on $\mathcal{A} (\lip ( [0,1] ), \mathcal{R}(1,W,L), \nleft\| \cdot \nright\|_{L^2 ( [0,1] )} )$.

		\begin{corollary}
            \label{eq:limit_unconstraint_wights_H1}
			For $W, L \in \mathbb{N}$, with $W, L \geq 2$, it holds that
			\begin{equation}
			\label{eq:error_lower_bound_infinity_weights}
				\mathcal{A} (\lip ( [0,1] ), \mathcal{R}(1,W,L), \nleft\| \cdot \nright\|_{L^2 ( [0,1] )} ) \geq  \min \biggl\{ \frac{1}{8}, C (W^2L^2 (\log ( W L ) )^2  )^{-1} \biggr\},
			\end{equation}
			with $C \in \mathbb{R}_+$  an absolute constant.
			\begin{proof}
				Set 
				\begin{equation}
				\label{eq:define_epsilon_coro76}
					\kappa := \mathcal{A} (\lip ( [0,1] ), \mathcal{R}(1,W,L), \nleft\| \cdot \nright\|_{L^2 ( [0,1] )} ).
				\end{equation}
				For $\kappa \geq \frac{1}{8}$, \eqref{eq:error_lower_bound_infinity_weights} holds trivially.
    For $\kappa < \frac{1}{8}$, 
				we first note that putting \eqref{eq:define_epsilon_coro76} together with \eqref{eq:why_thresholding_class_0}-\eqref{eq:why_thresholding_class_1}, yields 
				\begin{equation}
				\label{eq:prerequisite_76}
					\mathcal{A} (\lip ( [0,1] ), \Othres_1 \circ \mathcal{R}(1,W,L), \nleft\| \cdot \nright\|_{L^2 ( [0,1] )} ) \leq \kappa.
				\end{equation}
				It then follows from Proposition~\ref{lem:cardinality_approximation_class} with $\varepsilon = \kappa$, $\mathcal{G} = \lip ( [0,1] )$, $\mathcal{F} = \Othres_1 \circ \mathcal{R}(1,W,L)$, $\delta = \nleft\| \cdot \nright\|_{L^2 ( [0,1] )} $, and the prerequisite \eqref{eq:improper_cover} satisfied thanks to \eqref{eq:prerequisite_76},   that 
				\begin{equation}
				\label{eq:76_ud}
					N ( \kappa, \Othres_1 \circ \mathcal{R}(1,W,L), L^2 ( [0,1] ) ) \geq N ( 4 \kappa, \lip ( [0,1] ), L^2 ( [0,1] ) ).
				\end{equation}
				Next, we upper-bound the left-hand-side and lower-bound the right-hand-side of \eqref{eq:76_ud}. For the former, we note that Theorem~\ref{thm:lower_bound_infinite_weight} with $d = 1$ and $P$ the uniform distribution on $[0,1]$, yields
				\begin{equation}
				\label{eq:76_uu}
					\log ( N(\kappa,\Othres_1 \circ\mathcal{R}(1,W,L), L^2 ( [0,1] ) ) ) \leq C_1 W^2L^2 \log ( W  L ) \log ( \kappa^{-1} ),
				\end{equation}
				with $C \in \mathbb{R}_+$ an absolute constant.
    The lower bound on $N ( 4 \kappa, \lip ( [0,1] ), L^2 ( [0,1] ) )$ is obtained from Lemma~\ref{lem:packing_lower_bound_lipschitz} with $\varepsilon = 4 \kappa$ as
				\begin{equation}
				\label{eq:76_dd}
					\log ( N ( 4\kappa, \lip  ( [0,1] ), L^p ( [0,1])  ) ) \geq  C_2 ( 4\kappa )^{-1},
				\end{equation}
				with $C_{2} \in \mathbb{R}_+$ an absolute constant.
    Combining \eqref{eq:76_uu} and \eqref{eq:76_dd} with \eqref{eq:76_ud}, yields 
				\begin{equation*}
					C_1 W^2L^2 \log ( W  L ) \log ( \kappa^{-1} ) \geq  C_2 ( 4\kappa )^{-1},
				\end{equation*}
				which implies
				\begin{equation}
				\label{eq:condition_last_chapter}
					\frac{\kappa^{-1}}{\log ( \kappa^{-1} )} \leq  C_3  W^2L^2 \log ( WL ),
				\end{equation}
				with $C_3 := \max \{ \frac{4C_1}{C_2}, 4 \} \in \mathbb{R}_+$ an absolute constant. 
				We further upper-bound the right-hand-side of \eqref{eq:condition_last_chapter} according to
				\begin{align}
					C_3  W^2L^2 \log ( WL )  =&\, \frac{8 \log(C_3) C_3 W^2L^2 (\log ( WL ))^2}{\log( ( WL )^{8 \log(C_3)} )} \label{eq:prod_reduce_0}\\
					\leq&\, \frac{8 \log(C_3) C_3 W^2L^2 (\log ( WL ))^2}{\log(8 \log(C_3) C_3 W^2L^2 (\log ( WL ))^2)}, \label{eq:prod_reduce}
				\end{align}
				where in \eqref{eq:prod_reduce} we used 
				\begin{align*}
					( WL )^{8 \log(C_3)} \geq&\, ( WL )^{4 \log(C_3)} \cdot  ( WL )^{4}  \\
					\geq&\,  ( C_3  )^4 \cdot W^2L^2 ( \log(WL) )^2 \\
					\geq&\, 8\log(C_3) C_3 W^2L^2 ( \log(WL) )^2.
				\end{align*}
				Next, define $f: \mathbb{R}_+ \rightarrow \mathbb{R}$ according to $f(x) = \frac{x}{\log(x)}$. Then, \eqref{eq:condition_last_chapter}-\eqref{eq:prod_reduce} can be written as
				\begin{equation}
				\label{eq:f_and_epsilon}
					f(\kappa^{-1}) \leq f(8 \log(C_3) C_3 W^2L^2 (\log ( WL ))^2).
				\end{equation}
				We note that $\kappa^{-1} > ( \frac{1}{8} )^{-1} > e$ and $ 8 \log(C_3) C_3 W^2L^2 (\log ( WL ))^2 > e$, and  the function $f$ is strictly increasing on $(e, \infty)$ as $f'(x) = \ln (2)  \frac{\ln(x) - 1}{( \ln(x) )^2} > 0$, for $x \in (e,\infty)$.  It hence follows from \eqref{eq:f_and_epsilon} that $\kappa^{-1} \leq 8 \log(C_3) C_3 W^2L^2 (\log ( WL ))^2$, which is 
				\begin{equation*}
					\kappa = \mathcal{A} (\lip ( [0,1] ), \mathcal{R}(1,W,L), \nleft\| \cdot \nright\|_{L^2 ( [0,1] )} ) \geq ( 8 \log(C_3) C_3 )^{-1} ( W^2L^2 (\log ( WL ))^2 )^{-1}.
				\end{equation*}
				The proof is concluded upon taking $C = ( 8 \log(C_3) C_3 )^{-1}$. 
			\end{proof}
		\end{corollary}

As a byproduct of the results obtained in this section, we can conclude that in the approximation (in $L^2 ([0,1])$-norm) of functions in $\lip ([0,1])$ going from networks with bounded weights to networks with unbounded weights does not substantially improve approximation accuracy. Specifically, we have the following chain of inequalities
		\begin{align}
			\min \biggl\{ \frac{1}{8}, C (W^2L^2 (\log ( W L ) )^2  )^{-1} \biggr\} &\leq \mathcal{A}	( H^1 ( [0,1] ), \mathcal{R} ( 1, W, L ) , \nleft\| \cdot \nright\|_{L^2 ( [0,1] )} )\label{eq:finalfinal_1}\\
			&\leq\, \mathcal{A}	( H^1 ( [0,1] ), \mathcal{R} ( 1, W, L, 1 ), \nleft\| \cdot \nright\|_{L^2 ( [0,1] )} ) \label{eq:finalfinal_2}\\
			&\leq\,   C_2  ( W^2 L^2 \log (W)  )^{-1}, \label{eq:finalfinal_3}
		\end{align}
  where \eqref{eq:finalfinal_1} is the lower bound \eqref{eq:error_lower_bound_infinity_weights}, \eqref{eq:finalfinal_2} follows from $\mathcal{R} ( 1, W, L, 1 ) \subseteq \mathcal{R} ( 1, W, L )$, and \eqref{eq:finalfinal_3} is Lemma~\ref{lemma:approximation_lip} with $p=2$. Here, $C, C_2 \in \mathbb{R}_{+}$ are absolute constants.
This shows that the improvement obtainable from allowing unbounded weights is at most of order $\frac{\log(W)}{(\log(WL))^2}$.

	\appendix
	
\section{Notation and Basic Definitions}
\label{sec:notation}
	We denote the cardinality of the set $X$ by $\nleft| X \nright|$.  $\mathbb{N} = \{ 1,2,\dots \}$ designates the natural numbers, $\mathbb{R}$ stands for the real numbers, $\mathbb{R}_+$ for the positive real numbers, and $\emptyset$ for the empty set. The maximum, minimum, supremum, and infimum of the set $\mathbb{A} \subseteq \mathbb{R}$ are denoted by $\max \mathbb{A}$, $\min \mathbb{A}$, $\sup \mathbb{A}$, and $\inf \mathbb{A}$, respectively. The indicator function $1_P$ for proposition $P$ is equal to $1$ if $P$ is true and $0$ else. For a metric space $( \mathcal{X}, \delta )$ and sets $\mathcal{F}, \mathcal{G} \subseteq \mathcal{X}$, we define $\mathcal{A} ( \mathcal{G}, \mathcal{F}, \delta )  = \sup_{g \in \mathcal{G}} \inf_{f \in \mathcal{F}} \delta (f,g)$. 
	For a vector $b \in \mathbb{R}^d$, we let $\nleft\| b \nright\|_\infty := \max_{i = 1,\dots,d} \nleft| b_i \nright|$, $\nleft\| b \nright\|_0 := \sum_{i = 1}^d 1_{b_i \neq 0}$, and $\nleft\| b \nright\|_1 := \sum_{i = 1}^d  \nleft| b_i \nright| $. Similarly, for a matrix $A \in \mathbb{R}^{m \times n}$, we define $\nleft\| A \nright\|_\infty = \max_{i=1,...,m,j=1,...,n} \nleft| A_{i,j} \nright| $ and $\nleft\| A \nright\|_0 = \sum_{i = 1}^m \sum_{j = 1}^n  1_{A_{i,j} \neq 0}$. $1_{m}$ and $0_{m}$ stand for the	$m$-dimensional vector with all entries equal to $1$ and $0$, respectively. $I_m$ refers to the $m \times m$ identity matrix. $1_{m\times n}$ and $0_{m\times n}$ denote the $m\times n$ matrix with all entries equal to $1$ and $0$, respectively. 
 We designate the  block-diagonal matrix with diagonal element-matrices $A_1, \dots, A_n$, possibly of different dimensions, by $\text{diag} (A_1, \dots, A_n)$.
	The truncation operator $\Othres_E: \mathbb{R} \rightarrow [-E, E]$, $E \in \mathbb{R}_+$, is $\Othres_E ( x ) = \max \{ -E, \allowbreak \min \{ E, x \} \}$. $\log(\cdot)$ and $\ln(\cdot)$ denote the logarithm to base $2$ and base $e$, respectively. The ReLU activation function is defined as $\rho(x) = \max \{ x,0 \}$, for $x \in \mathbb{R}$, and, when applied to vectors, acts elementwise. The sign function $\text{sgn}: \mathbb{R} \rightarrow  \{ 0,1 \}$ is given by $\text{sgn}(x) =1$, for $x \geq 0$, and $\text{sgn} ( x ) = 0$, for $x < 0$. We use $S(A,b)$ to refer to the affine mapping $S(A,b) (x) = Ax + b, x \in \mathbb{R}^{n_2}$, with $A \in \mathbb{R}^{n_1 \times n_2}$, $b \in \mathbb{R}^{n_1}$. For the set $\mathbb{X} \subseteq \mathbb{R}^d$, with $d \in \mathbb{N}$, and the function $f: \mathbb{X} \rightarrow \mathbb{R}$, we define the $L^p ( \mathbb{X}) $-norm of $f$, with $p \in [1,\infty)$, according to $\nleft\| f \nright\|_{L^p ( \mathbb{X} )} = ( \int_{x \in \mathbb{X}} \nleft| f(x) \nright|^{p} d \mu(x)  )^{1\slash p}$, where $\mu$ is the Lebesgue measure on $\mathbb{R}^d$. The  $L^\infty(\mathbb{X}) $-norm of $f$ is given by $\nleft\| f \nright\|_{L^\infty ( \mathbb{X} )} = \sup_{x \in \mathbb{X}} \nleft| f(x) \nright|$, and, for a distribution $P$ on $\mathbb{X}$, we define the $L^2 ( P )$-norm of $f$ as $\nleft\| f \nright\|_{L^2 ( P )} = ( \int_{x \in \mathbb{X}} \nleft| f(x) \nright|^{2} d P(x)  )^{1\slash 2}$.  
	A constant is said to be absolute if it does not depend on any variables or parameters. 

	\section{Further Proofs}
\label{sec:complexity_proof}

		\subsection{Proof of Proposition~\ref{lem:cardinality_approximation_class}}
		\label{sub:approximation_error_lower_bound_by_covering_number}
		
		We start with a lemma that gives a lower bound on the cardinality of $\mathcal{F}$, in terms of the packing number of $\mathcal{G}$, under the condition \eqref{eq:improper_cover}. 
			\begin{lemma}
				\label{thm:memory_requirement}
				Let $(\mathcal{X}, \delta)$ be a metric space, $\mathcal{F}, \mathcal{G} \subseteq \mathcal{X}$,
    and $\varepsilon \in \mathbb{R}_+$. Assume that $\mathcal{A} ( \mathcal{G}, \mathcal{F}, \delta ) \leq \varepsilon$. Then, we have 
				\begin{equation*}
					| \mathcal{F} |  \geq M ( 2 \varepsilon, \mathcal{G}, \delta  ).
				\end{equation*}

				\begin{proof}
                    Arbitrarily fix $\varepsilon \in \mathbb{R}_+$.
     Suppose, for the sake of contradiction, that $\nleft| \mathcal{F} \nright|  < M ( 2 \varepsilon, \mathcal{G}, \delta  )$, which would imply the existence of a $(2 \varepsilon)$-packing $\mathcal{P}$ of $\mathcal{G}$ such that $| \mathcal{F} |  < | \mathcal{P} | $. In particular, $\mathcal{F}$ would be a finite set. 
     Since $\mathcal{P}$ is a subset of $\mathcal{G}$, we have  $\mathcal{A} ( \mathcal{P} , \mathcal{F}, \delta ) \leq \mathcal{A} ( \mathcal{G} , \mathcal{F}, \delta ) \leq \varepsilon$, and therefore every element of $\mathcal{P}$ would be contained in an $\varepsilon$-neighborhood, with respect to the metric $\delta$, of an $f \in \mathcal{F}$. As $\nleft| \mathcal{F} \nright| < \nleft| \mathcal{P} \nright|$ and there are $\nleft| \mathcal{F} \nright|$ such neighborhoods and $\nleft| \mathcal{P} \nright|$ elements to be contained in neighborhoods,  the pigeonhole principle implies the existence of  $\nmathbf{x}_1,\nmathbf{x}_2 \in \mathcal{P}$ such that $\nmathbf{x}_1,\nmathbf{x}_2 \in \{ x: \delta ( f_0, \nmathbf{x} ) \leq \varepsilon \} $ for some $f_0 \in \mathcal{F}$. It would then follow from the triangle inequality that $\delta ( \nmathbf{x}_1, \nmathbf{x}_2 ) \leq \delta ( \nmathbf{x}_1, f_0 ) + \delta ( f_0, \nmathbf{x}_2 ) \leq 2 \varepsilon$, which implies that $\mathcal{P}$ can not be a $(2 \varepsilon)$-packing. This establishes the desired contradiction.
				\end{proof}
			\end{lemma}
                
			We are now ready to prove Proposition~\ref{lem:cardinality_approximation_class}. 
   
                        If $N ( \varepsilon, \mathcal{F}, \delta ) = \infty$, then \eqref{eq:minimal_cardinality_G} holds trivially. For $N ( \varepsilon, \mathcal{F}, \delta ) < \infty$, suppose that $\mathcal{C}$ is a minimal $\varepsilon$-covering of $\mathcal{F}$. 
            Defining $p: \mathcal{F} \to \mathcal{C}$ according to $p(f) = \argmin_{c \in \mathcal{C}} \delta( f,c) $, we hence get
            \begin{equation}
				\delta ( f,  p(f)  ) \leq \varepsilon. \label{eq:proof_approximation_size_0}
		\end{equation}
			Elements of $\mathcal{G}$ can now be approximated by elements of $\mathcal{C}$, with corresponding minimax approximation error 
			\begin{align}
				\mathcal{A} ( \mathcal{G}, \mathcal{C}, \delta ) =&\, \sup_{g \in \mathcal{G}} \inf_{c \in \mathcal{C}} \delta ( g,c ) \label{eq:proof_approximation_size_1}\\
				=&\, \sup_{g \in \mathcal{G}} \inf_{f \in \mathcal{F}}  \inf_{c \in \mathcal{C}} \delta ( g,c ) \label{eq:proof_approximation_size_2}\\
				\leq &\, \sup_{g \in \mathcal{G}} \inf_{f \in \mathcal{F}}  \inf_{c \in \mathcal{C}} ( \delta ( g,f ) + \delta ( f,c ) )\label{eq:proof_approximation_size_3} \\
                = &\, \sup_{g \in \mathcal{G}} \inf_{f \in \mathcal{F}}  ( \delta ( g,f ) + \delta ( f,p(f) ) ) \label{eq:proof_approximation_size_4}\\
				\leq &\, \sup_{g \in \mathcal{G}} \inf_{f \in \mathcal{F}}  ( \delta ( g,f ) + \varepsilon ) \label{eq:proof_approximation_size_5}\\
				= &\, \mathcal{A} ( \mathcal{G}, \mathcal{F}, \delta ) + \varepsilon \label{eq:proof_approximation_size_6}, \\
				\leq&\, 2 \varepsilon. \label{eq:proof_approximation_size_7}
			\end{align}
			where \eqref{eq:proof_approximation_size_3} is by the triangle inequality, \eqref{eq:proof_approximation_size_4} follows by 
            definition of $p$,
            in \eqref{eq:proof_approximation_size_5} we used  \eqref{eq:proof_approximation_size_0}, and \eqref{eq:proof_approximation_size_7} is thanks to the assumption $\mathcal{A} ( \mathcal{G}, \mathcal{F}, \delta ) \leq \varepsilon$. 
			Application of Lemma~\ref{thm:memory_requirement} with $\mathcal{F}=\mathcal{C}$, $\varepsilon$ replaced by $2 \varepsilon$, and the prerequisite  satisfied thanks to \eqref{eq:proof_approximation_size_1}-\eqref{eq:proof_approximation_size_7}, yields
			\begin{equation*}
				N ( \varepsilon, \mathcal{F}, \delta ) = | \mathcal{C} |  \geq M ( 4 \varepsilon, \mathcal{G}, \delta  ),
			\end{equation*}
			which together with $M ( 4 \varepsilon, \mathcal{G}, \delta  ) \geq N ( 4 \varepsilon, \mathcal{G}, \delta  )$, owing to Lemma~\ref{lem:equivalence_covering_packing}, concludes the proof.  \qedhere

	\subsection{Proof of Lemma~\ref{lem:1dpinfinity}}
	\label{sub:proof_of_lemma_lem:1dpinfinity}
   Fix a maximal $(2\varepsilon)$-packing $\{ f_i \}_{i = 1}^{M (2\varepsilon,\mathcal{R}(1, W,L,B), {L^1 ( [0,1] )} ) }$  of $\mathcal{R}(1, W,L,B)$ with respect to  the $L^1 ( [0,1] )$-norm. We shall lift this packing into a $(2\varepsilon)$-packing of $\mathcal{R}(d,W,L,B)$ with respect to the $L^p ( [0,1]^d )$-norm. Specifically,  for $i \in \{ 1,\dots, M (2\varepsilon,\mathcal{R}(1, W,L,B), {L^1 ( [0,1] )} )  \}$, as $f_i \in \mathcal{R}(1, W,L,B)$, there exists a network configuration 
   $\Phi_i = ( A_\ell^i, b_\ell^i)_{\ell = 1}^{\tilde{L}_i} \in \mathcal{N}(1, W,L,B) $ with $\tilde{L}_i \leq L$ such that $R ( \Phi_i ) = f_i$. 
   Let $(\tilde{A}_1^i,  \tilde{b}_1^i) = ( (\begin{smallmatrix}
				A_1^i, \, 0_{d' \times (d -1)  }
			\end{smallmatrix}), b_1^i )$, with $d'$ the number of rows of $A_1^i$,  and $( \tilde{A}_\ell^i, \tilde{b}_\ell^i ) = ( {A}_\ell^i, {b}_\ell^i )$, for $1 < \ell \leq \tilde{L}_{i}$, and set $g_i := R ( ( \tilde{A}_\ell^i, \tilde{b}_\ell^i)_{\ell = 1}^{\tilde{L}_i} )$.  We note that, for all $( x_1,\dots, x_d ) \in \mathbb{R}^d$, 
			\begin{equation*}
				\tilde{A}_1^i ( x_1,\dots, x_d )^T + \tilde{b}_1^i  = A_1^i x_1 + b_1^i,
			\end{equation*}
			which
   implies, for all $( x_1,\dots, x_d ) \in \mathbb{R}^d$,
			\begin{equation}
			\label{eq:construct_g_from_f}
				g_i (   x_1,\dots, x_d ) = R ( ( \tilde{A}_\ell^i, \tilde{b}_\ell^i)_{\ell = 1}^{\tilde{L}_i} ) ( x_1,\dots, x_d ) = R ( ( {A}_\ell^i, {b}_\ell^i)_{\ell = 1}^{\tilde{L}_i} ) ( x_1 ) = f_i ( x_i ).
			\end{equation}
			As $ ( \tilde{A}_\ell^i, \tilde{b}_\ell^i)_{\ell = 1}^{\tilde{L}_i} \in \mathcal{N} ( d, W, L ,B )$, we have $g_i \in R ( d, W, L ,B )$. 
   Next, we shall establish that $\{ g_i \}_{i =1}^{M (2\varepsilon,\mathcal{R}(1, W,L,B), {L^1 ( [0,1] )} ) }$ is a $(2\varepsilon)$-packing of $\mathcal{R}(d ,W,L,B)$ with respect to the  $L^p ( [0,1]^d )$-norm. To this end, let $q \in [1,\infty]$ be such that $\frac{1}{p} + \frac{1}{q} = 1$.  We then have, for $i,j \in \{ 1,\dots, M (2\varepsilon,\mathcal{R}(1, W,L,B), {L^1 ( [0,1] )} )  \}$ with  $i \neq j$, that
			\begin{align}
				\nleft\| g_i - g_j \nright\|_{L^p ( [0,1]^d )} =&\,  \nleft\| g_i - g_j \nright\|_{L^p ( [0,1]^d )} \nleft\| \mathbf{1} \nright\|_{L^q ( [0,1]^d )} \label{eq:lpl1_1} \\
				\geq &\, \nleft\| g_i - g_j \nright\|_{L^1 ( [0,1]^d )} \label{eq:lpl1_2}\\
				= &\, \int_{( x_1,\dots, x_d )\in [0,1]^d}  \nleft| g_i ( x_1,\dots,x_d ) - g_j ( x_1,\dots, x_d ) \nright| \, dx_1 \dots d x_d \label{eq:lpl1_3}\\
				= &\, \int_{( x_1,\dots, x_d )\in [0,1]^d}  \nleft| f_i ( x_1 ) - f_j ( x_1 ) \nright| \, dx_1 \dots d x_d  \label{eq:lpl1_4}\\
				= &\, \int_{x_1 \in  [ 0,1 ]}  \nleft| f_i ( x_1 ) - f_j ( x_1 ) \nright|  dx_1 \label{eq:lpl1_5}\\
				=& \, \nleft\| f_i - f_j \nright\|_{L^1 ( [0,1] )} \label{eq:lpl1_6}\\
				> &\, 2\varepsilon, \label{eq:lpl1_7}
			\end{align}
			where in \eqref{eq:lpl1_1} we denoted by $\mathbf{1}$  the constant function taking value $1$ on $[0,1]^d$,
   \eqref{eq:lpl1_2} follows from the H\"older inequality, in \eqref{eq:lpl1_4} we used \eqref{eq:construct_g_from_f}, and \eqref{eq:lpl1_7} is a consequence of $\{ f_i \}_{i = 1}^{M (2\varepsilon,\mathcal{R}(1, W,L,B), {L^1 ( [0,1] )} ) }$ being a $(2\varepsilon)$-packing with respect to the $L^1 ( [0,1] )$-norm. We have therefore established that $\{ g_i \}_{i =1}^{M (2\varepsilon,\mathcal{R}(1, W,L,B), {L^1 ( [0,1] )} ) }$ is a $(2\varepsilon)$-packing of $\mathcal{R}(d,W,L,B)$ with respect to the $L^p ( [0,1]^d )$-norm, and hence
			\begin{equation}
				M(2\varepsilon,\mathcal{R}(d,W,L,B), {L^p ( [0,1]^d )}  )  \geq  M (2\varepsilon,\mathcal{R}(1, W,L,B), {L^1 ( [0,1] )}  ).
			\end{equation}

	\subsection{Proof of Lemma~\ref{lem:packing_number_piecewise_linear_function}}
	\label{sub:proof_of_lemma_lem:packing_number_piecewise_linear_function}
		\newcommand{\Qres}{M}
		For $\varepsilon \geq \frac{E}{4N}$, we have $\log ( \lceil \frac{E}{4\varepsilon N} \rceil ) = 0$ so that \eqref{eq:lemma27} holds trivially. For $\varepsilon < \frac{E}{4N}$, we prove the statement by explicitly constructing an $\varepsilon$-packing of suitable cardinality.
  			To this end, for $y = ( y_i )_{i = 1}^N$, we define functions $f_{y} \in \Sigma ( X_N, \infty )$ as follows
\begin{equation}
			\label{eq:A4_def_fy}
				f_{y} (x) = 
                    \left\{
                    \begin{aligned}
					&0, && \text{for } x \in (-\infty, 0],\\
					&f_y \bigl ( \frac{i - 1}{N}\bigr) + \Qnc( x - \frac{i - 1}{\Qnc} ) ( y_{i} - f_y \bigl ( \frac{i - 1}{N}\bigr)  ), &&\text{for } x \in \bigl(\frac{i - 1}{\Qnc}, \frac{i}{\Qnc}\bigr], \quad i = 1,\dots, \Qnc,\\
					&y_N, &&\text{for } x \in (1,\infty). 
				\end{aligned}
                \right.
			\end{equation}
			and note that  $f_y ( 0 ) = 0$ and $f_y (\frac{i}{N}) =y_i $, for $i = 1,\dots, N$. Now, consider the set of functions 
			\begin{equation*}
				\mathcal{F}_{\Qnc,\Qres} := \biggl\{ f_y: y = ( y_i )_{i=1}^\Qnc \in \biggl( \biggl\{ \frac{\ell}{\Qres} E  \biggr\}_{\ell = 0}^\Qres \biggr)^\Qnc \biggr\} \subseteq \Sigma ( X_\Qnc, E ),
			\end{equation*}
			with $\Qres$ an integer to be specified later, namely such that $\mathcal{F}_{\Qnc,\Qres}$ is an $\varepsilon$-packing of $\Sigma ( X_\Qnc, E )$ with appropriate cardinality.  We proceed to derive a lower bound on the distance between distinct elements in $\mathcal{F}_{\Qnc,\Qres}$.  For $y = ( y_{i} )_{i=1}^\Qnc \in ( \{ \frac{\ell}{\Qres}E \}_{\ell = 0}^\Qres )^\Qnc $ and $\tilde{y} = ( \tilde{y}_{i} )_{i=1}^\Qnc \in ( \{ \frac{\ell}{\Qres}E \}_{\ell = 0}^\Qres )^\Qnc $ such that $y \neq \tilde{y}$, we let $j \in \{ 1,\dots, \Qnc \}$ be the smallest index for which $y_{j} \neq \tilde{y}_j$, and then get
			\begin{align}
				&\,\| f_y - f_{\tilde{y}} \|_{L^1([0,1])} \label{eq:constructing_packing_1}\\
				&=\, \int_{0}^1 | f_y(x) - f_{\tilde{y}}(x)| dx \\
				&\geq\, \int_{\frac{j-1}{\Qnc}}^{\frac{j}{\Qnc}} | f_y(x) - f_{\tilde{y}}(x)| dx \\
				&=\, \int_{\frac{j-1}{\Qnc}}^{\frac{j}{\Qnc}} \biggl| f_y \biggl ( \frac{j -1 }{N}\biggr) - f_{\tilde{y}} \biggl ( \frac{j-1}{N}\biggr)  \nonumber\\
				&\, + \Qnc\biggl( x - \frac{j-1}{\Qnc} \biggr) \biggl(  y_{j} - \tilde{y}_{j} - f_y \biggl ( \frac{j -1 }{N}\biggr) + f_{\tilde{y}} \biggl ( \frac{j -1 }{N}\biggr)   \biggr) \biggr| dx \label{eq:constructing_packing_2}  \\ 
				&=\, \int_{\frac{j-1}{\Qnc}}^{\frac{j}{\Qnc}} \biggl|  \Qnc\biggl( x - \frac{j-1}{\Qnc} \biggr) (  y_{j} - \tilde{y}_{j}  ) \biggr| dx  \label{eq:constructing_packing_3}  \\
				&\geq \, \int_{\frac{j-1}{\Qnc}}^{\frac{j}{\Qnc}} \biggl|  \Qnc\biggl( x - \frac{j-1}{\Qnc} \biggr) \frac{E}{\Qres} \biggr| dx \label{eq:constructing_packing_4}  \\
				&= \, \frac{E}{2\Qres\Qnc}. \label{eq:constructing_packing_6}
			\end{align}
			where in \eqref{eq:constructing_packing_2} we used \eqref{eq:A4_def_fy}, \eqref{eq:constructing_packing_3} follows from $f_y  ( \frac{j -1 }{N}) = f_{\tilde{y}}  ( \frac{j -1 }{N})$,
   and in \eqref{eq:constructing_packing_4} we used $\nleft| y_j - \tilde{y}_j \nright|\geq \frac{E}{M} $. Set $\Qres = \bigl \lceil \frac{E}{4 \varepsilon \Qnc} \bigr \rceil$. As $\varepsilon < \frac{E}{4N}$ by assumption, we have $\frac{E}{4 \varepsilon \Qnc} > 1$ and hence $\Qres = \bigl \lceil \frac{E}{4 \varepsilon \Qnc} \bigr \rceil <  \frac{E}{2 \varepsilon \Qnc} $, where we used $\lceil x \rceil < 2 x $, for $x > 1$. We therefore get $\frac{E}{2\Qres\Qnc} >  \varepsilon$, which, owing to \eqref{eq:constructing_packing_1}-\eqref{eq:constructing_packing_6}, establishes that 
   $\mathcal{F}_{\Qnc, \lceil \frac{E}{ 4\varepsilon \Qnc} \rceil }$ is an $\varepsilon$-packing of $\Sigma ( X_\Qnc, E )$ with respect to the $L^1 ( [0,1] )$-norm. The proof is concluded by noting that
			\begin{equation}
			\label{eq:a4_49_packing}
				 M(\varepsilon,\Sigma ( X_\Qnc, E ), {L^1 ( [0,1] )} ) \geq \left| \mathcal{F}_{\Qnc, \lceil \frac{E}{ 4\varepsilon \Qnc} \rceil } \right| = \left| \left( \left\{ \frac{\ell}{\Qres}E  \right\}_{\ell = 0}^{\lceil \frac{E}{ 4\varepsilon \Qnc} \rceil } \right)^\Qnc \right| \geq \left( \left \lceil \frac{E}{ 4\varepsilon \Qnc} \right \rceil \right)^N. 
			\end{equation}

	\subsection{Proof of Lemma~\ref{lemma:vc_dimension_upper_bounds}}
	\label{sub:proof_of_lemma_lemma:vc_dimension_upper_bounds}
		We first need a concept closely related to fat-shattering dimension. 
		\begin{definition}
		\cite[Definition 2]{bartlett2019nearly}
			Let $\mathcal{X}$ be a set and $\mathcal{F}$ a class of functions from $\mathcal{X}$ to $\mathbb{R}$. The pseudodimension of $\mathcal{F}$, written as $\text{Pdim} ( \mathcal{F})$, is the largest integer $m$ for which there exists $( x_1,\dots, x_m, y_1,\dots, y_m ) \in \mathcal{X}^m \times \mathbb{R}^m$ such that for every $( b_1,\dots, b_m ) \in \{ 0,1 \}^m$, there is an $f \in \mathcal{F}$ so that, for all $i \in \{ 1,\dots, m \}$,
                {
                \begin{equation}
			\label{eq:require_pseudo_dimension}
				f(x_i) \left\{ \begin{aligned}
					&>  y_i, \quad &&\text{if } b_i = 1,\\
					&\leq y_i,\quad &&\text{if } b_i = 0.  
				\end{aligned}
                    \right.
			\end{equation}
                }
		\end{definition}
		As condition \eqref{eq:require_pseudo_dimension} defining pseudodimension is weaker than condition \eqref{eq:require_fat_shattering_dimension} defining fat-shattering dimension, we have, for all function classes $\mathcal{F}$, that \cite[Theorem 11.13 (i)]{Anthony1999}
		\begin{equation}
		\label{eq:comparing_fat_and_pseudo}
			\text{fat} ( \mathcal{F}, \gamma ) \leq \text{Pdim} ( \mathcal{F} ), \quad \text{for } \gamma \in \mathbb{R}_+.
		\end{equation}
		We are now ready to show how Lemma~\ref{lemma:vc_dimension_upper_bounds} can be proved by applying results from \cite{bartlett2019nearly}. First, note that \cite{bartlett2019nearly} applies to families of network realizations whose associated configurations have a fixed architecture, whereas $\mathcal{N}(d, W,L)$, the object of interest here, consists of network configurations with different architectures.  To resolve this discrepancy, we employ an idea used in the proof of \cite[Lemma A.2]{FirstDraft2022}.
  Specifically, we consider the set $\mathcal{N}^* ( d, 2W,L ) = \{ ( A_\ell,b_\ell )_{\ell = 1}^{L}: A_1 \in \mathbb{R}^{2W \times d}, b_1 \in \mathbb{R}^{2W}, A_L \in \mathbb{R}^{1 \times 2W}, b_L \in \mathbb{R}, A_\ell \in \mathbb{R}^{2W \times 2W}, b_\ell \in \mathbb{R}^{2W}, \text{ for } \ell \in \{ 2,\dots, L-1 \} 
	    \}$ consisting of all network configurations with the (fixed) architecture
		\begin{equation}
		\label{eq:defining_architecture}
			( N_\ell )_{\ell = 0}^L =  (d, \underbrace{2W, \dots,}_{(L-1) \text{ times}} 1).
		\end{equation}
		The associated family of network realizations is $\mathcal{R}^* ( d, 2W,L ) = \{  R (\Phi ): \Phi \in  \mathcal{N}^* ( d, 2W,L )\}$. 
  It now follows from the proof of \cite[Lemma A.2]{FirstDraft2022} that $\mathcal{R}(d,W,L) \subseteq \mathcal{R}^*(d,2 W,L)$. 
  Next, we note that the network configurations in  $\mathcal{N}^* ( d, 2W,L )$ have $ n(d, 2W,L) := 2dW + 4W + 1 +  ( L-2 ) ( (2W)^2 + 2W )$ weights.   As $\mathcal{R}^* (d, 2W,L )$ consists of realizations of network configurations with fixed architecture, namely \eqref{eq:defining_architecture}, we can apply the  results in \cite{bartlett2019nearly}. Specifically, we obtain 
		\begin{align}
			\text{\text{Pdim}} ( \mathcal{R}^* ( d, 2W,L ) ) \leq&\, C  n(d, 2W,L) L \log ( n(d, 2W,L) ) \label{eq:vc_dimension_upper_bounds_vanilla_1} \\
			\leq &\, C ( 13W^2 L ) L \log ( 13 W^2 L ) \label{eq:vc_dimension_upper_bounds_vanilla_2} \\
			\leq &\, 65\, C W^2 L^2 ( \log(W L) ), \label{eq:vc_dimension_upper_bounds_vanilla_3}
		\end{align}
  where $C \in \mathbb{R}_+$ is an absolute constant, 
		in \eqref{eq:vc_dimension_upper_bounds_vanilla_1} we used \cite[Eq.~(2)]{bartlett2019nearly} combined with the discussion at the end of the paragraph immediately after \cite[Definition 2]{bartlett2019nearly},
  \eqref{eq:vc_dimension_upper_bounds_vanilla_2} follows from $n(d, 2W,L) \leq 2W^2 + 4W + 1 +  ( L-2 ) ( (2W)^2 + 2W ) \leq 13 W^2 L $ by the standing assumption $W \geq d$, and \eqref{eq:vc_dimension_upper_bounds_vanilla_3} is thanks to $\log(13 W^2 L) \leq \log((WL)^5) = 5 ( \log(W L) )$ as $L \geq 2$. Since $\mathcal{R}(d,W,L) \subseteq \mathcal{R}^*(d,2 W,L)$, as noted above, we get
		\begin{equation}
                \label{eq:upper_bound_pdim}
			\text{Pdim} ( \mathcal{R}(d,W,L) ) \leq  \text{Pdim} (  \mathcal{R}^*(d,2W,L)) \leq 65\, C W^2 L^2 ( \log(WL) ).   
		\end{equation}
            To upper-bound the fat-shattering dimension of $\Othres_1 \circ \mathcal{R}(d,W,L)$, we first note that
            $\Othres_1(x) = -1  + \rho(x + 1) - \rho( x- 1)$, for $x \in \mathbb{R}$, and hence $\Othres_1 \in \mathcal{R}(1, 2, 2)$, which upon application of \cite[Lemma H.3]{FirstDraft2022}, yields
            \begin{equation}
            \label{eq:pseudo_inclusion}
                \Othres_1 \circ \mathcal{R}(d,W,L) \subseteq  \mathcal{R}(d,\max\{W, 2\},L + 2) = \mathcal{R}(d, W,L + 2). 
            \end{equation}
            It then follows, for all $\gamma \in \mathbb{R}_+$, that
            \begin{align}
                \text{fat} ( \Othres_1 \circ \mathcal{R}(d,W,L), \gamma ) \leq &\, \text{Pdim}  ( \Othres_1 \circ \mathcal{R}(d,W,L)) \label{eq:pseudo_final_1}\\
                \leq &\, \text{Pdim}  (\mathcal{R}(d, W,L + 2)) \label{eq:pseudo_final_2}\\
                \leq &\, 65\, C W^2 (L+2)^2 ( \log(W(L+2)) ) \label{eq:pseudo_final_3}\\
                < &\,  520\, C W^2 L^2 ( \log(WL) ), \label{eq:pseudo_final_4}
            \end{align}
            where \eqref{eq:pseudo_final_1} is \eqref{eq:comparing_fat_and_pseudo} with $\mathcal{F} =  \Othres_1 \circ \mathcal{R}(d,W,L)$, \eqref{eq:pseudo_final_2} follows from \eqref{eq:pseudo_inclusion}, in  \eqref{eq:pseudo_final_3} we used \eqref{eq:upper_bound_pdim} with $L$ replaced by $L + 2$, and \eqref{eq:pseudo_final_4} follows from $(L+2)^2 \leq (2L)^2$ and $\log(W(L+2)) \leq \log(WL^2) < 2 \log(WL)$, recalling that $L \geq 2$. The proof is concluded upon taking $C_h = 520\, C$.

\section{Proof of Theorem~\ref{thm:expected_L2_error}}
\label{sec:complexity_base_prediction_error_bound}

	The proof will be effected by establishing a slightly stronger result; this is done to better illustrate the roles of the assumptions in Theorem~\ref{thm:expected_L2_error}. Specifically, we shall replace the assumption \eqref{eq:lsr_approximation_simplified} by 
	\begin{equation}
		\inf_{f \in \mathcal{F}_n} \nleft\| g - f \nright\|_{L^2 ( P)} \leq \Qapproximation, \quad \Qapproximation \in \mathbb{R}_+,
	\end{equation}
	and the  assumption \eqref{eq:numerical_assumption_simplified} by 
	\begin{equation}
		\frac{1}{n} \sum_{i = 1}^n ( \hat{f}_n (x_i)  - y_i )^2 \leq \inf_{f \in \mathcal{F}_n} \biggl( \frac{1}{n} \sum_{i = 1}^n ( f (x_i) - y_i )^2 \biggr) +  \Qnumeric, \quad \text{a.s.}, \quad  \Qnumeric \in \mathbb{R}_+.
	\end{equation} 
	With this new set of assumptions, we have the following theorem.

	\begin{theorem}
	\label{thm:expected_L2_error_general}
		Let $\mathbb{X} \subseteq \mathbb{R}^d$ and consider the regression function $g: \mathbb{X} \rightarrow \mathbb{R}$. Let $n \in \mathbb{N}$ and $\sigma \in \mathbb{R}_+$. Let $P$ be a distribution on $\mathbb{X}$, with the associated samples $( x_i,y_i )_{i = 1}^n = ( x_i, g(x_i) + \sigma \xi_i )_{i =1}^n$,
			where $(x_i)_{i = 1}^n$ are i.i.d. random variables of distribution $P$, $(\xi_i)_{i=1}^n$ are i.i.d. standard Gaussian random variables, 
   and $( x_i )_{i = 1}^n$ and $(\xi_i)_{i=1}^n$ are statistically independent. 

			Let $A, \kappa \in \mathbb{R}_+$,  and consider a class of functions $\mathcal{F}_n \subseteq L^\infty ( \mathbb{X} )$ such that   
			\begin{equation}
			\label{eq:lsr_approximation_general}
				\inf_{f \in \mathcal{F}_n} \nleft\| g - f \nright\|_{L^2 ( P )} \leq A, 
			\end{equation}
			and an $\mathcal{F}_n$-valued random variable $\hat{f}_n$ satisfying
			\begin{equation}
			\label{eq:numerical_assumption_general}
				\frac{1}{n} \sum_{i = 1}^n ( \hat{f}_n (x_i)  - y_i )^2 \leq \inf_{f \in \mathcal{F}_n} \biggl( \frac{1}{n} \sum_{i = 1}^n ( f (x_i) - y_i )^2 \biggr) +  \kappa, \quad  \text{a.s.}
			\end{equation}
			For all $\Qradius \in ( 0,1\slash 2 )$, it holds that
			\begin{equation*}
			\begin{aligned}
				&E ( \nleft\| \hat{f}_n - g  \nright\|^2_{L^2 ( P )}  ) \\
				&\leq 16( \Qapproximation^2 + \Qnumeric) + 64 ( \sigma +  \Qradius ) \Qradius + 800 ( \sigma + \sigma^2 + (\Qlinfty ( g, \mathcal{F}_n ))^2 ) \frac{\log  (N ( \Qradius, \mathcal{F}_n, {L^\infty ( \mathbb{X}  ) }  ))  + 1}{n},
			\end{aligned}
			\end{equation*}
			where  $\Qlinfty ( g, \mathcal{F}_n ):  =  \max \{ \nleft\| g \nright\|_{L^\infty ( \mathbb{X} )}, \sup_{f \in \mathcal{F}_n} \nleft\| f \nright\|_{L^\infty ( \mathbb{X}  )}   \}$.
	\end{theorem}

	Taking $\Qapproximation  = \varepsilon$, $\Qnumeric = \varepsilon^2$,  and $\Qradius = \varepsilon^2$, with $\varepsilon \in (0,1/2)$, in Theorem~\ref{thm:expected_L2_error_general} implies
		\begin{equation*}
			E ( \nleft\| \hat{f}_n - g  \nright\|^2_{L^2 ( P )}  ) \leq 800 ( 1 + \sigma +  \sigma^2 + ( \Qlinfty ( g, \mathcal{F}_n ) )^2 ) \biggl( \varepsilon^2 +  \frac{\log  (N (\varepsilon^2  , \mathcal{F}_n, L^\infty ( \mathbb{X}  )  ))  + 1}{n} \biggr),
		\end{equation*}
		which, together with $\sigma \leq \frac{1 + \sigma^2}{2}$,
        establishes Theorem~\ref{thm:expected_L2_error} with $C = 1200$.

	To prepare for the proof of Theorem~\ref{thm:expected_L2_error_general}, we state two auxiliary technical lemmata. The first one provides an upper bound on the expected empirical risk. 

	\begin{lemma}
	\label{lem:convergence_empirical_error}
		Let $\mathbb{X}$, $P$, $g$, $n$, $\sigma$, $( x_i, y_i )_{i = 1}^n$, $A,\kappa$, $\mathcal{F}_n$, and $\hat{f}_n$ be defined as in  Theorem~\ref{thm:expected_L2_error_general} and assume that (\ref{eq:lsr_approximation_general}) and (\ref{eq:numerical_assumption_general}) hold. For all $\Qradius \in ( 0,1\slash 2 )$, we have
		\begin{align*}
			& E \biggl( \frac{1}{n} \sum_{i = 1}^n ( \hat{f}_n (x_i) - g ( x_i ) )^2 \biggr) \\
			& \leq \, 2( \Qapproximation^2 + \Qnumeric) + 8 \sigma \Qradius + 100 ( \sigma + \sigma^2) \, \frac{\log  (N ( \Qradius,\, \mathcal{F}_n,\, {L^\infty ( \mathbb{X}  ) }  ))  + 1}{n}.
		\end{align*}
		\begin{proof}
			See Appendix~\ref{sub:convergence_of_the_empirical_risk}.
		\end{proof}
	\end{lemma}

	The second lemma relates the expected empirical risk to the expected prediction error.
	\begin{lemma}
	\label{lem:population_and_empirical}
Let $\mathbb{X}$, $P$, $g$, $n$, $\sigma$, $( x_i, y_i )_{i = 1}^n$, $A,\kappa$, $\mathcal{F}_n$, and $\hat{f}_n$ be defined as in  Theorem~\ref{thm:expected_L2_error_general} and assume that (\ref{eq:lsr_approximation_general}) and (\ref{eq:numerical_assumption_general}) hold. For all $\Qradius \in ( 0,1\slash 2 )$, we have		
		\begin{equation*}
		\begin{aligned}
			&E ( \nleft\| \hat{f}_n - g  \nright\|^2_{L^2 ( P )}  ) \\
			&\leq  8 E \biggl( \frac{1}{n} \sum_{i = 1}^n ( \hat{f}_n (x_i) - g ( x_i ) )^2 \biggr) +  64 \biggl( (\Qlinfty ( g, \mathcal{F}_n ))^2 \frac{\log (N ( \Qradius, \mathcal{F}_n, {L^\infty (\mathbb{X}  )} ))}{n} +  \Qradius^2 \biggr),
		\end{aligned}
		\end{equation*}
		where  $\Qlinfty ( g, \mathcal{F}_n ):  =  \max \{ \nleft\| g \nright\|_{L^\infty ( \mathbb{X} )}, \sup_{f \in \mathcal{F}_n} \nleft\| f \nright\|_{L^\infty ( \mathbb{X}  )}   \}$.

		\begin{proof}
			See Appendix~\ref{sub:convergence_of_the_empirical_process}.
		\end{proof}
	\end{lemma}

	Putting Lemmata~\ref{lem:convergence_empirical_error} and \ref{lem:population_and_empirical} together, we can now finalize the proof of 
  Theorem~\ref{thm:expected_L2_error_general} as follows:
		\begin{align}
		&E ( \nleft\| \hat{f}_n - g  \nright\|_{L^2 ( P )}  )\label{eq:recoverin_g_1} \\
		&\leq\,  8 E \biggl( \frac{1}{n} \sum_{i = 1}^n ( \hat{f}_n (x_i) - g ( x_i ) )^2 \biggr) +  64 \biggl( (\Qlinfty ( g, \mathcal{F}_n ))^2 \frac{\log (N ( \Qradius, \mathcal{F}_n, {L^\infty (\mathbb{X}  )} ))}{n} +  \Qradius^2 \biggr)\label{eq:recoverin_g_2}\\
		& \leq \, 16( \Qapproximation^2 + \Qnumeric) + 64 \sigma \Qradius + 800 ( \sigma + \sigma^2  ) \frac{\log  (N ( \Qradius, \mathcal{F}_n, {L^\infty ( \mathbb{X}  ) }  ))  + 1}{n}  \nonumber\\
		&\,\,+ 64 (\Qlinfty ( g, \mathcal{F}_n ))^2 \frac{\log (N ( \Qradius, \mathcal{F}_n, {L^\infty ( \mathbb{X}  )} ))}{n} + 64 \Qradius^2 \label{eq:recoverin_g_4}\\
		&\leq\, 16( \Qapproximation^2 + \Qnumeric) + 64 (\sigma +  \Qradius ) \Qradius + 800 ( \sigma + \sigma^2 + (\Qlinfty ( g, \mathcal{F}_n ))^2 ) \frac{\log  (N ( \Qradius,\, \mathcal{F}_n,\, {L^\infty ( \mathbb{X}  ) }  ))  + 1}{n}.  \label{eq:recoverin_g_5}
	\end{align}

	\subsection{Proof of Lemma~\ref{lem:convergence_empirical_error}}
	\label{sub:convergence_of_the_empirical_risk}
		Arbitrarily fix $f \in \mathcal{F}_n$. By assumption \eqref{eq:numerical_assumption_general}, we have 
		\begin{equation}
		\label{eq:original_basic_inequality}
			\frac{1}{n} \sum_{i = 1}^n ( \hat{f}_n (x_i) - y_i )^2 \leq  \frac{1}{n}  \sum_{i = 1}^n ( f (x_i) - y_i )^2 + \Qnumeric, \quad \text{a.s.}
		\end{equation}
		Substituting $y_i = g(x_i) + \sigma \xi_i$, $i = 1,\dots,n$, into \eqref{eq:original_basic_inequality}, yields 
		\begin{equation}
			\frac{1}{n} \sum_{i = 1}^n ( \hat{f}_n (x_i) -  g(x_i) - \sigma \xi_i  )^2 \leq  \frac{1}{n} \sum_{i = 1}^n ( f (x_i) -  g(x_i) - \sigma \xi_i  )^2 + \Qnumeric, \quad \text{a.s.}
		\end{equation}
		which results in 
		\begin{equation}
		\label{eq:basic_transformed_pointwise}
			\frac{1}{n} \sum_{i = 1}^n ( \hat{f}_n (x_i) - g ( x_i ) )^2 \leq \frac{1}{n} \sum_{i = 1}^n \biggl (  ( f (x_i) - g ( x_i ) )^2  +  \frac{2\sigma}{n} \xi_i ( \hat{f}_n (x_i)  - f ( x_i ) ) \biggr) + \Qnumeric, \quad \text{a.s.}
		\end{equation}
		Taking expectations in  \eqref{eq:basic_transformed_pointwise} yields
		\begin{align}
			&E \biggl( \frac{1}{n} \sum_{i = 1}^n ( \hat{f}_n (x_i) - g ( x_i ) )^2 \biggr)\label{eq:basic_inequality_implication_1}\\
			&\leq\, E \biggl( \frac{1}{n} \sum_{i = 1}^n ( f (x_i) - g ( x_i ) )^2 \biggr)  + \frac{2\sigma}{n} E \biggl(  \sum_{i = 1}^n \xi_i ( \hat{f}_n (x_i)  - f ( x_i ) ) \biggr) + \Qnumeric\label{eq:basic_inequality_implication_2}\\
			&=\, \nleft\| f - g \nright\|^2_{L^2 ( P  )}  + \frac{2\sigma}{n} E \biggl( \ \sum_{i = 1}^n \xi_i ( \hat{f}_n (x_i)  - f ( x_i ) ) \biggr) + \Qnumeric \label{eq:basic_inequality_implication_3} \\
			&= \,  \nleft\| f - g \nright\|^2_{L^2 ( P  )}  + \frac{2\sigma}{n} E \biggl( \sum_{i = 1}^n \xi_i ( \hat{f}_n (x_i)  - g ( x_i ) ) \biggr) + \Qnumeric \label{eq:basic_inequality_implication_4} 
		\end{align}
		where in \eqref{eq:basic_inequality_implication_3} we used that, for $i = 1,\dots, n$, $x_i$ is a random variable of distribution $P$ and hence $E (( f (x_i) - g ( x_i ) )^2 ) = \nleft\| f - g \nright\|^2_{L^2 ( P )}$, and \eqref{eq:basic_inequality_implication_4} follows from adding $ \frac{2 \sigma}{n} E(\sum_{i = 1 }^n \xi_i ( f(x_i) - g(x_i) ) )  = 0$
  to \eqref{eq:basic_inequality_implication_3}.
  As the choice of $f \in \mathcal{F}_n$ was arbitrary, \eqref{eq:basic_inequality_implication_1}-\eqref{eq:basic_inequality_implication_4} holds for all $f \in \mathcal{F}_n$, and therefore
		\begin{equation}
		\label{eq:basic_updated_0}
			E \biggl( \frac{1}{n} \sum_{i = 1}^n ( \hat{f}_n (x_i) - g ( x_i ) )^2 \biggr)  \leq \inf_{f \in \mathcal{F}_n} \nleft\| f - g \nright\|^2_{L^2 (P)}  +  \frac{2\sigma}{n}  E \biggl(\sum_{i = 1}^n \xi_i ( \hat{f}_n (x_i)  - g ( x_i ) ) \biggr) + \Qnumeric.
		\end{equation}
		With $\inf_{f \in \mathcal{F}_n} \nleft\| g - f \nright\|_{L^2 ( P )} \leq A$, owing to assumption \eqref{eq:lsr_approximation_general}, and setting 
		\begin{align}
			\Delta =&\, (  \hat{f}_n (x_i)  - g ( x_i ) )_{i = 1}^n, \label{eq:definition_Delta}\\
			\xi =&\, ( \xi_i )_{i =1}^n, \label{eq:definition_xi} 
		\end{align}
		it follows from \eqref{eq:basic_updated_0} that
		\begin{equation}
		\label{eq:new_basic_inequality}
			E \biggl(\frac{ \nleft\| \Delta \nright\|_2^2 }{n}\biggr) \leq \Qapproximation^2 + \Qnumeric + \frac{2 \sigma}{n} E ( \langle \xi, \Delta \rangle ).
		\end{equation}
		We next note that the quantity $E ( \langle \xi, \Delta \rangle )$ can be upper-bounded by the expected supremum of a Gaussian process according to
		\begin{equation}
		\label{eq:expected_supremum}
			E ( \langle \xi, \Delta \rangle ) = E (  \langle \xi, (  \hat{f}_n (x_i)  - g ( x_i ) )_{i = 1}^n \rangle  ) \leq E \biggl( \sup_{f\in \mathcal{F}_n} \langle \xi, (  f (x_i)  - g ( x_i ) )_{i = 1}^n \rangle  \biggr).
		\end{equation}
		The right-hand-side of \eqref{eq:expected_supremum} can be further upper-bounded either in terms of the covering number of $\mathcal{F}_n$, through one-step discretization, or by using the more advanced Dudley entropy integral bound, see e.g. \cite[Section 5.3]{wainwright2019high}.  The one-step discretization approach turns out to suffice for the purposes of this proof. Specifically, let $\{ f_j \}_{j = 1}^{N ( \Qradius, \, \mathcal{F}_n, \,{L^\infty ( \mathbb{X}  )}  )} $ be a $\Qradius$-covering of $\mathcal{F}_n$, with respect to the $L^\infty ( \mathbb{X}  )$-norm, and define the associated set of random vectors $\{\Delta_j\}_{j = 1}^{N ( \Qradius, \, \mathcal{F}_n, \, {L^\infty ( \mathbb{X}  )}  )}$ according to
		\begin{equation*}
			\Delta_j = (  f_j (x_i)  - g ( x_i ) )_{i = 1}^n.
		\end{equation*}
            Further, define the random vector $\Delta^p$ as the element in $\{ \Delta_j \}_{j =1}^{N ( \Qradius,\, \mathcal{F}_n,\, {L^\infty ( \mathbb{X}  )}  )}$ that is closest to  $\Delta$, in the sense of
		\begin{equation}
		\label{eq:definition_Delta_p}
			\Delta^p = \argmin_{\Delta_j: j = 1,\dots, N ( \Qradius, \, \mathcal{F}_n, \, {L^\infty ( \mathbb{X}  )}  )} \nleft\| \Delta - \Delta_j \nright\|_2,
		\end{equation}
		and note that, a.s., 
		\begin{align}
			\nleft\|  \Delta^p - \Delta \nright\|_2 =&\,\min_{j = 1,\dots, N ( \Qradius,\, \mathcal{F}_n, {L^\infty ( \mathbb{X}  )}  )} \nleft\| \Delta - \Delta_j \nright\|_2 \label{eq:projection_assumption_Delta_1}\\
			= &\, \min_{j = 1,\dots, N ( \Qradius,\, \mathcal{F}_n, {L^\infty ( \mathbb{X}  )}  )} \biggl( \sum_{i = 1}^n ( \hat{f}_n ( x_i ) - f_j ( x_i ) )^2 \biggr)^{1\slash 2}\label{eq:projection_assumption_Delta_2}\\
			\leq &\, \min_{j = 1,\dots, N ( \Qradius,\, \mathcal{F}_n, {L^\infty ( \mathbb{X}  )}  )} n^{1\slash 2} \nleft\| \hat{f}_n - f_j \nright\|_{L^\infty ( \mathbb{X}  )} \label{eq:projection_assumption_Delta_3}\\
			\leq &\, n^{1\slash 2} \Qradius,\label{eq:projection_assumption_Delta_4}
		\end{align}
		where
  in \eqref{eq:projection_assumption_Delta_4} we used the fact that $\{ f_j \}_{j = 1}^{N ( \Qradius,\, \mathcal{F}_n, {L^\infty ( \mathbb{X}  )}  )} $ is a $\Qradius$-covering of $\mathcal{F}_n$ with respect to the $L^\infty ( \mathbb{X}  )$-norm.
		We now get
			\begin{align}
				E ( \langle \xi, \Delta \rangle ) =&\, E ( \langle \xi, \Delta - \Delta^p \rangle ) + E ( \langle \xi, \Delta^p \rangle ) \label{eq:decompose_xi_gamma_0}\\
				= &\, E ( \langle \xi, \Delta - \Delta^p \rangle ) +   E ( \nleft\| \Delta^p \nright\|_2  \langle \xi,  \nleft\| \Delta^p \nright\|_2^{-1} \Delta^p \rangle ) \\
				\leq &\, E ( \langle \xi, \Delta - \Delta^p \rangle ) +  \sqrt{ E ( \nleft\| \Delta^p \nright\|_2^2  )} \sqrt{E ( \langle \xi, \nleft\| \Delta^p \nright\|_2^{-1} \Delta^p \rangle^2 )}, \label{eq:decompose_xi_gamma}
			\end{align}
			where \eqref{eq:decompose_xi_gamma} is thanks to the Cauchy-Schwarz inequality.   
			We shall next upper-bound the terms $E ( \langle \xi, \Delta - \Delta^p \rangle )$, $ E ( \nleft\| \Delta^p \nright\|_2^2  )$, and  $E ( \langle \xi, \nleft\| \Delta^p \nright\|_2^{-1} \Delta^p \rangle^2 )$, individually. First, note that 
			\begin{align}
				E ( \langle \xi , \Delta - \Delta^p \rangle ) \leq& \sqrt{E ( \langle \xi, \xi \rangle  )} \sqrt{E ( \langle \Delta - \Delta^p, \Delta - \Delta^p \rangle)} \label{eq:xi_residual_1} \\
				\leq &\,  \, n \Qradius, \label{eq:xi_residual_2}
			\end{align}
			where
   \eqref{eq:xi_residual_2} follows from $E ( \langle \xi, \xi \rangle  ) = E ( \sum_{i =1}^n \xi_i^2 ) = n$ and  $E ( \Delta - \Delta^p, \Delta - \Delta^p \rangle) = E ( \nleft\|\Delta - \Delta^p  \nright\|_2^2  ) \leq n \Qradius^2$ by  \eqref{eq:projection_assumption_Delta_1}-\eqref{eq:projection_assumption_Delta_4}. We proceed to upper-bound $E ( \nleft\| \Delta^p \nright\|_2^2  )$ as follows
			\begin{align}
				 E ( \nleft\| \Delta^p \nright\|_2^2  ) \leq&\,   E (  ( \nleft\| \Delta \nright\|_2 + \nleft\| \Delta^p - \Delta \nright\|_2    )^2)  \label{eq:pgamma_1}\\
				\leq &\,  E (  ( \nleft\| \Delta \nright\|_2 + n^{1\slash 2} \Qradius    )^2  ) \label{eq:pgamma_2}\\
				\leq   &\, E (2 \nleft\| \Delta  \nright\|^2_2 + 2 ( n^{1\slash 2} \Qradius)^2   ) \label{eq:pgamma_3}\\
				\leq &\, 2 E (\nleft\| \Delta  \nright\|^2_2) + 4 \sqrt{E (\nleft\| \Delta  \nright\|^2_2)} ( n^{1\slash 2} \Qradius )+ 2 ( n^{1\slash 2} \Qradius )^2   \label{eq:pgamma_4}\\
				=& \, 2 \Bigl( \sqrt{E (\nleft\| \Delta  \nright\|^2_2)} + n^{1\slash 2} \Qradius  \Bigr)^2, \label{eq:pgamma_5}
			\end{align}
			where in \eqref{eq:pgamma_2} we again used \eqref{eq:projection_assumption_Delta_1}-\eqref{eq:projection_assumption_Delta_4} 
   and \eqref{eq:pgamma_3} follows from $( a+b )^2 \leq 2 a ^2 + 2 b^2 $, with $a = \nleft\| \Delta \nright\|_2 $ and $b = n^{1\slash 2} \Qradius$. It remains to upper-bound  the term $E ( \langle \xi, \nleft\| \Delta^p \nright\|_2^{-1} \Delta^p \rangle^2 )$. To this end, we first note that
			\begin{equation}
			\label{eq:the_third_term}
				 E ( \langle \xi, \nleft\| \Delta^p \nright\|_2^{-1} \Delta^p \rangle^2 ) \leq E \biggl(  \max_{j = 1,\dots, N ( \Qradius,\, \mathcal{F}_n,\, {L^\infty ( \mathbb{X}  )}  )} \langle \xi, \nleft\| \Delta_j \nright\|_2^{-1} \Delta_j \rangle^2 \biggr).
			\end{equation}
			The right-hand-side of \eqref{eq:the_third_term} can now be further upper-bounded 
   through the moment generating function method. Specifically, let $t \in (0,1\slash 2)$ be a parameter to be determined later. We have  
			\begin{align}
				&E \biggl(\max_{j = 1,\dots, N ( \Qradius,\, \mathcal{F}_n,\, {L^\infty ( \mathbb{X}  )}  )} \langle \xi, \nleft\| \Delta_j \nright\|_2^{-1} \Delta_j \rangle^2\biggr)\label{eq:supzi_0}\\
				&=\, \frac{1}{t} \ln \Bigl( \exp\Bigl( E  \Bigl( t \max_{j = 1,\dots,  N ( \Qradius,\, \mathcal{F}_n,\, {L^\infty ( \mathbb{X}  )}  ) } \langle \xi, \nleft\| \Delta_j \nright\|_2^{-1} \Delta_j \rangle^2\Bigr) \Bigr) \Bigr) \label{eq:supzi_1}\\
				&\leq \, \frac{1}{t} \ln \Bigl( E \Bigl( \exp \Bigl( t \max_{j = 1,\dots,  N ( \Qradius,\, \mathcal{F}_n,\, {L^\infty ( \mathbb{X}  )}  ) } \langle \xi, \nleft\| \Delta_j \nright\|_2^{-1} \Delta_j \rangle^2 \Bigr) \Bigr) \Bigr) \label{eq:supzi_2}\\
				&= \, \frac{1}{t} \ln \Bigl( E \Bigl( \max_{j = 1,\dots,  N ( \Qradius,\, \mathcal{F}_n,\, {L^\infty ( \mathbb{X}  )}  ) }  \exp ( t  \langle \xi, \nleft\| \Delta_j \nright\|_2^{-1} \Delta_j \rangle^2 ) \Bigr)  \Bigr)\label{eq:supzi_3}\\
				&\leq \, \frac{1}{t} \ln \biggl( E \biggl( \sum_{j = 1,\dots,  N ( \Qradius,\, \mathcal{F}_n,\, {L^\infty ( \mathbb{X}  )}  ) }  \exp ( t  \langle \xi, \nleft\| \Delta_j \nright\|_2^{-1} \Delta_j \rangle^2 ) \biggr) \biggr) \label{eq:supzi_4}\\
				&= \, \frac{1}{t} \ln \biggl( \sum_{j = 1,\dots,  N ( \Qradius,\, \mathcal{F}_n,\, {L^\infty ( \mathbb{X}  )}  ) } E (   \exp ( t  \langle \xi, \nleft\| \Delta_j \nright\|_2^{-1} \Delta_j \rangle^2 ) ) \biggr), \label{eq:supzi_5}
			\end{align}
			where in \eqref{eq:supzi_2} we applied the Jensen inequality. For $j = 1,\dots, N ( \Qradius,\, \mathcal{F}_n,\, {L^\infty ( \mathbb{X}  )}  )$, conditioned on $( x_i )_{i =1}^n$,  $\nleft\| \Delta_j \nright\|_2^{-1} \Delta_j  = \frac{(  f_j (x_i)  - g ( x_i ) )_{i = 1}^n  }{\nleft\| (  f_j (x_i)  - g ( x_i ) )_{i = 1}^n \nright\|_2}$ is a deterministic vector of $\| \cdot \|_{2}$-norm equal to $1$ and hence 
   $\langle \xi, \nleft\| \Delta_j \nright\|_2^{-1} \Delta_j \rangle^2$ is a $\chi^2$ random variable with $1$ degree of freedom. We therefore get
   \begin{equation}
			\label{eq:conditional_mgf}
				E  ( \exp ( t  \langle \xi, \nleft\| \Delta_j \nright\|_2^{-1} \Delta_j \rangle^2 )  | ( x_i )_{i =1}^n) = ( 1 -  2t )^{-1\slash 2},  \quad \text{a.s.}
			\end{equation}
By the law of total expectation, it finally follows that
			\begin{equation*}
				E  ( \exp ( t  \langle \xi, \nleft\| \Delta_j \nright\|_2^{-1} \Delta_j \rangle^2 )  ) = ( 1 -  2t )^{-1 \slash 2}, 
			\end{equation*}
			for $j = 1,\dots, N ( \Qradius,\, \mathcal{F}_n,\, {L^\infty ( \mathbb{X}  )}  )$, which together with \eqref{eq:supzi_0}-\eqref{eq:supzi_5} yields 
			\begin{align}
				E \biggl(\max_{j = 1,\dots, N ( \Qradius,\, \mathcal{F}_n,\, {L^\infty ( \mathbb{X}  )}  )} \langle \xi, \nleft\| \Delta_j \nright\|_2^{-1} \Delta_j \rangle^2\biggr) \leq &\, \frac{1}{t} \ln  (  N ( \Qradius,\, \mathcal{F}_n,\, {L^\infty ( \mathbb{X}  )}  ) ( 1 - 2t )^{-1 \slash 2} ) \label{eq:E_max_1}\\
				= &\, \frac{\ln  N ( \Qradius,\, \mathcal{F}_n,\, {L^\infty ( \mathbb{X}  )}  )}{t} - \frac{\ln ( 1- 2t )}{2 t}.\label{eq:E_max_2}
			\end{align}
			Putting \eqref{eq:the_third_term}  and \eqref{eq:E_max_1}-\eqref{eq:E_max_2}  with $t = \frac{1}{10}$ together, we obtain
			\begin{equation}
			\label{eq:xipgamma_product}
			 	E ( \langle \xi, \nleft\| \Delta^p \nright\|_2^{-1} \Delta^p \rangle^2 ) \leq  \frac{\ln  N ( \Qradius,\, \mathcal{F}_n,\, {L^\infty ( \mathbb{X}  )}  )}{1 \slash 10} - \frac{\ln ( 4\slash 5 )}{1\slash 5} \leq 10 \log(  N ( \Qradius,\, \mathcal{F}_n,\, {L^\infty ( \mathbb{X}  )}  )) + 2.
			\end{equation} 
			Next, we have
			\begin{align}
				&E \biggl(\frac{ \nleft\| \Delta \nright\|_2^2 }{n}\biggr)\label{eq:Delta_square_1}\\
				&\leq \, \Qapproximation^2 + \Qnumeric + \frac{2 \sigma}{n} E ( \langle \xi, \Delta \rangle )\label{eq:Delta_square_2}\\
				&\leq \, \Qapproximation^2 + \Qnumeric + \frac{2 \sigma}{n} \biggl( E ( \langle \xi, \Delta - \Delta^p \rangle ) +  \sqrt{ E ( \nleft\| \Delta^p \nright\|_2^2  )} \sqrt{E ( \langle \xi, \nleft\| \Delta^p \nright\|_2^{-1} \Delta^p \rangle^2 )} \biggr)\label{eq:Delta_square_21}\\
				&\leq \, \Qapproximation^2 + \Qnumeric + \frac{2 \sigma}{n} \biggl( n \Qradius + \sqrt{2}  \biggl(  \sqrt{E (\nleft\| \Delta  \nright\|^2_2)} + n^{1\slash 2} \Qradius \biggr) \sqrt{{ 10} \log  (N ( \Qradius,\, \mathcal{F}_n,\, {L^\infty ( \mathbb{X}  )}  )) + 2}\biggr)\label{eq:Delta_square_22}\\
				&\leq \, \Qapproximation^2 + \Qnumeric + \frac{2\sigma}{n} \biggl( n \Qradius + {5}  \biggl(  \sqrt{E (\nleft\| \Delta  \nright\|^2_2)} + n^{1\slash 2} \Qradius \biggr) \sqrt{\log  (N ( \Qradius,\, \mathcal{F}_n,\, {L^\infty ( \mathbb{X}  )}  )) + 1} \biggr ) \\
				&= \, \Qapproximation^2 + \Qnumeric + 2 \sigma\Qradius   + {10} \sigma \sqrt{\frac{E (\nleft\| \Delta  \nright\|^2_2)}{n}}\sqrt{\frac{\log  (N ( \Qradius,\, \mathcal{F}_n,\, {L^\infty ( \mathbb{X}  ) }  ))  + 1}{n}} \nonumber \nonumber \\
				&\,\, + {10}\sigma  \Qradius  \sqrt{\frac{\log  (N ( \Qradius,\, \mathcal{F}_n,\, {L^\infty ( \mathbb{X}  ) }  ))  + 1}{n}}, \label{eq:Delta_square_5}
			\end{align}
			where \eqref{eq:Delta_square_2} is \eqref{eq:new_basic_inequality}, in \eqref{eq:Delta_square_21} we used \eqref{eq:decompose_xi_gamma_0}-\eqref{eq:decompose_xi_gamma}, and \eqref{eq:Delta_square_22} follows from \eqref{eq:xi_residual_1}-\eqref{eq:xi_residual_2}, \eqref{eq:pgamma_1}-\eqref{eq:pgamma_5}, and \eqref{eq:xipgamma_product}.
			Regarding the last two terms in \eqref{eq:Delta_square_5},  we note that first
			\begin{equation}
			\label{eq:ab_1}
			\begin{aligned}
				&\,{10} \sigma \sqrt{\frac{E (\nleft\| \Delta  \nright\|^2_2)}{n}}\sqrt{\frac{\log  (N ( \Qradius,\, \mathcal{F}_n,\, {L^\infty ( \mathbb{X}  ) }  ))  + 1}{n}}\\
				&\leq\, \frac{E (\nleft\| \Delta  \nright\|^2_2)}{2n} + {50}  \sigma^2 \frac{\log  (N ( \Qradius,\, \mathcal{F}_n,\, {L^\infty ( \mathbb{X}  ) }  ))  + 1}{n},
			\end{aligned}
			\end{equation}
			where we  used $ab \leq \frac{a^2 + b^2}{2} $, with $a = \sqrt{\frac{E (\nleft\| \Delta  \nright\|^2_2)}{n}}$ and $b = {10}\sigma \sqrt{\frac{\log  (N ( \Qradius,\, \mathcal{F}_n,\, {L^\infty ( \mathbb{X}  ) }  ))  + 1}{n}}$, and second
			\begin{equation}
			\label{eq:ab_2}
				{10}\sigma  \Qradius  \sqrt{\frac{\log  (N ( \Qradius,\, \mathcal{F}_n,\, {L^\infty ( \mathbb{X}  ) }  ))  + 1}{n}} \leq 2 \sigma \Qradius^2 + {\frac{25}{2}}\sigma \frac{\log  (N ( \Qradius,\, \mathcal{F}_n,\, {L^\infty ( \mathbb{X}  ) }  ))  + 1}{n},
			\end{equation}
			which again is by $ab \leq \frac{a^2 + b^2}{2} $, here with $a = 2\sigma^{1\slash 2} \Qradius$ and $b = {5} \sigma^{1\slash 2} \sqrt{\frac{\log  (N ( \Qradius,\, \mathcal{F}_n,\, {L^\infty ( \mathbb{X}  ) }  ))  + 1}{n}}$. Substituting \eqref{eq:ab_1} and \eqref{eq:ab_2} into \eqref{eq:Delta_square_1}-\eqref{eq:Delta_square_5}, rearranging terms, and using $\Qradius^2 \leq \Qradius$,
   yields
			\begin{align}
				&E \biggl(\frac{ \nleft\| \Delta \nright\|_2^2 }{n}\biggr)\\
				&= \, 2 \Qapproximation^2 + 2 \Qnumeric + 4  \sigma \Qradius  + 4 \sigma \Qradius^2 + {100}  \sigma^2 \frac{\log  (N ( \Qradius,\, \mathcal{F}_n,\, {L^\infty ( \mathbb{X}  ) }  ))  + 1}{n}   \\
				&\,\,+ {25}\sigma \frac{\log  (N ( \Qradius,\, \mathcal{F}_n,\, {L^\infty ( \mathbb{X}  ) }  ))  + 1}{n} \nonumber\\
				&\leq \, 2( \Qapproximation^2 + \Qnumeric) + 8 \sigma \Qradius + {100} ( \sigma + \sigma^2 ) \frac{\log  (N ( \Qradius,\, \mathcal{F}_n,\, {L^\infty ( \mathbb{X}  ) }  ))  + 1}{n} \label{eq:delta_delta2}.
			\end{align} 
   The proof is concluded upon noting that $E (\frac{ \nleft\| \Delta \nright\|_2^2 }{n}) = E ( \frac{1}{n} \sum_{i = 1}^n ( \hat{f}_n (x_i) - g ( x_i ) )^2 )$.

	\subsection{Proof of Lemma~\ref{lem:population_and_empirical}}
	\label{sub:convergence_of_the_empirical_process}

		We start with one-step discretization \cite[Section 5.3]{wainwright2019high} for $\mathcal{F}_n$. Let $\{ f_j \}_{j = 1}^{N ( \Qradius,\, \mathcal{F}_n,\, {L^\infty ( \mathbb{X}  )}  )} $ be a $\Qradius$-covering of $\mathcal{F}_n$ with respect to the $L^\infty ( \mathbb{X}  )$-norm. Then, for each $f \in \mathcal{F}_n$, there exists an index $j(f) \in \left\{ 1,\dots,  N ( \Qradius,\, \mathcal{F}_n,\, {L^\infty ( \mathbb{X}  )}  )\right\}$ such that
		\begin{equation}
		\label{eq:covering_for_F_empirical_process}
			\nleft\| f - f_{j ( f )} \nright\|_{L^\infty ( \mathbb{X}  )} \leq \Qradius.
		\end{equation}
		Next, we have, a.s., 
		\begin{align}
			\nleft\| \hat{f}_n - g  \nright\|^2_{L^2 ( P )} =&\, \int  ( \hat{f}_n (x) - g ( x ) )^2 d P(x)  \label{eq:hatfg} \\
			 \leq &\, \int  \Bigl( 2 \bigl( \hat{f}_n (x) - f_{j( \hat{f}_n )} ( x )  \bigr)^2 +  2 \bigl( f_{j ( \hat{f}_n )} (x) - g (x)   \bigr)^2 \Bigr) d P(x) \label{eq:hatfg_1}\\
			 = &\,  2  \nleft\| \hat{f}_n - f_{j( \hat{f}_n )} \nright\|_{L^2 ( P )}^2   + 2 \nleft\|  f_{j(\hat{f}_n)} - g   \nright\|_{L^2 ( P )}^2 \\
			 \leq&\, 2 \Qradius^2  + 2 \nleft\|  f_{j(\hat{f}_n)} - g   \nright\|_{L^2 ( P )}^2 \label{eq:hatfg_2},
		\end{align}
		where in \eqref{eq:hatfg_1} we used $( a+b )^2 \leq 2 a^2 + 2b^2$, and \eqref{eq:hatfg_2} follows from \eqref{eq:covering_for_F_empirical_process}. Taking expectations in \eqref{eq:hatfg}-\eqref{eq:hatfg_2}, yields 
		\begin{equation}
		\label{eq:hatfg_expectation_1}
			E ( \nleft\| \hat{f}_n - g  \nright\|_{L^2 ( P )}^2 ) \leq 2 \Qradius^2  + 2 E \bigl( \nleft\|  f_{j(\hat{f}_n)} - g   \nright\|_{L^2 ( P )}^2 \bigr).
		\end{equation}
		Moreover, we have 
		\begin{align}
			&\,E \biggl( \frac{1}{n} \sum_{i = 1}^n ( \hat{f}_n (x_i) - g ( x_i ) )^2 \biggr)\label{eq:hatfgsqaure_0}\\
			&\geq\, E \biggl( \frac{1}{n} \sum_{i = 1}^n \biggl( \frac{1}{2} \bigl( f_{j ( \hat{f}_n )} (x_i) - g ( x_i ) \bigr)^2 - \bigl(f_{ j(\hat{f}_n)} ( x_i ) - \hat{f}_n ( x_i ) \bigr)^2  \biggr) \biggr) \label{eq:hatfgsqaure_1}\\
			&\geq\, \frac{1}{2} E \biggl(  \frac{1}{n} \sum_{i = 1}^n (  f_{j ( \hat{f}_n )} (x_i) - g ( x_i ) )^2   \biggr) - \Qradius^2, \label{eq:hatfgsqaure_2}
		\end{align}
		where \eqref{eq:hatfgsqaure_1} follows from $ a^2 \geq \frac{1}{2} ( a+b )^2 - b^2$, with $a =  \hat{f}_n (x_i) - g ( x_i ) $ and $b = f_{ j(\hat{f}_n)} ( x_i ) - \hat{f}_n ( x_i ) $, for $i = 1,\dots, n$, and in \eqref{eq:hatfgsqaure_2} we used \eqref{eq:covering_for_F_empirical_process}.
  We next note that
		\begin{align}
			&\,E ( \nleft\| \hat{f}_n - g  \nright\|_{L^2 ( P )}^2  )  - 8 E \biggl( \frac{1}{n} \sum_{i = 1}^n ( \hat{f}_n (x_i) - g ( x_i ) )^2 \biggr) \label{eq:hatfg_reduce_0} \\
			&\leq\,   2 \Qradius^2 + 2  E ( \nleft\|  f_{j(\hat{f}_n)} - g   \nright\|_{L^2 ( P )}^2)  - 4 E \biggl(  \frac{1}{n} \sum_{i = 1}^n ( f_{j ( \hat{f}_n )} (x_i) - g ( x_i ) )^2  \biggr)    + 8 \Qradius^2 \label{eq:hatfg_reduce_1} \\
			&=\,  2 E  \biggl(  \nleft\|  f_{j(\hat{f}_n)} - g   \nright\|_{L^2 ( P )}^2 -  \frac{2}{n} \sum_{i = 1}^n ( f_{j ( \hat{f}_n )} (x_i) - g ( x_i ) )^2  \biggr) + 10 \Qradius^2 \label{eq:hatfg_reduce_3} \\
			&\leq \, 2 E  \biggl( \sup_{j = 1,\dots,N ( \Qradius,\, \mathcal{F}_n, \, {L^\infty ( \mathbb{X}  )}  ) } \biggl( \nleft\|  f_{j} - g   \nright\|_{L^2 ( P )}^2-   \frac{2}{n} \sum_{i = 1}^n ( f_{j} (x_i) - g ( x_i ) )^2   \biggr)\biggr) + 10 \Qradius^2, \label{eq:hatfg_reduce_4} 
		\end{align}
		where \eqref{eq:hatfg_reduce_1} follows from  \eqref{eq:hatfg_expectation_1} and \eqref{eq:hatfgsqaure_0}-\eqref{eq:hatfgsqaure_2}.
  To simplify notation, in the following, we let 
		\begin{equation}
		\label{eq:definition_Z}
			Z_{j,i} =  \frac{( f_{j} (x_i) - g ( x_i )  )^2}{4(\Qlinfty ( g, \mathcal{F}_n ))^2}, \quad i = 1,\dots,n, \quad j = 1,\dots,N ( \Qradius, \,\mathcal{F}_n, \,{L^\infty ( \mathbb{X}  )}  ).
		\end{equation}
  Next, note that, for fixed $j$,  $\{ Z_{j,i} \}_{i = 1}^n$ are i.i.d nonnegative random variables with mean
		\begin{equation}
		\label{eq:define_Z_mean}
			\mu_j = E ( Z_{j,1} ) = \int \frac{( f_{j} (x) - g ( x )  )^2}{4(\Qlinfty ( g, \mathcal{F}_n ))^2}  d P(x) =  \frac{1}{4(\Qlinfty ( g, \mathcal{F}_n ))^2} \nleft\|  f_{j} - g   \nright\|_{L^2 ( P )}^2.
		\end{equation}
		Moreover, for $j  = 1,\dots,N ( \Qradius, \,\mathcal{F}_n, \,{L^\infty ( \mathbb{X}  )}  ) $, $i = 1,\dots, n$, we have 
		\begin{equation}
		\label{eq:bounded_Zji}
			Z_{j,i} \in [ 0,1], \quad \text{a.s.} 
		\end{equation}
		as a consequence of $\nleft\| f_j - g \nright\|_{L^\infty ( \mathbb{X}  )} \leq \nleft\| f_j\nright\|_{L^\infty ( \mathbb{X}  )}  + \nleft\| g \nright\|_{L^\infty ( \mathbb{X}  )} \leq 2 \Qlinfty ( g, \mathcal{F}_n ) $. We can now rewrite
  \eqref{eq:hatfg_reduce_0}-\eqref{eq:hatfg_reduce_4} as
		\begin{align}
			&E ( \nleft\| \hat{f}_n - g  \nright\|_{L^2 ( P )}^2  )  - 8 E \biggl( \frac{1}{n} \sum_{i = 1}^n ( \hat{f}_n (x_i) - g ( x_i ) )^2 \biggr) \label{eq:hatfg_decompose_1}\\
			&\leq \,  8 (\Qlinfty ( g, \mathcal{F}_n ))^2  E  \biggl( \sup_{j = 1,\dots,N ( \Qradius,\, \mathcal{F}_n,\, {L^\infty ( \mathbb{X}  )}  ) } \biggl( \mu_j -   \frac{2}{n} \sum_{i = 1}^n Z_{j,i}  \biggr)\biggr) + 10\, \Qradius^2 \label{eq:hatfg_decompose_2}
		\end{align}
		and use the following lemma to upper-bound $E ( \sup_{j = 1,\dots,N ( \Qradius,\, \mathcal{F}_n, \,{L^\infty ( \mathbb{X}  )}  ) } ( \mu_j -   \frac{2}{n} \sum_{i = 1}^n Z_{j,i}  )) $.
	
		\begin{lemma}
		\label{lem:expectation_bound_Z}
			Let $U,V\in \mathbb{N}$. For $j = 1,\dots, U $, let $\{ Z_{j,i}  \}_{i = 1}^V$ be $i.i.d$ nonnegative random variables of means $\mu_j$ and taking values in $[0,1]$ a.s.  Then,
			\begin{equation}
			\label{eq:stronger_upper_bounds}
				E \biggl( \sup_{j = 1, ..., U} \biggl( \mu_j  - \frac{2}{V} \sum_{i = 1}^V Z_{j,i} \biggr) \biggr) \leq \frac{{8} \log (U)}{V}.
			\end{equation}
			\begin{proof}
				See Section~\ref{sub:proof_of_lemma_expectation_bound_Z}.
			\end{proof}
		\end{lemma}
		Application of Lemma~\ref{lem:expectation_bound_Z} with $U = N ( \Qradius,\, \mathcal{F}_n, \,{L^\infty ( \mathbb{X}  )} )$, $V = n$, and the prerequisite satisfied thanks to \eqref{eq:bounded_Zji}, yields 
		\begin{equation}
			 E  \biggl( \sup_{j = 1,\dots,N ( \Qradius, \,\mathcal{F}_n,\, {L^\infty ( \mathbb{X}  )}  ) } \biggl( \mu_j -   \frac{2}{n} \sum_{i = 1}^n Z_{j,i}  \biggr)\biggr) \leq  \frac{{8} \log (N ( \Qradius,\, \mathcal{F}_n, \,{L^\infty ( \mathbb{X}  )} ))}{n},
		\end{equation}
		which together with \eqref{eq:hatfg_decompose_1}-\eqref{eq:hatfg_decompose_2} finishes the proof of Lemma~\ref{lem:population_and_empirical} according to
		\begin{align*}
			&E ( \nleft\| \hat{f}_n - g  \nright\|_{L^2 ( P )}^2  )  - 8 E \biggl( \frac{1}{n} \sum_{i = 1}^n ( \hat{f}_n (x_i) - g ( x_i ) )^2 \biggr) \\
			&\leq\,   8 (\Qlinfty ( g, \mathcal{F}_n ))^2    \frac{{8} \log (N ( \Qradius, \,\mathcal{F}_n,\, {L^\infty ( \mathbb{X}  )} ))}{n} + 10 \Qradius^2 \\
			& \leq  {64} \biggl((\Qlinfty ( g, \mathcal{F}_n ))^2 \frac{\log (N ( \Qradius,\, \mathcal{F}_n,\, {L^\infty ( \mathbb{X}  )} ))}{n} + \Qradius^2\biggr).
		\end{align*}

We finally note that the standard upper bound on sub-Gaussian maxima, as e.g. in \cite[Exercise 2.12]{wainwright2019high}, would yield $E ( \sup_{j = 1, ..., U} ( \mu_j  - \frac{1}{V} \sum_{i = 1}^V Z_{j,i} ) ) \leq C \sqrt{\frac{\log(U)}{V}}$, for some absolute constant $C$. Identifying $\log(U)/V$ with
the right-most term in the prediction error upper bound (\ref{eq:generalization_error}) in Theorem~\ref{thm:expected_L2_error}, given by $(\log  (N (\varepsilon^2  , \mathcal{F}_n,\, L^\infty ( \mathbb{X}  )  ))  + 1)/n$, as was essentially done above to finish the proof of Lemma~\ref{lem:population_and_empirical},
we can see that as the prediction error upper bound goes to zero, we will be in the regime $\log(U)/V < 1$.
In this case Lemma~\ref{lem:expectation_bound_Z} provides an upper bound that is stronger order-wise than the standard bound on sub-Gaussian maxima.
This improvement is fundamental in our context as we will want the prediction error to approach zero
at a certain rate.

        \subsection{Proof of Lemma~\ref{lem:expectation_bound_Z}}
        \label{sub:proof_of_lemma_expectation_bound_Z}

		\newcommand{\Ij}{j}
		\newcommand{\Ii}{i}

			Let
			\begin{align}
				Q :=& \, \sup_{\Ij = 1, ..., U} \biggl( \mu_j  - \frac{2}{V} \sum_{\Ii = 1}^V Z_{\Ij,\Ii} \biggr) \label{eq:definition_R_1} \\
				=& \, 4  \sup_{\Ij  = 1, ..., U}  \biggl( \frac{1}{V} \sum_{\Ii = 1}^V  \frac{1}{2}( \mu_\Ij  -  Z_{\Ij,\Ii} )  - \frac{1}{4}\mu_\Ij \biggr). \label{eq:definition_R_2}
			\end{align}
			For $\Ij \in \{1,\dots, U \}$, let
			\begin{equation}
			\label{eq:define_T}
				T_{\Ij,\Ii} =  \frac{1}{2}( \mu_\Ij  -  Z_{\Ij,\Ii} ), \quad  i = 1,\dots,V,
			\end{equation}
			and note that, for fixed $\Ij$,  $\{ T_{\Ij,\Ii} \}_{\Ii = 1}^V$ are i.i.d. random variables with
			\begin{align}
				E ( T_{\Ij,\Ii} ) = E \biggl( \frac{1}{2}\biggl( \mu_\Ij  -  Z_{\Ij,\Ii} \biggr) \biggr) = 0,\label{eq:T_mean_zero}\\
				\nleft| T_{\Ij,\Ii}  \nright|  \leq \frac{1}{2} ( \nleft| \mu_\Ij \nright| + \nleft| Z_{\Ij,\Ii} \nright| ) \leq 1,\label{eq:T_bounded_1} \quad \text{ a.s.}
			\end{align}
            and variance
			\begin{align}
				\sigma^2_\Ij = &\, E \biggl( \biggl(  \frac{1}{2}( \mu_\Ij  -  Z_{\Ij,i} \bigr) \biggr)^2\biggr) \label{eq:variance_sigma_1} \\
				=&\, \frac{1}{4} E ( ( Z_{\Ij,i}^2)  - \frac{1}{4} \mu_\Ij^2 \label{eq:variance_sigma_2}\\
				\leq &\, \frac{1}{4} E ( Z_{\Ij,i} ) \label{eq:variance_sigma_3}\\
				\leq &\, \frac{1}{4} \mu_\Ij, \label{eq:variance_sigma_4}
			\end{align}
			where
   \eqref{eq:variance_sigma_3} is by $Z_{\Ij, i} \in [0,1]$ a.s.
   With
   \begin{equation*}
				S : = \sup_{\Ij = 1, ..., U} \biggl( \frac{1}{V} \sum_{\Ii =1 }^V T_{\Ij,\Ii}  - \sigma^2_{\Ij} \biggr),
	\end{equation*}
   it follows from 
   \eqref{eq:variance_sigma_1}-\eqref{eq:variance_sigma_4} that
			\begin{equation}
			\label{eq:relation_RS}
				Q \leq 4 S,  \quad \text{a.s.}
			\end{equation}

			We next upper-bound $E ( S )$ by upper-bounding its moment generating function $\exp ( E ( tS )), t \in (0,V)$, and first note that
			\begin{equation}
			\label{eq:chernoff_bound_1}
				E ( S ) = \frac{1}{t} \ln (\exp ( E ( tS ))) \leq \frac{1}{t} \ln ( E ( \exp ( t S ) ) ),
			\end{equation}
			where we used the Jensen inequality.
   Next, we have
   			\begin{align}
				E ( \exp(t S) ) =&\, E \biggl( \exp \biggl(t \sup_{\Ij = 1, ..., U} \biggl( \frac{1}{V} \sum_{\Ii =1 }^V T_{\Ij,\Ii}  - \sigma^2_{\Ij} \biggr)\biggr)\biggr)\label{eq:mgf_S_1}\\
				= &\,  E \biggl( \sup_{\Ij  = 1, ..., U} \exp \biggl(t \biggl( \frac{1}{V} \sum_{\Ii =1 }^V T_{\Ij,\Ii}  - \sigma^2_{\Ij} \biggr)\biggr)\biggr) \label{eq:mgf_S_2}\\
				\leq &\, E \biggl( \sum_{\Ij  = 1}^{U} \exp \biggl(t \biggl( \frac{1}{V} \sum_{\Ii =1 }^V T_{\Ij,\Ii}  - \sigma^2_{\Ij} \biggr)\biggr)\biggr) \label{eq:mgf_S_3}\\
				= &\, \sum_{\Ij  = 1}^U E \biggl(  \exp \biggl(t \biggl( \frac{1}{V} \sum_{\Ii =1 }^V T_{\Ij,\Ii}  - \sigma^2_{\Ij} \biggr) \biggr)\biggr) \label{eq:mgf_S_4}\\
				= &\, \sum_{\Ij  = 1}^U E \biggl(  \exp ( - t \sigma_\Ij^2 )\prod_{\Ii = 1}^V \exp \bigg( \frac{t}{V} T_{\Ij,\Ii} \biggr) \biggr)\label{eq:mgf_S_5}\\
				= &\, \sum_{\Ij  = 1}^U \exp ( - t \sigma_\Ij^2 )\prod_{\Ii = 1}^V E \biggl( \exp \bigg( \frac{t}{V} T_{\Ij,\Ii} \biggr) \biggr). \label{eq:mgf_S_6}
			\end{align}
To upper-bound $E ( \exp ( \frac{t}{V} T_{\Ij,\Ii} ) )$, for $\Ij = 1,\dots, U$, $\Ii  = 1,\dots, V$, we note that, for $\lambda \in [0,1)$, 
			\begin{align}
				E \biggl( \exp \biggl( \lambda T_{\Ij,\Ii} \biggr) \biggr) =&\, E \biggl( \sum_{k = 0}^\infty \frac{ (  \lambda T_{\Ij,\Ii} )^k }{k!}  \biggr) \label{eq:bernstein_1} \\
				= &\, E \biggl( 1  + \lambda T_{\Ij,\Ii} + \sum_{k = 2}^\infty \frac{ (  \lambda T_{\Ij,\Ii} )^k }{k!}  \biggr)\label{eq:bernstein_2} \\
				\leq & \,  1  +  E \biggl( \sum_{k = 2}^\infty \frac{  \lambda  ^k ( T_{\Ij,\Ii} )^2 }{k!}  \biggr) \label{eq:bernstein_3}\\
				= & \,  1  +  E \Bigl(   T_{\Ij,\Ii} ^2 \Bigr) \sum_{k = 2}^\infty \frac{  \lambda  ^k  }{k!}  \label{eq:bernstein_4}\\
				\leq & \,  1  +  \sigma_\Ij^2  \lambda^2 \sum_{k = 2}^\infty \frac{  \lambda ^{k-2}  }{2} \label{eq:bernstein_5}\\
				\leq &\, 1 + \frac{\sigma_\Ij^2  \lambda ^2 }{2 ( 1 -  \lambda   )}\label{eq:bernstein_6}\\
				\leq &\, \exp \biggl( \frac{\sigma_j^2  \lambda ^2 }{2 ( 1 -  \lambda   )} \biggr), \label{eq:bernstein_7}
			\end{align}
			where
   \eqref{eq:bernstein_3} follows from $\nleft| T_{\Ij, \Ii} \nright| \leq 1$ a.s,
   in \eqref{eq:bernstein_6} we used $\sum_{k = 0}^\infty \lambda^k = \frac{1}{1 - \lambda}$, for $\lambda \in [0,1)$, and \eqref{eq:bernstein_7} is by $ 1 + x \leq \exp ( x )$, $x \in \mathbb{R}$. Using \eqref{eq:bernstein_1}-\eqref{eq:bernstein_7}, with $\lambda = \frac{t}{V}$,  in \eqref{eq:mgf_S_6}, we obtain, for $t \in (0, V)$, 
			\begin{align*}
				E ( \exp(t S) ) \leq&\, \sum_{\Ij = 1}^U \exp \biggl( -t \sigma_\Ij^2 + V \biggl(  \frac{\sigma_\Ij^2  ( \frac{t}{V} )^2 }{2 ( 1 -  \frac{t}{V}   )}  \biggr) \biggr) \\
				=&\, \sum_{\Ij = 1}^U \exp \biggl( V \sigma_\Ij^2 \biggl( -\frac{t}{V}  + \biggl(  \frac{  ( \frac{t}{V}  )^2 }{2 ( 1 -  \frac{t}{V}   )}  \biggr) \biggr) \biggr),
			\end{align*}
			which, by setting $t = \frac{V}{2}$, yields
			\begin{equation}
			\label{eq:upperbound_mgf_S}
				E \biggl( \exp\biggl (\frac{V}{2} S \biggr) \biggr) \leq
    \sum_{\Ij = 1}^U \exp\biggl(- \frac{V \sigma_\Ij^2}{4}\biggr) \leq U.
			\end{equation}
			Finally, using \eqref{eq:upperbound_mgf_S} in \eqref{eq:chernoff_bound_1}, with $t= \frac{V}{2}$, we obtain
			\begin{equation}
			\label{eq:chernoff_bound_2}
				E ( S )
    \leq \frac{2 \ln ( U )}{V} \leq \frac{ {2} \log ( U ) }{V}.
			\end{equation}
			The proof is concluded upon noting that 
			\begin{equation*}
				E \biggl( \sup_{\Ij = 1, ..., U} \biggl( \mu_j  - \frac{2}{V} \sum_{\Ii = 1}^V Z_{\Ij,\Ii} \biggr) \biggr) = E ( Q ) \leq 4 E ( S ) \leq \frac{ {8} \log ( U ) }{V},
 			\end{equation*}
 			where $Q \leq 4 S$ a.s. is by \eqref{eq:relation_RS}.

\section{Proof of Theorem~\ref{thm:covering_number_sparse}}
\label{sec:proof_of_results_in_section_sec:sparsely_connected_relu_networks}

	We prove the upper bound \eqref{eq:upper_bound_sparse} and the lower bound \eqref{eq:lower_bound_sparse} in Appendices~\ref{sub:proof_of_proposition_proposition:covering_number_upper_bound_sparse} and \ref{sub:proof_of_proposition_proposition:covering_number_lower_bound_sparse}, respectively. Before delving into the proofs, 
we note that it suffices to consider the case $s < L(W^2+W)$ as otherwise the network would qualify as fully connected, i.e., all weights may be nonzero, and
Theorem~\ref{thm:covering_number_upper_bound_fully_connected_bounded_weight} applies. Further, for simplicity of exposition and consistency with Theorem~\ref{thm:covering_number_upper_bound_fully_connected_bounded_weight}, we decided to work with the quantity $LW^2$ throughout as opposed to $L(W^2+W)$, simply by using $LW^2 \leq L(W^2+W) \leq 2LW^2$.

	\subsection{Proof of the Upper Bound \eqref{eq:upper_bound_sparse}}
	\label{sub:proof_of_proposition_proposition:covering_number_upper_bound_sparse}

		The overall proof architecture is identical to that of the upper bound \eqref{eq:upper_bound_fully_connected_bounded_output} in Theorem~\ref{thm:covering_number_upper_bound_fully_connected_bounded_weight}.
Specifically, we construct an explicit $\varepsilon$-covering of $\mathcal{R} (d,W,L,B,s)$ with elements in \linebreak
$\mathcal{R}_{[-B,B] \cap 2^{-b} \mathbb{Z}}(d,W,L,\infty, s)$\footnote{We note that $\mathcal{R}_{[-B,B] \cap 2^{-b} \mathbb{Z}}(d,W,L,\infty, s) = \mathcal{R}_{[-B,B] \cap 2^{-b} \mathbb{Z}}(d,W,L,B, s)$. We chose, however, to use the notation $\mathcal{R}_{[-B,B] \cap 2^{-b} \mathbb{Z}}(d,W,L,\infty, s)$ for consistency with later parts of the proof involving $\mathcal{R}_{\mathbb{A}}(d,W,L,\infty, s)$ with general $\mathbb{A} \subseteq \mathbb{R}$. }, where $b \in \mathbb{N}$ is a parameter suitably depending on $\varepsilon$ and determined later. 

The proof requires two preparatory technical elements, the first of which, namely
Lemma~\ref{lem:define_covering_sparse}, is very similar to Lemma~\ref{lem:define_covering_fully_connected}.

		\begin{lemma}
			\label{lem:define_covering_sparse}
			Let $p \in [1,\infty]$, $d,W,L, s, b\in \mathbb{N}$, and $B \in \mathbb{R}_+$ with  $B \geq 1$. Then, the set $\mathcal{R}_{[-B,B] \cap 2^{-b} \mathbb{Z}}(d,W,L,\infty,s)$ is an $(L (W+1)^L B^{L-1} 2^{-b})$-covering of $\mathcal{R}(d,W,L,B,s)$ with respect to the $L^p ( [0,1]^d )$-norm.
			\begin{proof}
				Let $q_b: [-B,B] \rightarrow [-B,B] \cap 2^{-b} \mathbb{Z} $ be defined as
                        \begin{equation*}
    					q_{b} ( x ) = \left \{ \begin{aligned}
    						\,2^{-b} \lfloor 2^b x \rfloor,\quad & \text{ for } x \in [0,B],\\
    						\,2^{-b} \lceil 2^b x \rceil, \quad  & \text{ for } x \in [-B,0),
    					\end{aligned}
                            \right.
    				\end{equation*}
                and note that $| x - q_b ( x ) | \leq 2^{-b} $, for all $x \in [-B , B]$. Arbitrarily fix $f \in\mathcal{R}(d,W,L,B,s)$. By definition, there exists $\Phi  =  (( A_\ell,\nmathbf{b}_\ell )  )_{\ell = 1}^{\widetilde{L}} \in \mathcal{N}(d,W,L,B, s)$, with $\widetilde{L} \leq L$,  such that $R ( \Phi ) = f$. We now quantize the weights of $\Phi$ according to\footnote{Note that $\mathcal{N}_{[-B,B] \cap 2^{-b} \mathbb{Z}}(d,W,L,\infty, s) = \mathcal{N}_{[-B,B] \cap 2^{-b} \mathbb{Z}}(d,W,L,B, s)$. We chose, however, to use the notation $\mathcal{R}_{[-B,B] \cap 2^{-b} \mathbb{Z}}(d,W,L,\infty, s)$ for consistency with later parts of the proof.}
				\begin{equation*}
					Q_b ( \Phi ) = (( q_b ( A_\ell ),q_b ( b_\ell ) )  )_{\ell = 1}^{\widetilde{L}} \in \mathcal{N}_{[-B,B] \cap 2^{-b} \mathbb{Z}}(d,W,L,\infty,s),
				\end{equation*}
				where $q_b$ acts elementwise. Here, we used the fact that, owing to $q_{b} ( 0 ) = 0$, the connectivity of $Q_b ( \Phi )$ is no greater than that of $\Phi$.
				Next, note that 
				\begin{equation*}
					\| \Phi - Q_b ( \Phi ) \| = \max_{\ell = 1, \dots, \widetilde{L}} \max \bigl\{ \| A_\ell - q_b ( A_\ell )\|_\infty,  \| b_\ell - q_b ( b_\ell )\|_\infty  \bigr\}  \leq 2^{-b},
				\end{equation*}
				which together with Lemma~\ref{lem:quantization} yields 
				\begin{equation}
				\label{eq:proof_bounded_fully_sparse_1}
					\|  R ( \Phi )  -  R ( Q_b ( \Phi ) )  \|_{L^\infty ( [0,1]^d )} \leq L (W+1)^L B^{L-1} \|  \Phi   -  Q_b ( \Phi )   \| \leq L (W+1)^L B^{L-1} 2^{-b}.
				\end{equation}
				As 
    $$
    \|  R ( \Phi )  -  R ( Q_b ( \Phi ) )  \|_{L^p ( [0,1]^d )} \leq \sup_{x \in [0,1]^d} |  R ( \Phi ) (x) -  R ( Q_b ( \Phi ) ) (x)  | = \|  R ( \Phi )  -  R ( Q_b ( \Phi ) )  \|_{L^\infty ( [0,1]^p )},
    $$
    it follows from  \eqref{eq:proof_bounded_fully_sparse_1} that 
				\begin{equation}
					\|  R ( \Phi )  -  R ( Q_b ( \Phi ) )  \|_{L^p ( [0,1]^d )} \leq L (W+1)^L B^{L-1} 2^{-b}.
				\end{equation}
				As $f \in\mathcal{R}(d,W,L,B,s)$ was chosen arbitrarily and $R ( Q_b ( \Phi ) ) \in \mathcal{R}_{[-B,B] \cap 2^{-b} \mathbb{Z}}(d,W,L,\infty,s) \subseteq \mathcal{R}(d,W,L,B,s)$, we can conclude that $\mathcal{R}_{[-B,B] \cap 2^{-b} \mathbb{Z}}(d,W,L,\infty,s)$ is an $(L (W+1)^L B^{L-1} 2^{-b})$-covering of $\mathcal{R}(d,W,L,B,s)$ with respect to the $L^p ( [0,1]^d )$-norm.
			\end{proof}
		\end{lemma}

		In the second preparatory step, we upper-bound the cardinality of the covering just identified, specifically we shall consider sets 
  $\mathcal{R}_{\mathbb{A}}(d,W,L,\infty,s)$ with general $\mathbb{A}$ and then particularize the result to $\mathbb{A}=[-B,B] \cap 2^{-b} \mathbb{Z}$ below.

		\begin{lemma}
			\label{lem:counting_cardinality_sparse}
			Let $d,W,L, s \in \mathbb{N}$, and $\mathbb{A} \subseteq \mathbb{R}$, with $s \geq \max \{ W,L \}$ and $\nleft| \mathbb{A} \nright| \geq 2$. Then,
			\begin{equation}
			\label{eq:counting_cardinality_sparse}
		 	\begin{aligned}
				\log  (\nleft|  \mathcal{R}_\mathbb{A} ( d,W,L,\infty, s )  \nright| ) \leq &\,\log (\nleft|  \mathcal{N}_\mathbb{A} ( d,W,L,\infty, s )  \nright| )\\
				\leq&\,\, 5s ( \log (L ( W+1 ))  + \log ( \nleft| \mathbb{A} \nright|  )).
			\end{aligned}
			\end{equation}

			\begin{proof}
				By definition, 
				\begin{align*}
					&\,\mathcal{N}_\mathbb{A} ( d,W,L,\infty, s )\\
					&=\, \{ ( A_\ell, b_\ell)_{\ell = 1}^{\widetilde{L}} \in \mathcal{N} (d): \mathcal{W} (( A_\ell, b_\ell)_{\ell = 1}^{\widetilde{L}}) \leq W,\, \widetilde{L} \leq L,\, \mathcal{M} (( A_\ell, b_\ell)_{\ell = 1}^{\widetilde{L}}) \leq s,\, \coef ( ( A_\ell, b_\ell)_{\ell = 1}^{\widetilde{L}} ) \subseteq \mathbb{A}  \}.
				\end{align*}
				There are at most  $\sum_{\widetilde{L} = 1}^{L} W^{\widetilde{L}} \leq L W^{L}$ different architectures $( {N}_0, \dots, {N}_{\widetilde{L}} )$\footnote{Note that $N_0 = d$ is fixed. },  $\widetilde{L} \leq L$, for {network configurations} in $\mathcal{N}_\mathbb{A} ( d,W,L,\infty, s )$. For a given architecture $( {N}_0, \dots, {N}_{\widetilde{L}} )$, the total number of weights, including zero and nonzero ones, satisfies
                $\sum_{\ell = 1}^{\widetilde{L}} ( N_{\ell} \, N_{\ell - 1} + N_\ell) \leq L (W^2 + W)$
                and there are at most
				\begin{equation*}
 					\sum_{i = 0}^{s} {L ( W^2 +W ) \choose i} \leq \sum_{i = 0}^{s} ( L ( W^2 +W ) )^{i}\leq ( L ( W^2 +W ) )^{s +1}
				\end{equation*} 
				different ways to choose the positions of the nonzero weights in the {network configuration}. Finally, given an architecture and the positions of the nonzero weights, there are at most $\nleft| \mathbb{A} \nright|^{s}$ different ways to choose these nonzero weights. It therefore follows that
				\begin{equation}
				\label{eq:counting_sparse_network}
					\nleft| \mathcal{N}_\mathbb{A} ( d,W,L,\infty, s ) \nright| \leq  L W^L  \cdot ( L ( W^2 +W ) )^{s +1} \cdot \nleft| \mathbb{A} \nright|^{s},
				\end{equation}
				and hence
				\begin{align}
					&\,\log ( \nleft| \mathcal{N}_\mathbb{A} ( d,W,L,\infty, s ) \nright| )\label{eq:counting_sparse_network_1}\\
					&\leq \, \log(L) + L \log(W) + ( s+1 ) \log ( L( W^2 + W) ) + s \log ( \nleft| \mathbb{A} \nright|) \label{eq:counting_sparse_network_2}\\
					&\leq \, s \log (L(W+1)) + ( s+1 ) \log ( L( W^2 + W) ) + s \log ( \nleft| \mathbb{A} \nright|)   \label{eq:counting_sparse_network_21}\\
					&\leq \, s \log (L(W+1)) + 4s \log ( L ( W+1 )) + s \log (\nleft| \mathbb{A} \nright| ) \label{eq:counting_sparse_network_3}\\
					&\leq\, 5s ( \log (L ( W+1 ))  + \log ( \nleft| \mathbb{A} \nright|  )).\label{eq:counting_sparse_network_4}
 				\end{align}
 				where \eqref{eq:counting_sparse_network_21} follows from $\log(L) + L \log(W) \leq s \log(L) + s \log(W) \leq s \log (L ( W+1 ))$ and in \eqref{eq:counting_sparse_network_3} we used  $( s+1 ) \leq 2s$ and $L( W^2 + W) \leq (L ( W+1 ))^2 $. The result then follows from \eqref{eq:counting_sparse_network_1}-\eqref{eq:counting_sparse_network_4} together with  $\nleft|  \mathcal{R}_\mathbb{A} ( d,W,L, \infty, s )  \nright| \leq \nleft|  \mathcal{N}_\mathbb{A} ( d,W,L,\infty, s )  \nright| $. 
			\end{proof}
		\end{lemma}

		We are now ready to prove the upper bound \eqref{eq:upper_bound_sparse}.  Fix $\varepsilon \in ( 0, 1\slash 2 )$,
  let
			\begin{equation}
				 b := \biggl \lceil \log \biggl( \frac{L (W+1)^L B^{L-1}}{\varepsilon} \biggr) \biggr\rceil,
			\end{equation}
			and note that $L (W+1)^L B^{L-1} 2^{-b} \leq \varepsilon$. It follows from Lemma~\ref{lem:define_covering_sparse} that $\mathcal{R}_{[-B,B] \cap 2^{-b} \mathbb{Z}}(d,W,L,\infty,s)$ is an $\varepsilon$-covering of $\mathcal{R}(d,W,L,B,s)$ with respect to the $L^p ( [0,1]^d )$-norm. By the minimality of the covering number, we have
			\begin{equation}
			\label{eq:using_the_covering_sparse}
				N ( \varepsilon, \mathcal{R}(d,W,L,B,s), L^p ( [0,1]^d ) )  \leq | \mathcal{R}_{[-B,B] \cap 2^{-b} \mathbb{Z}}(d,W,L,\infty,s) |.
			\end{equation}
			It then follows that
			\begin{align}
				&\,\log ( N(\varepsilon,\mathcal{R}(d,W,L,B,s),L^p ( [0,1]^d ) ) ) \label{eq:51proof_counting_1}\\
				&\leq\, \log ( \nleft| \mathcal{R}_{[-B,B] \cap 2^{-b} \mathbb{Z}}(d,W,L,\infty,s) \nright| )\label{eq:51proof_counting_2}\\
				&\leq \, 5 s ( \log ( L ( W+1 ) ) + \log( \nleft|[-B,B] \cap 2^{-b} \mathbb{Z}\nright|  ) ) \label{eq:51proof_counting_3}\\
				&\leq\, 5s \biggl( \log(L ( W+1 )) +  3 \log \biggl( \frac{L (W+1)^L B^{L}}{\varepsilon} \biggr) \biggr) \label{eq:51proof_counting_4}\\
				&\leq \, 5 s \biggl(2\log \biggl( \frac{(W+1)^L B^{L}}{\varepsilon}  \biggr) + 6\log \biggl(  \frac{(W+1)^L B^{L}}{\varepsilon}  \biggr) \biggr)\label{eq:51proof_counting_5}\\
				&= \, 40 s \log \biggl( \frac{(W+1)^L B^{L}}{\varepsilon} \biggr)\label{eq:51proof_counting_6} \\
				&= \, 40 \min \{ s , 2W^2 L \} \log \biggl( \frac{(W+1)^L B^{L}}{\varepsilon} \biggr)\\
				&\leq \, 80 \min \{ s , W^2 L \} \log \biggl( \frac{(W+1)^L B^{L}}{\varepsilon} \biggr). \label{eq:51proof_counting_8}
			\end{align}
			where \eqref{eq:51proof_counting_3} is by Lemma~\ref{lem:counting_cardinality_sparse}, \eqref{eq:51proof_counting_4} follows from 
   \eqref{eq:ccw_1}-\eqref{eq:ccw_6}, and in \eqref{eq:51proof_counting_5} we used
   $L ( W+1 ) \leq ( W+1 )^{L}  ( W+1 )  \leq \frac{( W+1 )^{2L} B^{2L}}{\varepsilon^{2}}$ and $\frac{L (W+1)^L B^{L}}{\varepsilon}  \leq \frac{( W+1 )^{2L}  B^{2L}}{\varepsilon^2}$. This concludes the proof upon taking $C=80$.

	\subsection{Proof of the Lower Bound \eqref{eq:lower_bound_sparse}}
	\label{sub:proof_of_proposition_proposition:covering_number_lower_bound_sparse}
   Set $D:= 60^2 \cdot 6$,
define the auxiliary variable
			\begin{equation}
				\overline{W} = \biggl \lfloor \sqrt{\frac{s}{6L}} \biggr\rfloor,
			\end{equation}
			and note that $\overline{W} \leq W$ as a consequence of $\sqrt{\frac{s}{6L} } < \sqrt{\frac{2W^2 L }{6L}} < W$. Further, we have $\overline{W} = \bigl\lfloor \sqrt{\frac{s}{6L} }\bigr\rfloor \geq \biggl\lfloor \sqrt{\frac{D d^2 L}{6L} }\biggr\rfloor= \biggl\lfloor \sqrt{\frac{60^2 \cdot 6 \cdot d^2 L}{6L} }\biggr\rfloor = 60d \geq \max \{ 60,d \}$.  It now follows from $\overline{W} \leq W$ that
			\begin{equation}
			\label{eq:D2_reduce_W}
				R ( d,\overline{W},L, B, s)  \subseteq R ( d,W,L, B, s).
			\end{equation}
			As $2 \overline{W}^2L = 2  \lfloor \sqrt{\frac{s}{6L}} \rfloor^2 L  \leq 2 \cdot \frac{s}{6L}  \cdot L < s$,  we can conclude that
			\begin{equation}
			\label{eq:remove_redundancy_tilde}
				R ( d,\overline{W},L, B, s) = R ( d,\overline{W},L, B).
			\end{equation} 
			Combining \eqref{eq:D2_reduce_W} and \eqref{eq:remove_redundancy_tilde}, we obtain the inclusion
			\begin{equation*}
				R ( d,\overline{W},L, B) \subseteq R ( d,W,L, B, s),
			\end{equation*}
			which, thanks to \eqref{eq:covering_packing_inclusion_2} in Lemma~\ref{lem:covering_packing_inclusion}, yields
			\begin{equation}
			\label{eq:52_nondegenerate_overall}
				 N(\varepsilon,\mathcal{R}(d,W,L,B,s),L^p ( [0,1]^d ) )  \geq  N(2 \varepsilon,\mathcal{R}(d,\overline{W},L,B),L^p ( [0,1]^d ) ).
			\end{equation}
   Application of \eqref{eq:lower_bound_fully_connected_bounded_output_main} in Theorem~\ref{thm:covering_number_upper_bound_fully_connected_bounded_weight} with $W$ replaced by $\overline{W}$, $\varepsilon$ replaced by $2 \varepsilon$, and the prerequisites satisfied owing to $\overline{W} \geq \max \{ 60,d \}$,  $L \geq 60$, and $2\varepsilon \in (0, 1 \slash 2)$, yields a lower bound on the right-hand-side of \eqref{eq:52_nondegenerate_overall} according to
			\begin{equation}
			\label{eq:52_nondegenerate_0}
				\log ( N(2 \varepsilon,\mathcal{R}(d,\overline{W},L,B),{L^p ( [0,1]^d )} ) ) \geq c_1 \overline{W}^2 L\log \biggl( \frac{( \overline{W} + 1 )^L B^{L}}{2 \varepsilon} \biggr),
			\end{equation}
			with $c_1 \in \mathbb{R}_{+}$ an absolute constant.
			We can now further lower-bound the right-hand-side of \eqref{eq:52_nondegenerate_0} according to
			\begin{align}
    & c_1 \overline{W}^2 L\log \biggl( \frac{( \overline{W} + 1 )^L B^{L}}{2 \varepsilon} \biggr)  \label{eq:52_nondegenerate_1}\\
				& \geq \,\frac{c_1}{24} s \log \biggl( \frac{(\overline{W} + 1 )^L B^{L}}{2 \varepsilon} \biggr) \label{eq:52_nondegenerate_2}\\
				& > \, \frac{c_1}{48} s \log \biggl( \frac{(\overline{W} + 1 )^L B^{L}}{\varepsilon} \biggr)\label{eq:52_nondegenerate_3} \\
    	           & = \, \frac{c_1}{48} \min \{s, 2W^2 L\} \log \biggl( \frac{(\overline{W} + 1 )^L B^{L}}{\varepsilon} \biggr)\label{eq:52_nondegenerate_31} \\
				& \geq \, \frac{c_1}{ 48} \min \{ s, W^2 L \} \log \biggl( \frac{(\overline{W} + 1 )^L B^{L}}{\varepsilon} \biggr), \label{eq:52_nondegenerate_4}
			\end{align}
			where  \eqref{eq:52_nondegenerate_2} follows from  $\overline{W}^2L = \lfloor \sqrt{\frac{s}{6L} }\rfloor^2 L \geq (\frac{1}{2} \sqrt{\frac{s}{6L}}  )^2 L \geq \frac{1}{24} s$, as $\lfloor x \rfloor \geq \frac{1}{2} x$, for $x \geq 1$, and \eqref{eq:52_nondegenerate_3} is by $\log ( \frac{( \overline{W} + 1 )^{L} B^{L}}{ 2\varepsilon} ) = \frac{1}{2} \log ( \frac{( \overline{W} + 1 )^{2L} B^{2L}}{4\varepsilon^2} ) > \frac{1}{2} \log ( \frac{( \overline{W} + 1 )^{L} B^{L}}{\varepsilon} )$, since $\varepsilon \in (0, \frac{1}{4})$.
   To replace $\overline{W}$ in \eqref{eq:52_nondegenerate_4} by $\widetilde{W} = \min \{ \lceil \sqrt{\frac{s}{L} }\,\rceil, W \} $, we note that 
			\begin{align}
				(\overline{W} + 1 )^2 = &\, (\overline{W} + 1 ) \cdot \biggl(\biggl \lfloor \sqrt{\frac{s}{6L}} \biggr\rfloor + 1 \biggr)  \label{eq:asdfu_0}  \\
				\geq &\, 61 \cdot \biggl \lceil \sqrt{\frac{s}{6L}} \biggr\rceil \label{eq:asdfu}\\
                    > &\, 10 \sqrt{\frac{s}{L}} \\
				> &\, \biggl\lceil \sqrt{\frac{s}{L} }\,\biggr\rceil + 1 \\
				\geq &\, \widetilde{W} + 1,  \label{eq:asdfu_1}
			\end{align}
			where in \eqref{eq:asdfu} we used $\overline{W} \geq 60$. Using \eqref{eq:asdfu_0}-\eqref{eq:asdfu_1}, we get
   			\begin{equation}
			\label{eq:remove_overline_W}
				\log \biggl( \frac{(\overline{W} + 1 )^L B^{L}}{\varepsilon} \biggr) = \frac{1}{2} \log \biggl( \frac{(\overline{W} + 1 )^{2L} B^{2L}}{\varepsilon^2} \biggr) \geq \frac{1}{2} \log \biggl( \frac{(\widetilde{W} + 1 )^{L} B^{L}}{\varepsilon} \biggr).
			\end{equation}
			Putting \eqref{eq:52_nondegenerate_overall}, \eqref{eq:52_nondegenerate_1}-\eqref{eq:52_nondegenerate_4}, and \eqref{eq:remove_overline_W} together, we obtain 
			\begin{equation}
			\label{eq:52_nondegenerate_final}
				\log ( N(\varepsilon,\mathcal{R}(d,W,L,B,s),{L^p ( [0,1]^d )} ))  \geq \frac{c_1}{96} \min \{ s, W^2 L \} \log \biggl( \frac{( \widetilde{W} +1 )^L B^{L}}{\varepsilon} \biggr),
			\end{equation}
which concludes the proof upon setting $c =\frac{c_1}{96}$.

\section{Proof of the Lower Bound \eqref{eq:44_results_quantized} in Theorem~\ref{thm:networks_quantized_weights_covering_number_bound}}
\label{sec:proof_of_proposition_prop:covering_number_lower_bound_quantized_weights}	
	The proof of the lower bound \eqref{eq:44_results_quantized} relies on several technical ingredients, which we present first.  
 We start with a result that allows to reduce the general case $a \in \mathbb{N}$ to $a=1$.

	\begin{lemma}
	\label{lem:equivalence_covering_number}
		Let $d,W,L,a,b \in \mathbb{N}$, with $W\geq 2$. We have
		\begin{align}
			\mathcal{R}_b^a ( d, W,L) =& \,  2^{(a - 1)L}\cdot \mathcal{R}_{a+b-1}^1 ( d, W,L) \label{eq:product_relation_quantized_weights_1} \\
			=&\, \{2^{(a - 1)L} \cdot f: f\in \mathcal{R}_{a+b-1}^1 ( d, W,L) \},		\label{eq:product_relation_quantized_weights_2}
		\end{align}
		and, for all $\varepsilon  \in \mathbb{R}_+$, $ p \in [1,\infty]$,
		\begin{equation}
		\label{eq:equivalent_different_precision}
			N ( \varepsilon, \mathcal{R}_b^a ( d, W,L), {L^p ( [0,1]^d )}  ) = N \biggl( \frac{\varepsilon}{2^{( a-1 )L}}, \mathcal{R}_{a+b-1}^1 ( d, W,L),{L^p ( [0,1]^d )} \biggl).
		\end{equation}

		\begin{proof}
			We start by establishing \eqref{eq:product_relation_quantized_weights_1}. Arbitrarily fix $g \in \mathcal{R}_b^a ( d, W,L)$. From Lemma~\ref{lem:extension}, with its prerequisites satisfied thanks to $\{ -1,0,1 \} \subseteq \mathbb{Q}_{b}^a $ and $\mathbb{Q}_{b}^a = - \mathbb{Q}_{b}^a$,  we can infer the existence of a network configuration $\Phi = ( A_\ell, b_\ell )_{\ell = 1}^L$, with $\mathcal{W} ( \Phi ) \leq W$ and $\coef ( \Phi ) \in \mathbb{Q}_b^a$, such that $R ( \Phi ) = g$. Define $\widetilde{\Phi} = ( \frac{1}{2^{a -1}} A_\ell, \frac{1}{2^{a -1}}  b_\ell )_{\ell = 1}^L$, and note that $\mathcal{W} ( \widetilde{\Phi} ) \leq W$, $\mathcal{L} ( \widetilde{\Phi} ) = L$, and $\coef ( \widetilde{\Phi} ) \subseteq \frac{1}{2^{a-1}} \mathbb{Q}_b^a = \mathbb{Q}^1_{b+a -1} $, and therefore $R( \widetilde{\Phi} ) \in \mathcal{R}_{a+b-1}^1 ( d,W,L)$. We then have 
			\begin{align}
				g =&\, R( \Phi ) \\
				=&\, \affine (  A_L,   b_L  )\circ\rho \circ \cdots \circ \rho \circ \affine (  A_1,   b_1  ) \label{eq:c5_reduce_1}\\
				=&\, ( 2^{(a-1)} )^L \affine \biggl( \frac{1}{2^{a -1}} A_L, \frac{1}{2^{a -1}}  b_L  \biggr)\circ\rho \circ \cdots \circ \rho \circ \affine \biggl( \frac{1}{2^{a -1}} A_1, \frac{1}{2^{a -1}}  b_1  \biggr) \label{eq:apply_phg_1}\\
				=&\, 2^{(a-1)L} \cdot R ( \widetilde{\Phi} )\\
				\in &\, \{2^{(a - 1)L} \cdot f: f\in\mathcal{R}_{a+b-1}^1 ( d, W,L) \}, 
			\end{align}
			where in \eqref{eq:apply_phg_1} we used the positive homogeneity of the ReLU function, namely, $\rho(kx) =  k \rho(x)$ for all $x \in \mathbb{R}$ and $k \in \mathbb{R}_+$. As the choice of $g \in \mathcal{R}_b^a ( d, W,L)$ was arbitrary, we have established that 
			\begin{equation}
				\mathcal{R}_b^a ( d, W,L) \subseteq 2^{(a - 1)L} \cdot \mathcal{R}_{a+b-1}^1 ( d, W,L). \label{eq:c5_reduce_main_1}
			\end{equation}
			Upon noting that the reverse inclusion can be proved similarly,
   \eqref{eq:product_relation_quantized_weights_1} follows.

			We proceed to prove \eqref{eq:equivalent_different_precision}. Fix $\varepsilon \in \mathbb{R}_+$ and $p \in [1,\infty]$. It follows from \eqref{eq:product_relation_quantized_weights_1} that, for every $\varepsilon$-covering $\mathcal{C}$ of $\mathcal{R}_b^a ( d, W,L)$ with respect to the $L^p ( [0,1]^d )$-norm, $\frac{1}{2^{(a - 1)L}} \cdot \mathcal{C}$ is an $\frac{\varepsilon}{2^{(a - 1)L}}$-covering of $R_{a+b-1}^1 ( d, W,L)$, which allows us to conclude that
   			\begin{equation}
			\label{eq:E1_reverse_1}
			 	N ( \varepsilon, \mathcal{R}_b^a ( d, W,L), {L^p ( [0,1]^d )}  ) \geq N \biggl( \frac{\varepsilon}{2^{( a-1 )L}}, \mathcal{R}_{a+b-1}^1 ( d, W,L), {L^p ( [0,1]^d )} \biggr).
			\end{equation} 
			Moreover, based on \eqref{eq:product_relation_quantized_weights_1}, we can also conclude that, for every $\frac{\varepsilon}{2^{(a - 1)L}}$-covering $\mathcal{C}$ of $\mathcal{R}_{a+b-1}^1 ( d, W,L)$ with respect to the $L^p ( [0,1]^d )$-norm, $2^{( a -1)L} \cdot \mathcal{C}$ is an  $\varepsilon$-covering of $\mathcal{R}_b^a ( d, W,L)$, which leads to
   			\begin{equation}
			\label{eq:E1_reverse_2}
				N \biggl( \frac{\varepsilon}{2^{( a-1 )L}}, \mathcal{R}_{a+b-1}^1 ( d, W,L), {L^p ( [0,1]^d )} \biggr) \geq N ( \varepsilon, \mathcal{R}_b^a ( d, W,L), {L^p ( [0,1]^d )}  ).
			\end{equation}
			Combining \eqref{eq:E1_reverse_1} and \eqref{eq:E1_reverse_2}, yields \eqref{eq:equivalent_different_precision}.
		\end{proof}
	\end{lemma}
 We note that the constraint $W \geq 2$ in Lemma~\ref{lem:equivalence_covering_number} is not restrictive for the purposes of the proof of the lower bound in Theorem~\ref{thm:networks_quantized_weights_covering_number_bound} as below we will take the absolute constant $D$ in Theorem~\ref{thm:networks_quantized_weights_covering_number_bound} to be greater than $2$. 

Next, we derive	lower bounds on the covering number, separately, for large and moderate values of $b$. We start with the case of large $b$.

	\begin{lemma} 
		\label{lem:quantized_case_1}
		Let $p \in [1,\infty]$ and $d,W,L,b \in \mathbb{N}$, with $W,L \geq 60$. Assume that $b > \log (L) + L \log ( W+1 ) + 3$.  Then, there exist absolute constants $c_1,c_2, c_3 \in \mathbb{R}_+$ such that the following statements hold:
		\begin{itemize}
			\item For all $\varepsilon  \in (0, L (W+1)^L 2^{-b}]$, 
			\begin{equation}
			\label{eq:quantized_case_10}
				\log ( N (\varepsilon,  \mathcal{R}^1_b ( d,W,L ),L^p ( [0,1]^d )  ) ) \geq c_1 W^2 L b.
			\end{equation}
			\item If $L (W+1)^L 2^{-b} < \frac{1}{32}$, then, for all $\varepsilon \in (L (W+1)^L 2^{-b}, \frac{1}{32}) $, 
			\begin{equation}
			\label{eq:quantized_case_11}
				\log ( N (\varepsilon,  \mathcal{R}^1_b ( d,W,L ),L^p ( [0,1]^d )  ) ) \geq c_2 W^2 L \log \biggl(\frac{(W+1)^L}{\varepsilon}\biggr).
			\end{equation}
			\item For all $\varepsilon \in ( 0, \frac{1}{32} )$,
			\begin{equation}
			\label{eq:quantized_case_12}
				\log (N(\varepsilon,\mathcal{R}^1_b ( d,W,L ),L^p ( [0,1]^d ) )) \geq  
					c_3 W^2 L  \cdot \min \biggl\{b,\log \biggl(\frac{(W+1)^L}{\varepsilon}\biggr) \biggr\}.
			\end{equation}
		\end{itemize}
	\end{lemma}
	\begin{proof}
		We start with the proof of \eqref{eq:quantized_case_10}. 
		It follows from Lemma~\ref{lem:define_covering_fully_connected} with $B = 1$ that 
		\begin{equation}
		\label{eq:44error}
			\mathcal{A} ( \mathcal{R}(d,W,L,1), \mathcal{R}_{[-1,1] \cap 2^{-b} \mathbb{Z}}(d,W,L), \nleft\| \cdot \nright\|_{L^p ( [0,1]^d )} ) \leq L (W+1)^L 2^{-b}.
		\end{equation}
		As $( [-1,1] \cap 2^{-b} \mathbb{Z} ) \subseteq ((-4,4) \cap 2^{-b} \mathbb{Z}  )  = \mathbb{Q}_b^1$, we have $\mathcal{R}_{[-1,1]\cap 2^{-b} \mathbb{Z}} ( d,W,L ) \subseteq \mathcal{R}_b^1 (  d,W,L )$ which together with \eqref{eq:44error} yields 
		\begin{equation}
		\label{eq:lalala}
		\begin{aligned}
			&\,\mathcal{A} (\mathcal{R} ( d,W,L,1 ),  \mathcal{R}^1_b ( d,W,L ),\nleft\| \cdot \nright\|_{L^p ( [0,1]^d )} )\\
			&\leq\, \mathcal{A} ( \mathcal{R}(d,W,L,1), \mathcal{R}_{[-1,1] \cap 2^{-b} \mathbb{Z}}(d,W,L), \nleft\| \cdot \nright\|_{L^p ( [0,1]^d )} )\\
			&\leq\, L (W+1)^L 2^{-b}.
		\end{aligned}
		\end{equation}
		Application of Proposition~\ref{lem:cardinality_approximation_class} with $\delta = \|\cdot\|_{L^p ( [0,1]^d )} $, $\mathcal{G} = \mathcal{R} ( d,W,L,1 )$, $\mathcal{F} = \mathcal{R}^1_b ( d,W,L )$,  $\varepsilon = L (W+1)^L 2^{-b}$, and the prerequisite \eqref{eq:improper_cover} satisfied thanks to \eqref{eq:lalala},   yields 
		\begin{equation}
		\label{eq:44apply21}
		\begin{aligned}
			&N (L (W+1)^L 2^{-b},  \mathcal{R}^1_b ( d,W,L ), {L^p ( [0,1]^d )}   )\\
			&\geq N (4L (W+1)^L 2^{-b},  \mathcal{R} ( d,W,L,1), {L^p ( [0,1]^d )}   ).
		\end{aligned}
		\end{equation}
		The right-hand-side of \eqref{eq:44apply21} can be lower-bounded by application of \eqref{eq:lower_bound_fully_connected_bounded_output_main} in Theorem~\ref{thm:covering_number_upper_bound_fully_connected_bounded_weight} with $B = 1$, upon noting that the prerequisites are satisfied thanks to $W,L \geq 60$ by assumption, and $4L (W+1)^L 2^{-b} < 4L (W+1)^L 2^{- ( \log(L) + L \log(W+1) +3 )} \leq  1\slash 2$. Specifically, we obtain
  		\begin{equation}
		\label{eq:44apply21_200}
		\begin{aligned}
			\log ( N (4L (W+1)^L 2^{-b},  \mathcal{R} ( d,W,L, 1),L^p ( [0,1]^d ) ) ) \geq&\, c_4 W^2 L \log \biggl( \frac{(W+1)^L }{4L (W+1)^L  2^{-b}} \biggr)\\
			=&\, c_4 W^2 L \log \biggl( \frac{1}{L  2^{-b + 2}} \biggr),
		\end{aligned}
 \end{equation}
    with $c_4 \in \mathbb{R}_{+}$ an absolute constant.
		We note that  
		\begin{align}
			\log \biggl( \frac{1}{L  2^{-b + 2}} \biggr)  = &\, \frac{b}{4} + \biggl( \frac{3b}{4} - 2  - \log (L) \biggr) \label{eq:easdfe_1} \\
			 >  &\, \frac{b}{4} + \biggl( \frac{3}{4} ( \log (L) + L \log ( W+1 ) + 3 ) - 2  - \log (L) \biggr) \label{eq:easdfe_2}\\
			 > &\, \frac{b}{4} + \biggl( \frac{3}{4} ( 2 \log (L) + 3 ) - 2  - \log (L) \biggr) \label{eq:easdfe_3}\\
			 > &\, \frac{b}{4}, \label{eq:easdfe_4}
		\end{align}
		where \eqref{eq:easdfe_2} follows from the assumption $b > \log (L) + L \log ( W+1 ) + 3$, and in \eqref{eq:easdfe_3} we used $L \log (W + 1) \geq L > \log(L)$, for $L \geq 60$.  	Putting \eqref{eq:44apply21}, \eqref{eq:44apply21_200}, and \eqref{eq:easdfe_1}-\eqref{eq:easdfe_4} together yields 
		\begin{equation*}
			\log ( N (L (W+1)^L 2^{-b},  \mathcal{R}^1_b ( d,W,L ),L^p ( [0,1]^d )  ) ) \geq \frac{c_4}{4} W^2 L b.
		\end{equation*}
		As the covering number is a non-decreasing function of the covering ball radius $\varepsilon$, we have 
		\begin{equation}
			\label{eq:44_result_1}
			\log ( N (\varepsilon,  \mathcal{R}^1_b ( d,W,L ),L^p ( [0,1]^d )  ) ) \geq \frac{c_4}{4} W^2 L b, \quad \text{for all } \varepsilon \in ( 0 , L (W+1)^L 2^{-b}].
		\end{equation}
		Upon setting $c_1 = \frac{c_4}{4}$, this finalizes the proof of \eqref{eq:quantized_case_10}.

		We proceed to the proof of \eqref{eq:quantized_case_11}. Arbitrarily fix 
  $\varepsilon \in (L (W+1)^L 2^{-b}, \frac{1}{32})$ and let
		\begin{equation}
		\label{eq:45_proof_0}
			\widetilde{b} := \biggl \lfloor  \log \biggl( \frac{L (W+1)^L}{2\varepsilon} \biggr) \biggr \rfloor.
		\end{equation}
		We first note that 
		\begin{equation}
			\widetilde{b} \leq  \biggl \lfloor  \log \biggl( \frac{L (W+1)^L}{2L (W+1)^L 2^{-b}} \biggr) \biggr \rfloor =  \lfloor  \log (  2^{b-1} )  \rfloor \leq b,
		\end{equation}
		which leads to the inclusion $\mathcal{R}^1_{\widetilde{b}} ( d,W,L ) \subseteq  \mathcal{R}^1_{b} ( d,W,L ) $. Thanks to Lemma~\ref{lem:covering_packing_inclusion}, we hence get
		\begin{equation}
		\label{eq:45_proof_100}
			\log ( N (\varepsilon,  \mathcal{R}^1_{b} ( d,W,L ),L^p ( [0,1]^d )  ) ) \geq \log ( N (2 \varepsilon,  \mathcal{R}^1_{\widetilde{b}} ( d,W,L),L^p ( [0,1]^d )   ) ). 
		\end{equation}
		To lower-bound the right-hand-side of \eqref{eq:45_proof_100},  we apply \eqref{eq:44_result_1} with $b$ replaced by $\widetilde{b}$, $\varepsilon$ replaced by $2\varepsilon$, the prerequisite $\widetilde{b} > \log (L) + L \log ( W+1 ) + 3$  satisfied owing to 
		\begin{align}
			\widetilde{b} >&\, \log (L) + L \log ( W+1 ) + \log \biggl( \frac{1}{2 \varepsilon} \biggr) - 1 \\
			> &\,\log (L) + L \log ( W+1 ) + \log \biggl( \frac{1}{2 \cdot \frac{1}{32}} \biggr) - 1\\
			=&\, \log (L) + L \log ( W+1 ) + 3,
		\end{align}
		and the prerequisite $2 \varepsilon	 \in (0, L (W+1)^L 2^{-\widetilde{b}}]$ satisfied thanks to  
		\begin{equation}
		\label{eq:44prerequisite_0}
			2\varepsilon = L (W+1)^L 2^{-\log ( \frac{L (W+1)^L}{2\varepsilon} )} \leq L (W+1)^L 2^{-\widetilde{b}},
		\end{equation}
		to obtain 
		\begin{equation}
		\label{eq:E2_proof_21111}
			\log ( N (2\varepsilon,  \mathcal{R}^1_{\widetilde{b}} ( d,W,L ),L^p ( [0,1]^d )   ) ) \geq \frac{c_4}{4} W^2 L \widetilde{b}.
		\end{equation}
		We further note that 
		\begin{align}
			\widetilde{b} =&\, \biggl \lfloor  \log \biggl( \frac{L (W+1)^L}{2\varepsilon} \biggr) \biggr \rfloor \label{eq:E2_proof_210}\\
			\geq&\,  \log \biggl( \frac{L (W+1)^L}{2\varepsilon} \biggr) -1\\
			\geq&\, \frac{1}{2}\log \biggl( \frac{L (W+1)^L}{2\varepsilon} \biggr), \label{eq:E2_proof_21}\\
            > &\, \frac{1}{2}\log \biggl( \frac{(W+1)^L}{\varepsilon} \biggr), \label{eq:E2_proof_1111}
		\end{align}
		where in \eqref{eq:E2_proof_21} we used $\log ( \frac{L (W+1)^L}{2\varepsilon} ) \geq \log ( \frac{60 (61)^{61}}{2\frac{1}{32}} ) > 2$ as $W,L \geq 60$ and $\varepsilon < \frac{1}{32}$, both by assumption, and \eqref{eq:E2_proof_1111} follows from $L \geq 60$. Using the lower bound \eqref{eq:E2_proof_210}-\eqref{eq:E2_proof_1111} in \eqref{eq:E2_proof_21111}, results in $\log ( N (2\varepsilon,  \mathcal{R}^1_{\widetilde{b}} ( d,W,L ),L^p ( [0,1]^d )   ) ) \geq \frac{c_4}{8} W^2 L \log ( \frac{ (W+1)^L}{\varepsilon} )$, which together with \eqref{eq:45_proof_100} yields
		\begin{equation}
		\label{eq:E2_proof_case2}
			\log ( N (\varepsilon,  \mathcal{R}^1_{b} ( d,W,L ),L^p ( [0,1]^d )  ) ) \geq \frac{c_4}{8} W^2 L \log \biggl( \frac{(W+1)^L}{\varepsilon} \biggr).
		\end{equation}
		As $\varepsilon$ was chosen arbitrarily in $ (L (W+1)^L 2^{-b}, \frac{1}{32})$, \eqref{eq:E2_proof_case2} concludes the proof of \eqref{eq:quantized_case_11} upon setting $c_2 = \frac{c_4}{8}$.

		It remains to establish $\eqref{eq:quantized_case_12}$. To this end, we first consider the case
  $L (W+1)^L 2^{-b} \geq \frac{1}{32} $. It follows from \eqref{eq:quantized_case_10} that
		\begin{equation*}
			\log ( N (\varepsilon,  \mathcal{R}^1_b ( d,W,L ),L^p ( [0,1]^d )  ) ) \geq c_1 W^2 L b, \quad \text{for } \varepsilon \in ( 0,1/32 ] \subseteq ( 0,  L (W+1)^L 2^{-b}].
		\end{equation*}
		For $L (W+1)^L 2^{-b} < \frac{1}{32}  $, we obtain from \eqref{eq:quantized_case_10} and \eqref{eq:quantized_case_11} that 
            \begin{equation*}
			\log ( N (\varepsilon,  \mathcal{R}^1_b ( d,W,L ),L^p ( [0,1]^d )  ) ) \geq \left \{
                \begin{aligned}
                    &\,c_1 W^2 L b,&& \text{for } \varepsilon \in (0, L (W+1)^L 2^{-b}],\\
				&\,c_2 W^2 L \log \Bigl(\frac{(W+1)^L}{\varepsilon}\Bigr),&& \text{for } \varepsilon \in \Bigl(L (W+1)^L 2^{-b}, \frac{1}{32}\Bigr).
                \end{aligned}
                \right.
		\end{equation*}
		 Combining these results, we obtain, for $\varepsilon \in ( 0,\frac{1}{32} )$,
		\begin{align*}
			\log ( N (\varepsilon,  \mathcal{R}^1_b ( d,W,L ),L^p ( [0,1]^d )  ) ) \geq&\, \min \biggl\{ c_1 W^2 L b, c_2 W^2 L \log \biggl(\frac{(W+1)^L}{\varepsilon}\biggr) \biggr\} \\
			\geq& \,  \min \{ c_1,c_2 \}W^2 L  \cdot \min \biggl\{b,  \log \biggl(\frac{(W+1)^L}{\varepsilon}\biggr) \biggr\},
		\end{align*}
		which, upon setting $c_3 = \min \{ c_1,c_2\}$, concludes the proof.
	\end{proof}

	We continue with the case of moderate $b$. A key ingredient here is the depth-precision tradeoff developed in \cite{FirstDraft2022}, formalized
 as follows.

	\begin{lemma}
	\label{prop:tradeoff}
		Let $d,W,L,k\in \mathbb{N}$. For all $a,b \in \mathbb{N}$, it holds that 
		\begin{equation}
		\label{eq:tradeoff}
			\mathcal{R}^{ka}_{kb} \left(d,  W,  L   \right) \subseteq  \mathcal{R}^a_b ( d, 16W,(k+2) L ).
		\end{equation}
		\begin{proof}
			The case $d = 1$ is \cite[Proposition 4.1]{FirstDraft2022}. The proof for general $d$ is structurally identical to that of \cite[Proposition 4.1]{FirstDraft2022} and is provided, for completeness, in Appendix~\ref{sub:proof_of_lemma_prop:tradeoff}. 
		\end{proof}

	\end{lemma}
	The lower bound on the covering number of $\mathcal{R}^1_b ( d,W,L )$, for moderate values of $b$, is as follows.
	\begin{lemma}
		\label{lem:quantized_case_2}
		Let $p \in [1,\infty]$,  $d,W,L,b \in \mathbb{N}$, with $W, L \geq 960$.  Assume that $ 72000 \, \frac{\log (W)}{L} < b \leq 4 L \log (W) $. Then, we have 
		\begin{equation}
			\log (N(\varepsilon,\mathcal{R}^1_b ( d,W,L ),L^p ( [0,1]^d ) )) \geq c\, W^2 L b, \quad \text{ for all } \varepsilon \in \biggl( 0,\frac{1}{100}\biggr),
		\end{equation}
		with $c \in \mathbb{R}_+$ an absolute constant. 
	\end{lemma}

	\begin{proof}
		Let $k \in \bigl\{ 2,\dots, \lfloor\frac{L}{60} \rfloor - 2  \bigr\}$ be an integer to be determined later. Application of Lemma~\ref{prop:tradeoff} with $a = 1$, $W$ replaced by $\lfloor W\slash 16 \rfloor$, and $L$ replaced by $\lfloor L\slash ( k+2 )\rfloor$, yields
		\begin{equation}
		\label{eq:42_apply_tradeoff}
			\mathcal{R}^{k}_{k b} (d,  \lfloor W \slash 16 \rfloor,  \lfloor L \slash (k+2) \rfloor  ) \subseteq  \mathcal{R}^{1}_{b}( d, 16\lfloor W \slash 16 \rfloor,(k+2) \lfloor L \slash (k+2) \rfloor ).
		\end{equation}
		As $k \geq 2$, we have the inclusion relation 
  		\begin{equation}
		\label{eq:42_apply_tradeoff_apply_1}
			\mathcal{R}^{1}_{k b} (d,  \lfloor W \slash 16 \rfloor,  \lfloor L \slash (k+2) \rfloor  ) \subseteq \mathcal{R}^{k}_{k b} (d,  \lfloor W \slash 16 \rfloor,  \lfloor L \slash (k+2) \rfloor  ).
		\end{equation}
		Noting that $16\lfloor W \slash 16 \rfloor \leq W$ and $(k+2) \lfloor L \slash (k+2) \rfloor \leq L$, it follows that
		\begin{equation}
		\label{eq:42_apply_tradeoff_apply_2}
			\mathcal{R}^{1}_{b}( d, 16\lfloor W \slash 16 \rfloor,(k+2) \lfloor L \slash (k+2) \rfloor ) \subseteq \mathcal{R}^{1}_{b}( d, W, L ).
		\end{equation}
		Combining \eqref{eq:42_apply_tradeoff_apply_1} and \eqref{eq:42_apply_tradeoff_apply_2} with \eqref{eq:42_apply_tradeoff}, we obtain 
		\begin{equation}
		\label{eq:embed_kb_b}
			\mathcal{R}^{1}_{k b} (d,  \lfloor W \slash 16 \rfloor,  \lfloor L \slash (k+2) \rfloor  ) \subseteq \mathcal{R}^{1}_{b}( d, W, L ).
		\end{equation}
		Application of \eqref{eq:covering_packing_inclusion_2} in Lemma~\ref{lem:covering_packing_inclusion} to the inclusion relation \eqref{eq:embed_kb_b} then yields, for all $\varepsilon \in \mathbb{R}_+$,
		\begin{equation}
		\label{eq:E4_substitute_lower}
			N ( \varepsilon,\mathcal{R}^{1}_{b}( d, W, L ), L^p ( [0,1]^d )  ) \geq N ( 2 \varepsilon, \mathcal{R}^{1}_{k b} (d,  \lfloor W \slash 16 \rfloor,  \lfloor L \slash (k+2) \rfloor  ),  L^p ( [0,1]^d )).
		\end{equation}
		We next lower-bound the covering number of $\mathcal{R}^{1}_{k b} (d,  \lfloor W \slash 16 \rfloor,  \lfloor L \slash (k+2) \rfloor  )$ to get a lower bound on the covering number of $\mathcal{R}^{1}_{b}( d, W, L )$.  To this end, we identify a value of $k \in \bigl\{ 2,\dots, \lfloor\frac{L}{60} \rfloor - 2  \bigr\}$ that allows us to apply Lemma~\ref{lem:quantized_case_1} to $\mathcal{R}^{1}_{k b} (d,  \lfloor W \slash 16 \rfloor,  \lfloor L \slash (k+2) \rfloor  )$. This is done as follows. Define $f: \mathbb{R}_+ \rightarrow \mathbb{R}$ as $f(x) =  xb -  5 \frac{L}{x} \log (W) $, and note that $f$ is strictly increasing on its domain and satisfies
  		\begin{align}
			f(1) =&\, b -  5  L \log (W)  < 0, \label{eq:47proof_1} \\
			f\biggl(\biggl \lfloor\frac{L}{60} \biggr\rfloor  - 2\biggr) > &\,
			f \biggl( \frac{L}{120} \biggr) \label{eq:47proof_4}\\
			= &\, \frac{L}{120} b -  5 \frac{L}{\frac{L}{120}} \log (W)  \\
			=& \, \frac{1}{120} ( Lb - 72000 \log(W) )\\
			> &\, 0.\label{eq:47proof_7} 
		\end{align}
		Here, in \eqref{eq:47proof_1} we used the assumption $b \leq 4 L \log (W)$, \eqref{eq:47proof_4} follows from $\lfloor\frac{L}{60} \rfloor  - 2 > \frac{L}{60}   - 3 \geq \frac{L}{60}  - \frac{1}{320} L > \frac{L}{120}$, as $L \geq 960$, and in \eqref{eq:47proof_7} we invoked the assumption $b > 72000 \, \frac{\log(W)}{L} $.  We now choose $k \in \bigl\{ 2,\dots, \lfloor\frac{L}{60} \rfloor  - 2\bigr\}$ to be the unique integer such that $f(k-1) < 0$ and $f(k) \geq 0$, namely,
		\begin{align}
			(k-1)b -  5 \frac{L}{k-1} \log (W)  <&\, 0 \label{eq:E4_f_1} \\
			kb -  5  \frac{L}{k} \log(W) \geq&\, 0. \label{eq:E4_f_2}
		\end{align}
		For this $k$, we now lower-bound the covering number of $\mathcal{R}^{1}_{k b} (d,  \lfloor W \slash 16 \rfloor,  \lfloor L \slash (k+2) \rfloor  )$ by applying Lemma~\ref{lem:quantized_case_1} with $b$ replaced by $kb$, $W$ replaced by $\lfloor W \slash 16 \rfloor$, and $L$ replaced by $\lfloor L \slash ( k+2 )\rfloor$. We first verify the prerequisites of Lemma~\ref{lem:quantized_case_1} by noting that $\lfloor W\slash 16 \rfloor\geq \lfloor 960 \slash 16 \rfloor  = 60$, $\lfloor L\slash (k+2) \rfloor \geq \lfloor L\slash ( \lfloor\frac{L}{60} \rfloor  - 2 + 2\bigr) \rfloor \geq 60$, and
		\begin{align}
			kb \geq&\,  5  \frac{L}{k} \log(W)\label{eq:E4_f_20}\\
			\geq&\,   \frac{L}{k} \log(W) +   \frac{L}{k} \log(W) + 3 \frac{L}{\lfloor\frac{L}{60} \rfloor  - 2} \log(W) \label{eq:E4_f_21} \\
			\geq&\, \log (\lfloor L \slash (k+2) \rfloor) + \lfloor L \slash (k+2) \rfloor \log(\lfloor W\slash 16 \rfloor + 1)  + 3,\label{eq:E4_f_22}
		\end{align}
		where  \eqref{eq:E4_f_21} follows from $k \in \bigl\{ 2,\dots, \lfloor\frac{L}{60} \rfloor  - 2\bigr \}$, and in \eqref{eq:E4_f_22} we used $\frac{L}{k} \log(W) \geq \lfloor L \slash (k+2) \rfloor \geq \log ( \lfloor L \slash (k+2) \rfloor )$, $\frac{L}{k} \log(W) \geq \lfloor L \slash (k+2) \rfloor \log(\lfloor W\slash 16 \rfloor +1)$, and $\frac{L}{\lfloor\frac{L}{60} \rfloor  - 2} \log(W) \geq 1$. 
    Then, application of \eqref{eq:quantized_case_12} in Lemma~\ref{lem:quantized_case_1} with $\varepsilon = \frac{1}{50}$, yields
		\begin{align}
			&\,\log ( N( 1\slash 50,\mathcal{R}^1_{kb} ( d,\lfloor W \slash 16 \rfloor, \lfloor L \slash (k+2) \rfloor),L^p ( [0,1]^d ) ) )\label{eq:47somerhs_1}\\
			&\geq \, c_3 ( \lfloor W \slash 16 \rfloor )^2 \lfloor L \slash (k+2) \rfloor) \cdot \min \biggl\{  kb,  \log \biggl(\frac{( \lfloor W \slash 16 \rfloor + 1 )^{\lfloor L \slash ( k+2 ) \rfloor}}{1\slash 50}\biggr) \biggr\}, \label{eq:47somerhs_2}
		\end{align}
		with $c_3 \in \mathbb{R}_{+}$ an absolute constant.
  To lower-bound the right-hand-side of \eqref{eq:47somerhs_2}, we note that 
		\begin{equation}
		\label{eq:E4_substituting_1}
			\lfloor W \slash 16 \rfloor \geq \frac{1}{2} \cdot \frac{W}{16} = \frac{W}{32},
		\end{equation}
		\begin{equation}
		\label{eq:E4_substituting_2}
			\lfloor L \slash ( k+2) \rfloor \geq \lfloor L \slash (4k) \rfloor \geq  \frac{L}{8k},
		\end{equation}
		and
		\begin{align}
			\log \biggl(\frac{( \lfloor W \slash 16 \rfloor + 1 )^{\lfloor L \slash ( k+2 ) \rfloor}}{1\slash 50}\biggr) \geq&\,\, \lfloor L \slash ( k+2 ) \rfloor \log ( \lfloor W \slash 16 \rfloor + 1) \label{eq:47somerhs_3}\\
			\geq&\, \lfloor L \slash ( k+2 ) \rfloor \log ( W \slash 16 ) \label{eq:47somerhs_31}\\
			\geq &\, \frac{1}{16} \frac{L}{k -1 } \log (W)\label{eq:47somerhs_4}\\
			> &\, \frac{1}{80} (k - 1)b \label{eq:47somerhs_5} \\
			\geq&\, \frac{1}{160} k b, \label{eq:47somerhs_6}
		\end{align}
		where in \eqref{eq:47somerhs_4} we used $\lfloor L \slash ( k+2 ) \rfloor \geq \frac{1}{2} \cdot L \slash ( k+2 ) \geq  
  \frac{1}{8} \frac{L}{k-1}$ and $\log ( W \slash 16 ) \geq
  \frac{1}{2} \log(W)$, \eqref{eq:47somerhs_5} follows from \eqref{eq:E4_f_1}, and \eqref{eq:47somerhs_6} is thanks to $k \geq 2$. Then, using \eqref{eq:E4_substituting_1}, \eqref{eq:E4_substituting_2}, and \eqref{eq:47somerhs_3}-\eqref{eq:47somerhs_6} in \eqref{eq:47somerhs_1}-\eqref{eq:47somerhs_2}, yields
		\begin{align}
			&\,\log ( N( 1\slash 50,\mathcal{R}^1_{kb} ( d,\lfloor W \slash 16 \rfloor, \lfloor L \slash (k+2) \rfloor),L^p ( [0,1]^d ) ) )\\
			&\geq \, c_3 \biggl(\frac{W}{32}\biggr)^2 \biggl( \frac{L}{8k} \biggr) \min \biggl\{  kb,  \frac{1}{160}kb \biggr\} \\
			&=\, \frac{c_3}{32^2 \cdot 8 \cdot 160} W^2 L b.
		\end{align}
		As the covering number is non-decreasing in $\varepsilon$, we have shown that, for all $\varepsilon \in ( 0, \frac{1}{100} )$,
		\begin{equation}
		\label{eq:E4_substitute_lower_tobe}
			\log ( N( 2 \varepsilon,\mathcal{R}^1_{kb} ( d,\lfloor W \slash 16 \rfloor, \lfloor L \slash (k+2) \rfloor),L^p ( [0,1]^d ) ) ) \geq  \frac{c_3}{32^2 \cdot 8 \cdot 160} W^2 L b.
		\end{equation}
		We conclude the proof by putting \eqref{eq:E4_substitute_lower} and \eqref{eq:E4_substitute_lower_tobe} together and setting $c = \frac{c_3}{32^2 \cdot 8 \cdot 160} $. \qedhere

		\end{proof}

		We are now ready to proceed to the proof of the lower bound \eqref{eq:44_results_quantized}.
			Let $D := 960$ and $E := 2 \cdot 10^{5}$ and arbitrarily fix an $\varepsilon \in ( 0, \frac{1}{100} )$. By assumption, we hence have $W \geq D \geq 960$, $L \geq D \geq 960$, and 
			\begin{equation}
			\label{eq:8961_0}
				L(a+b) \geq E \,\log(W) = 2 \cdot 10^{5} \log(W).
			\end{equation}
			It follows from \eqref{eq:equivalent_different_precision} in Lemma~\ref{lem:equivalence_covering_number} that
			\begin{equation}  
			\label{eq:c_quantized_lower_1}
				N ( \varepsilon, \mathcal{R}_b^a ( d, W,L), L^p ( [0,1]^d )  ) = N \biggl( \frac{\varepsilon}{2^{( a-1 )L}}, R_{a+b-1}^1 ( d, W,L), {L^p ( [0,1]^d )} \biggl).
			\end{equation}
			We proceed by distinguishing two cases, the first one being
			\begin{equation}
			\label{eq:8961_1}
				( a+b -1 ) > \log (L) + L \log ( W+1 ) + 3.
			\end{equation}
			Then, it follows from \eqref{eq:quantized_case_12} in Lemma~\ref{lem:quantized_case_1} with $\varepsilon$ replaced by $\frac{\varepsilon}{2^{( a-1 )L}}$, $b$ replaced by $a + b -1$, and the prerequisite $( a+b -1 ) > \log (L) + L \log ( W+1 ) + 3$ satisfied thanks to \eqref{eq:8961_1}, that, for $\frac{\varepsilon}{2^{( a-1 )L}} \in ( 0, \frac{1}{100})$, 
			\begin{align}
				&\,\log \biggl (N\biggl( \frac{\varepsilon}{2^{( a-1 )L}},\mathcal{R}^1_{a+b -1 } ( d,W,L ),L^p ( [0,1]^d ) \biggr)\biggr)\label{eq:c_quantized_lower_2}\\
				&\geq \,c_1   W^2 L \cdot \min \biggl\{( a+b -1 ),\log \biggl(\frac{(W+1)^L 2^{( a-1 )L}}{\varepsilon}\biggr) \biggr\}, \label{eq:c_quantized_lower_3}
			\end{align}
			with $c_1 \in \mathbb{R}_{+}$ an absolute constant. The second case is
			\begin{equation}
			\label{eq:8961_2}
			 	( a+ b -1  )\leq \log (L) + L \log ( W+1 ) + 3.
			\end{equation} 
			Using $ a + b - 1 \geq \frac{1}{2} ( a + b )$ in  \eqref{eq:8961_0} and $\log (L) + L \log ( W+1 ) + 3 < L + L \log ( W^2 ) + L < 4L \log(W)$  in  \eqref{eq:8961_2}, yields 
			\begin{equation}
			\label{eq:8961_21}
			 	72000 \frac{\log (W)}{L}  < 100000 \frac{\log (W)}{L} \leq  ( a+ b -1  ) < 4 L \log (W).
			\end{equation} 
			Now, application of Lemma~\ref{lem:quantized_case_2} with $\varepsilon$ replaced by $\frac{\varepsilon}{2^{( a-1 )L}}$, $b$ replaced by $a + b -1$,  the prerequisite $72000 \frac{\log (W)}{L} < a + b -1 \leq 4 L \log (W) $ satisfied thanks to \eqref{eq:8961_21}, yields, for $\frac{\varepsilon}{2^{( a-1 )L}} \in ( 0, \frac{1}{100})$,
			\begin{align}
				&\, \log \biggl ( N\biggl( \frac{\varepsilon}{2^{( a-1 )L}},\mathcal{R}^1_{a+b -1 } ( d,W,L ),L^p ( [0,1]^d ) \biggr) \biggr ) \label{eq:c_quantized_lower_40}\\
				&\geq\, c_2 W^2 L ( a +  b - 1 )\label{eq:c_quantized_lower_4}\\
				&\geq\, c_2 W^2 L \cdot \min \biggl\{( a+b -1 ),\log \biggl(\frac{(W+1)^L 2^{( a-1 )L}}{\varepsilon}\biggr) \biggr\},\label{eq:c_quantized_lower_5}
			\end{align}
			with $c_2 \in \mathbb{R}_{+}$ an absolute constant. Combining \eqref{eq:c_quantized_lower_2}-\eqref{eq:c_quantized_lower_3} and \eqref{eq:c_quantized_lower_40}-\eqref{eq:c_quantized_lower_5}, yields, for all $a,b \in \mathbb{N}$,
			\begin{equation}
			\label{eq:61_final_1}
			\begin{aligned}
				&\,\log \biggl ( N\biggl( \frac{\varepsilon}{2^{( a-1 )L}},\mathcal{R}^1_{a+b -1 } ( d,W,L ),L^p ( [0,1]^d ) \biggr) \biggr) \\
				&\geq\, \min \{ c_1,c_2 \} W^2 L \cdot  \min \biggl\{( a+b -1 ),\log \biggl(\frac{(W+1)^L 2^{( a-1 )L} }{\varepsilon}\biggr) \biggr\}.
			\end{aligned}
			\end{equation}
			Using 
			\begin{align}
				( a+b - 1) \geq&\, \frac{1}{2} ( a+b ),\\
				\log \biggl(\frac{(W+1)^L 2^{( a-1 )L}}{\varepsilon}\biggr) = &\,\frac{1}{2} \log \biggl(\frac{(W+1)^{2L} 2^{2( a-1 )L}}{\varepsilon^2}\biggr) \label{eq:c_quantized_lower_6} \\
				\geq  &\,\frac{1}{2} \log \biggl(\frac{(W+1)^{2L} 2^{(a- 1)L}}{\varepsilon^2}\biggr) \label{eq:c_quantized_lower_6} \\
				\geq&\, \frac{1}{2} \log \biggl(\frac{(W+1)^{L} 2^{aL} }{\varepsilon}\biggr),  \label{eq:c_quantized_lower_7}
			\end{align}
			together with \eqref{eq:c_quantized_lower_1} in \eqref{eq:61_final_1},  yields 
			\begin{equation*}
				\log \bigl( N ( \varepsilon, \mathcal{R}_b^a ( d, W,L), L^p ( [0,1]^d )  ) \bigr ) \geq \frac{\min \{ c_1, c_2 \}}{2}\, W^2 L \cdot \min \biggl\{ ( a+b ), \log \biggl(\frac{(W+1)^{L} 2^{aL} }{\varepsilon}\biggr) \biggr\}.
			\end{equation*}
			The proof is concluded upon setting $c = \frac{\min \{ c_1, c_2\}}{2} $.

		\subsection{Proof of Lemma~\ref{prop:tradeoff}}
		\label{sub:proof_of_lemma_prop:tradeoff}

		We first state the following technical lemma. 
		\begin{lemma}
		\cite[Proposition F.1]{FirstDraft2022}
		\label{thm:general_representation}
		Let $d,W,L \in \mathbb{N}$, and let $\mathbb{A} \subseteq \mathbb{R}$ be a finite set satisfying $\{ -1,0,1 \} \subseteq \mathbb{A}$. Then, for every $k \in \mathbb{N}$ and all $u,v \in \mathbb{A} \cap \mathbb{R}_{\geq 0}$, it holds that
		\begin{equation*}
			\mathcal{R}_{\mathcal{T}_1 ( \mathbb{A},u,v,k)} (d,  W,  L  ) \subseteq  \mathcal{R}_\mathbb{A} ( d, 16W,(k+3) L )
		\end{equation*}
		with
		\begin{equation}
			\label{eq:def_Auvk}
			\mathcal{T}_1 ( \mathbb{A},u,v,k) := \biggl\{ \sum_{i = 0}^{k} (u^i \alpha_i + v^i \beta_i) : | \alpha_i |,| \beta_i | \in \mathbb{A} , i = 0,\dots, k \biggr\}.	
		\end{equation} 
	\end{lemma}
	
	We are now ready to prove Lemma~\ref{prop:tradeoff}. For $k =1$, \eqref{eq:tradeoff} is trivially satisfied.  For $k \geq 2$, we note that
		\begin{align}
			&\,\mathcal{T}_1 (\mathbb{Q}^a_b,2^{-b},2^a,k-1) \label{eqline:def_A_binary_10} \\
			&=\, \Biggl\{ \sum_{i = 0}^{k-1} (2^{-bi} \alpha_i + 2^{ai} \beta_i) : | \alpha_i |,| \beta_i | \in \mathbb{Q}_{b}^a , i = 0,\dots, k-1 \Biggr\} \\ 
			&=\, \Biggl\{ \sum_{i = 0}^{k-1} (2^{-bi} \alpha_i + 2^{ai} \beta_i) : \alpha_i , \beta_i \in \mathbb{Q}_{b}^a , i = 0,\dots, k-1 \Biggr\} \label{eqline:def_A_binary_1}\\ 
			&\supseteq\, \Biggl\{ \pm \sum_{i = -kb}^{ka} 2^{i} c_i: c_i \in \{ 0,1 \} \Biggr\}\label{eqline:def_A_binary_2}\\
			&= \, \mathbb{Q}_{kb}^{ka}, \label{eqline:def_A_binary_19}
		\end{align}
		where in \eqref{eqline:def_A_binary_2} we used $\mathbb{Q}^{a}_b =  \{ \pm \sum_{i = -b}^a \theta_i 2^{i}: \theta_i \in \{ 0,1 \} \}$. Thanks to \eqref{eqline:def_A_binary_10}-\eqref{eqline:def_A_binary_19}, we have 
		\begin{equation}
			\label{eq:proof_tradeoff_inclusion_1}
			\mathcal{R}^{ka}_{kb} (d, W,  L) \subseteq \mathcal{R}_{\mathcal{T}_1 (\mathbb{Q}^a_b,2^{-b},2^a,k-1)} (d, W,  L  ).
		\end{equation}
		Application of Lemma~\ref{thm:general_representation} with $u = 2^{-b}$, $v = 2^{a}$, $\mathbb{A} = \mathbb{Q}_b^a$, and $k$ replaced by $k-1$, yields
		\begin{equation}
			\label{eq:proof_tradeoff_inclusion_2}
			\begin{aligned}
			\mathcal{R}_{\mathcal{T}_1 (\mathbb{Q}_b^a,2^{-b},2^a,k-1)} ( d, W,  L  ) 
			\subseteq&\,   \mathcal{R}_b^a (d, 16W,(k+2)L).
			\end{aligned}
		\end{equation}
		The proof is finalized by combining \eqref{eq:proof_tradeoff_inclusion_1} and \eqref{eq:proof_tradeoff_inclusion_2} to obtain  \eqref{eq:tradeoff}.

\section{Auxiliary results}
\label{sec:auxiliary_results}

	\subsection{Relation Between Covering Number and Packing Number}
	\label{sub:relation_between_covering_number_and_packing_number}
		The following lemma on covering and packing numbers is frequently used throughout this paper.

		\begin{lemma}
		\label{lem:equivalence_covering_packing}
		\label{lem:covering_packing_inclusion}
		Let $( \mathcal{X}, \delta )$ be a metric space and $\varepsilon \in \mathbb{R}_+$.  It holds that
		\begin{equation}
			\label{eq:equivalence_covering_packing}
			M ( 2 \varepsilon, \mathcal{X}, \delta ) \leq N ( \varepsilon, \mathcal{X}, \delta ) \leq M ( \varepsilon, \mathcal{X}, \delta ).
		\end{equation}
		Let $\mathcal{Y} \subseteq \mathcal{X}$. We have, 
		\begin{align}
			M ( \varepsilon, \mathcal{X}, \delta     ) \geq&\, M ( \varepsilon, \mathcal{Y}, \delta   ), \label{eq:covering_packing_inclusion_1} \\
			N ( \varepsilon, \mathcal{X}, \delta     ) \geq&\, N ( 2\varepsilon, \mathcal{Y}, \delta   ), \label{eq:covering_packing_inclusion_2}
		\end{align}
		\begin{proof}
			Relation \eqref{eq:equivalence_covering_packing} is \cite[Lemma 5.5]{wainwright2019high}. To prove \eqref{eq:covering_packing_inclusion_1}, we simply note that every $\varepsilon$-packing of $\mathcal{Y}$ is also an $\varepsilon$-packing of $\mathcal{X}$, and hence
			\begin{equation}
			\label{eq:a1_mm}
				M ( \varepsilon, \mathcal{X}, \delta    ) \geq M ( \varepsilon, \mathcal{Y}, \delta    ).
			\end{equation} 
			Then, \eqref{eq:covering_packing_inclusion_2} follows from $N ( \varepsilon, \mathcal{X}, \delta    ) \geq M ( 2\varepsilon, \mathcal{X}, \delta    ) \geq M ( 2\varepsilon, \mathcal{Y}, \delta    ) \geq N ( 2 \varepsilon, \mathcal{X}, \delta    )$, where the first and the last inequalities are by \eqref{eq:equivalence_covering_packing} and the second by \eqref{eq:a1_mm}.
		\end{proof}
	\end{lemma}

	\subsection{Augmenting Networks}
	\label{sub:extension_of_the_networks}

		This section is concerned with a technical lemma,  which shows how a ReLU network of a given depth can be augmented to a deeper network while retaining its input-output relation.

		\begin{lemma}
		\label{lem:extension}
		Let $ d,W,L \in \mathbb{N}$, with $W \geq 2$, and let $\mathbb{A} \subseteq \mathbb{R}$, with $\{ -1,0, 1 \} \subseteq \mathbb{A}$ and $\mathbb{A} = -\mathbb{A}$. For every $f \in \mathcal{R}_\mathbb{A} ( d,W,L)$, there exists a network configuration $\Phi \in \mathcal{N}_\mathbb{A} ( d,W,L)$ such that $\mathcal{L} (\Phi) = L$ and $\mathcal{R} (\Phi) = f$. 
		\end{lemma}

		A special case of Lemma~\ref{lem:extension}, namely $\mathbb{A} = \mathbb{R}$, is documented in \cite[Lemma H.2]{FirstDraft2022}. The proof of Lemma~\ref{lem:extension} is almost identical to that of \cite[Lemma H.2]{FirstDraft2022}, but will be provided for completeness.
  		
		\begin{proof}
            [Proof of Lemma~\ref{lem:extension}]
   By definition, there exists a network configuration $\widetilde{\Phi} = ( ( \widetilde{A}_\ell,\widetilde{b}_\ell ))_{\ell = 1}^{\widetilde{L}} \in \mathcal{N}_{\mathbb{A}} ( d,W,L)$, with $\widetilde{L} \leq L$, 
			such that $R ( \widetilde{\Phi} ) = f$. If $\widetilde{L} = L$, setting $\Phi = \widetilde{\Phi}$,
   we have $\mathcal{L} (\Phi) = L$ and $\mathcal{R} (\Phi) = f$.  For $L > \widetilde{L}$, we let ${\Phi} := ( ( A_\ell,b_\ell ))_{\ell = 1}^L$,
			with 
			\begin{equation}
				\label{eq:extension}
				(A_\ell, b_\ell)  := (\widetilde{A}_\ell, \widetilde{b}_\ell),\quad \text{for } 1 \leq \ell < \widetilde{L},\footnote{Here and in what follows, we use the convention that if there does not exist an $\ell$, in this case, satisfying the constraint, the assignment is skipped; in the present case, this would apply if $\widetilde{L} = 1$.
    }
			\end{equation}
			$A_{\widetilde{L}} := \begin{pmatrix} \widetilde{A}_{\widetilde{L}}\\-\widetilde{A}_{\widetilde{L}} \end{pmatrix}$, $b_{\widetilde{L}} := \begin{pmatrix} 	\widetilde{b}_{\widetilde{L}}\\ -\widetilde{b}_{\widetilde{L}} \end{pmatrix},$
			$A_\ell := I_{2}$, $b_\ell := 0_{2}$, for $\ell$ such that $\widetilde{L} < \ell < L  $, and $A_{L} := (\begin{matrix}
				1 & -1
			\end{matrix})$, $b_{L} := 0$. Invoking the assumptions $\{ -1,0, 1 \} \subseteq \mathbb{A}$ and $\mathbb{A} = -\mathbb{A}$, this yields ${\Phi} \in \mathcal{N}_\mathbb{A} ( d,W,L)$, with $\mathcal{L} ( \Phi ) = L$. We now note that   
			\begin{align}
			 	\affine ( A_{L}, b_{L} ) \circ \rho \circ \cdots \circ \rho\circ  \affine ( A_{\widetilde{L}}, b_{\widetilde{L}} ) =&\, \affine ( A_{L}, b_{L} ) \circ \rho\circ  \affine ( A_{\widetilde{L}}, b_{\widetilde{L}} ) \label{eqline:rho_rho_0}\\
			 	=&\, \affine ( \widetilde{A}_{\widetilde{L}}, \widetilde{b}_{\widetilde{L}} ) \label{eqline:rho_rho},
			 \end{align} 
			 where in \eqref{eqline:rho_rho_0} we used  $\rho \circ \affine ( A_\ell, b_\ell ) = \rho \circ \affine ( I_{2}, 0_{2} ) = \rho$, for $\widetilde{L} < \ell < L$, and $\rho\circ \rho = \rho$, and \eqref{eqline:rho_rho} is by $\left(\affine ( A_{L}, b_{L} ) \circ\,  \rho\, \circ  \affine ( A_{\widetilde{L}}, b_{\widetilde{L}} )\right) (x) = \rho ( \widetilde{A}_{\widetilde{L}} x + \widetilde{b}_{\widetilde{L}} ) - \rho ( - \widetilde{A}_{\widetilde{L}} x  - \widetilde{b}_{\widetilde{L}} ) = \widetilde{A}_{\widetilde{L}} x + \widetilde{b}_{\widetilde{L}} =  \affine ( \widetilde{A}_{\widetilde{L}}, \widetilde{b}_{\widetilde{L}} )(x)$, for $x \in \mathbb{R}_{d'}$, with $d'$ denoting the number of columns of $\widetilde{A}_{\widetilde{L}}$. Combining \eqref{eqline:rho_rho_0}-\eqref{eqline:rho_rho} and \eqref{eq:extension}, we see that $ R( {\Phi} ) =   R ( \widetilde{\Phi} ) =  f$. \qedhere
    
		\end{proof}

        \subsection{Existence of the Empirical Risk Minimizer in \eqref{eq:assumption_H1_minimum}}
        \label{subs:existence_empirical_risk}
            Arbitrarily fix a sample $(x_i,y_i)_{i=1}^n \in ([0,1]\times \mathbb{R})^n$. The existence of the empirical risk minimizer, for this fixed sample, is equivalent to the existence of a network configuration $\Phi \in \mathcal{N} ( 1, \lceil D + 1 \rceil, L(n), 1)$ whose 
            associated truncated realization $\Othres_1(R(\Phi))$ minimizes the empirical risk.
            Noting that the set $\mathcal{N} ( 1, \lceil D + 1 \rceil, L(n), 1)$ can be written as a finite disjoint union of network configurations, each element in the union corresponding to a given network architecture, we only have to show the existence of a minimizer over each given element in this union. Arbitrarily fix an element in the union with associated architecture $(N_0, \dots, N_\ell)$ and let $\mathcal{N}_{N_0,\dots, N_\ell}$ be the corresponding set of network configurations. Next, note that $\mathcal{N}_{N_0,\dots, N_\ell} = [-1,1]^{\sum_{i = 1}^{\ell } ( N_{i} N_{i-1} + {N}_i)} $ is a compact set and, by Lemma~\ref{lem:quantization}, the mapping
            $\Phi \in \mathcal{N}_{N_0,\dots, N_\ell} \mapsto    \frac{1}{n} \sum_{i = 1}^n ( \Othres_1(R(\Phi) (x_i)) - y_i )^2 $ is continuous. As continuous functions on compact sets attain minima, there exists a $\Phi \in \mathcal{N}_{N_0,\dots, N_\ell}$ that minimizes the empirical risk within the set $\mathcal{N}_{N_0,\dots, N_\ell}$, as was to be shown. 
            The argument is concluded by noting that the choice of the sample $(x_i,y_i)_{i=1}^n \in ([0,1]\times \mathbb{R})^n$ was arbitrary and hence there exists an empirical risk minimizer for each sample $(x_i,y_i)_{i=1}^n \in ([0,1]\times \mathbb{R})^n$.

\bibliography{my_bib}

@article{achour2022general,
  title={A general approximation lower bound in ${L}^p$ norm, with applications to feed-forward neural networks},
  author={Achour, El Mehdi and Foucault, Armand and Gerchinovitz, S{\'e}bastien and Malgouyres, Fran{\c{c}}ois},
  journal={Advances in Neural Information Processing Systems},
  volume={35},
  pages={22396--22408},
  year={2022}
}

@article{mendelson2002rademacher,
  title={Rademacher averages and phase transitions in Glivenko-Cantelli classes},
  author={Mendelson, Shahar},
  journal={IEEE transactions on Information Theory},
  volume={48},
  number={1},
  pages={251--263},
  year={2002},
  publisher={IEEE}
}

@Article{Caragea2022,
  author   = {Caragea, Andrei and Lee, Dae Gwan and Maly, Johannes and Pfander, G\"{o}tz and Voigtlaender, Felix},
  journal  = {SIAM Journal on Mathematics of Data Science},
  title    = {Quantitative Approximation Results for Complex-Valued Neural Networks},
  year     = {2022},
  number   = {2},
  pages    = {553-580},
  volume   = {4},
  doi      = {10.1137/21M1429540},
  eprint   = { https://doi.org/10.1137/21M1429540 },
}

@Article{Caragea2023,
  author    = {Andrei Caragea and Philipp Petersen and Felix Voigtlaender},
  journal   = {The Annals of Applied Probability},
  title     = {{Neural network approximation and estimation of classifiers with classification boundary in a Barron class}},
  year      = {2023},
  number    = {4},
  pages     = {3039 -- 3079},
  volume    = {33},
  doi       = {10.1214/22-AAP1884},
  publisher = {Institute of Mathematical Statistics},
}

@Article{Cybenko1989,
  author    = {G. Cybenko},
  journal   = {Mathematics of Control, Signals, and Systems},
  title     = {Approximation by superpositions of a sigmoidal function},
  year      = {1989},
  month     = {Dec},
  number    = {4},
  pages     = {303--314},
  volume    = {2},
  doi       = {10.1007/bf02551274},
  publisher = {Springer Science and Business Media {LLC}},
}

@Article{Barron1993,
  author    = {A.R. Barron},
  journal   = {{IEEE} Transactions on Information Theory},
  title     = {Universal approximation bounds for superpositions of a sigmoidal function},
  year      = {1993},
  month     = {May},
  number    = {3},
  pages     = {930--945},
  volume    = {39},
  doi       = {10.1109/18.256500},
  publisher = {Institute of Electrical and Electronics Engineers ({IEEE})},
}

@Article{Hornik1989,
  author    = {Kurt Hornik and Maxwell Stinchcombe and Halbert White},
  title     = {Multilayer feedforward networks are universal approximators},
  journal   = {Neural Networks},
  year      = {1989},
  volume    = {2},
  number    = {5},
  pages     = {359--366},
  month     = {Jan},
  doi       = {10.1016/0893-6080(89)90020-8},
  publisher = {Elsevier {BV}},
}

@Article{Funahashi1989,
  author    = {Ken-Ichi Funahashi},
  title     = {On the approximate realization of continuous mappings by neural networks},
  journal   = {Neural Networks},
  year      = {1989},
  volume    = {2},
  number    = {3},
  pages     = {183--192},
  month     = {Jan},
  doi       = {10.1016/0893-6080(89)90003-8},
  publisher = {Elsevier {BV}},
}

@article{elbrachter2018dnn,
  title={{DNN} expression rate analysis of high-dimensional {PDEs}: {A}pplication to option pricing},
  author={Elbr{\"a}chter, Dennis and Grohs, Philipp and Jentzen, Arnulf and Schwab, Christoph},
  journal={Constructive Approximation},
  volume={55},
  number={1},
  pages={3--71},
  year={2022},
  publisher={Springer}
}

@Article{deep-approx-18,
  author   = {Bölcskei, Helmut and Grohs, Philipp and Kutyniok, Gitta and Petersen, Philipp},
  title    = {Optimal approximation with sparsely connected deep neural networks},
  journal  = {SIAM Journal on Mathematics of Data Science},
  year     = {2019},
  volume   = {1},
  number   = {1},
  pages    = {8--45},
}

@Article{deep-it-2019,
  author   = {Elbrächter, Dennis and Perekrestenko, Dmytro and Grohs, Philipp and Bölcskei, Helmut},
  journal  = {IEEE Transactions on Information Theory},
  title    = {Deep neural network approximation theory},
  year     = {2021},
  month    = may,
  number   = {5},
  pages    = {2581--2623},
  volume   = {67},
}

@Article{YAROTSKY2017103,
  author   = {Dmitry Yarotsky},
  journal  = {Neural Networks},
  title    = {{Error bounds for approximations with deep ReLU networks}},
  year     = {2017},
  issn     = {0893-6080},
  pages    = {103 - 114},
  volume   = {94},
  doi      = {https://doi.org/10.1016/j.neunet.2017.07.002},
}

@Article{Schmidt-Hieber2017,
  author    = {Johannes Schmidt-Hieber},
  journal   = {The Annals of Statistics},
  title     = {{Nonparametric regression using deep neural networks with ReLU activation function}},
  year      = {2020},
  number    = {4},
  pages     = {1875 -- 1897},
  volume    = {48},
  doi       = {10.1214/19-AOS1875},
  publisher = {Institute of Mathematical Statistics},
}

@Article{bartlett2019nearly,
  author  = {Bartlett, Peter L and Harvey, Nick and Liaw, Christopher and Mehrabian, Abbas},
  journal = {Journal of Machine Learning Research},
  title   = {Nearly-tight {VC}-dimension and Pseudodimension Bounds for Piecewise Linear Neural Networks.},
  year    = {2019},
  number  = {63},
  pages   = {1--17},
  volume  = {20},
}

@InProceedings{yarotsky2019phase,
  author    = {Yarotsky, Dmitry and Zhevnerchuk, Anton},
  booktitle = {Advances in Neural Information Processing Systems},
  title     = {The phase diagram of approximation rates for deep neural networks},
  year      = {2020},
  pages     = {13005--13015},
  volume    = {33},
}

@Article{Kohler2019,
  author    = {Michael Kohler and Sophie Langer},
  journal   = {The Annals of Statistics},
  title     = {{On the rate of convergence of fully connected deep neural network regression estimates}},
  year      = {2021},
  number    = {4},
  pages     = {2231 -- 2249},
  volume    = {49},
  doi       = {10.1214/20-AOS2034},
  publisher = {Institute of Mathematical Statistics},
}

@Book{wainwright2019high,
  author    = {Wainwright, Martin J},
  publisher = {Cambridge University Press},
  title     = {{High-Dimensional Statistics: A Non-Asymptotic Viewpoint}},
  year      = {2019},
  edition   = {2nd},
  volume    = {48},
}

@Article{nakada2020adaptive,
  author  = {Nakada, Ryumei and Imaizumi, Masaaki},
  title   = {Adaptive Approximation and Generalization of Deep Neural Network with Intrinsic Dimensionality},
  journal = {Journal of Machine Learning Research},
  year    = {2020},
  volume  = {21},
  number  = {174},
  pages   = {1--38},
}

@Article{shen2019nonlinear,
  author    = {Shen, Zuowei and Yang, Haizhao and Zhang, Shijun},
  title     = {Nonlinear approximation via compositions},
  journal   = {Neural Networks},
  year      = {2019},
  volume    = {119},
  pages     = {74--84},
  publisher = {Elsevier},
}

@Article{Chen2019,
  author   = {Chen, Minshuo and Jiang, Haoming and Liao, Wenjing and Zhao, Tuo},
  journal  = {Information and Inference: A Journal of the IMA},
  title    = {{Nonparametric regression on low-dimensional manifolds using deep ReLU networks: Function approximation and statistical recovery}},
  year     = {2022},
  issn     = {2049-8772},
  month    = mar,
  number   = {4},
  pages    = {1203-1253},
  volume   = {11},
  doi      = {10.1093/imaiai/iaac001},
  eprint   = {https://academic.oup.com/imaiai/article-pdf/11/4/1203/48362173/iaac001.pdf},
}

@InProceedings{telgarsky2016benefits,
  author    = {Telgarsky, Matus},
  booktitle = {29th Annual Conference on Learning Theory},
  title     = {Benefits of depth in neural networks},
  year      = {2016},
  address   = {Columbia University, New York, USA},
  month     = {23--26 June},
  pages     = {1517--1539},
  series    = {Proceedings of Machine Learning Research},
  volume    = {49},
}

@article{GUHRING2021107,
title = {Approximation rates for neural networks with encodable weights in smoothness spaces},
journal = {Neural Networks},
volume = {134},
pages = {107-130},
year = {2021},
issn = {0893-6080},
author = {Ingo Gühring and Mones Raslan}
}

@Article{PETERSEN2018296,
  author   = {Philipp Petersen and Felix Voigtlaender},
  journal  = {Neural Networks},
  title    = {{Optimal approximation of piecewise smooth functions using deep ReLU neural networks}},
  year     = {2018},
  issn     = {0893-6080},
  pages    = {296-330},
  volume   = {108},
  doi      = {https://doi.org/10.1016/j.neunet.2018.08.019},
}

@Article{FirstDraft2022,
  author  = {Ou, Weigutian and Schenkel, Philipp and B{\"o}lcskei, Helmut},
  journal = {arXiv preprint arXiv:2405.01952},
  title   = {{Three quantization regimes for ReLU networks}},
  year    = {2024},
}

@Article{shen2019deep,
  author  = {Zuowei Shen and Haizhao Yang and Shijun Zhang},
  journal = {Communications in Computational Physics},
  title   = {Deep Network Approximation Characterized by Number of Neurons},
  year    = {2020},
  issn    = {1991-7120},
  number  = {5},
  pages   = {1768--1811},
  volume  = {28},
  doi     = {10.4208/cicp.OA-2020-0149},
}

@Book{Gyoerfi2002,
  author    = {Gy{\"o}rfi, L{\'a}szl{\'o} and Kohler, Michael and Krzy{ż}ak, Adam and Walk, Harro},
  publisher = {Springer},
  title     = {A distribution-free theory of nonparametric regression},
  series    = {Springer Series in Statistics},
  year      = {2002},
  volume    = {1},
}

@Article{VAPNIK1971,
  author  = {Vapnik, VN and Chervonenkis, AY},
  journal = {Theory Probab. Appl.},
  title   = {On the uniform convergence of relative frequencies of events to their probabilities},
  year    = {1971},
  number  = {2},
  pages   = {264--280},
  volume  = {16},
}

@Book{Sh1980,
  author  = {Sh, Birman M and Solomjak, MZ},
  title   = {{Quantitative analysis in Sobolev imbedding theorems and applications to spectral theory}},
  series    = {
American Mathematical Society translations},
  publisher = {American Mathematical Society},
  year    = {1980},
}

@article{Donoho1993,
title = {Unconditional Bases Are Optimal Bases for Data Compression and for Statistical Estimation},
journal = {Applied and Computational Harmonic Analysis},
volume = {1},
number = {1},
pages = {100-115},
year = {1993},
issn = {1063-5203},
doi = {https://doi.org/10.1006/acha.1993.1008},
author = {David L. Donoho}
}

@article{Donoho1996,
title = {Unconditional Bases and Bit-Level Compression},
journal = {Applied and Computational Harmonic Analysis},
volume = {3},
number = {4},
pages = {388-392},
year = {1996},
issn = {1063-5203},
doi = {https://doi.org/10.1006/acha.1996.0032},
author = {David L. Donoho}
}

@InProceedings{Vardi2022,
  author    = {Vardi, Gal and Yehudai, Gilad and Shamir, Ohad},
  booktitle = {Proc. of Thirty Fifth Conference on Learning Theory},
  title     = {{Width is less important than depth in ReLU neural networks}},
  year      = {2022},
  month     = {July},
  pages     = {1249--1281},
  publisher = {PMLR},
  series    = {Proceedings of Machine Learning Research},
  volume    = {178},
}

@Article{Stone1982,
  author    = {Charles J. Stone},
  journal   = {The Annals of Statistics},
  title     = {Optimal Global Rates of Convergence for Nonparametric Regression},
  year      = {1982},
  month     = {Dec},
  pages = {1040 -- 1053},
  number    = {4},
  volume    = {10},
  doi       = {10.1214/aos/1176345969},
  publisher = {Institute of Mathematical Statistics},
}

@Article{Janowsky1989,
  author    = {Steven A. Janowsky},
  journal   = {Physical Review A},
  title     = {Pruning versus clipping in neural networks},
  year      = {1989},
  month     = {June},
  number    = {12},
  pages     = {6600--6603},
  volume    = {39},
  doi       = {10.1103/physreva.39.6600},
  publisher = {American Physical Society ({APS})},
}

@Article{Blalock2020,
  author  = {Blalock, Davis and Gonzalez Ortiz, Jose Javier and Frankle, Jonathan and Guttag, John},
  journal = {Proceedings of Machine Learning and Systems},
  title   = {What is the state of neural network pruning?},
  year    = {2020},
  pages   = {129--146},
  volume  = {2},
}

@article{Maly2022,
title = {A simple approach for quantizing neural networks},
journal = {Applied and Computational Harmonic Analysis},
volume = {66},
pages = {138-150},
year = {2023},
issn = {1063-5203},
doi = {https://doi.org/10.1016/j.acha.2023.04.004},
author = {Johannes Maly and Rayan Saab}
}

@Article{Mendelson2003,
  author    = {Mendelson, Shahar and Vershynin, Roman},
  journal   = {Inventiones Mathematicae},
  title     = {Entropy and the combinatorial dimension},
  year      = {2003},
  number    = {1},
  pages     = {37--55},
  volume    = {152},
  publisher = {Springer},
}

@Book{Anthony1999,
  author    = {Anthony, Martin and Bartlett, Peter L.},
  publisher = {Cambridge University Press},
  title     = {{Neural Network Learning: Theoretical Foundations}},
  year      = {1999},
  volume    = {9},
  place     = {Cambridge},
}

@Article{Yang1999,
  author    = {Yang, Yuhong and Barron, Andrew},
  journal   = {Annals of Statistics},
  title     = {Information-theoretic determination of minimax rates of convergence},
  volume = {27},
  number = {4},
  year      = {1999},
  pages     = {1564--1599},
  publisher = {JSTOR},
}

@Book{Huyen2022,
  author    = {Huyen, Chip},
  publisher = {O'Reilly Media, Inc.},
  title     = {Designing Machine Learning Systems},
  year      = {2022},
}

@Article{Bartlett2017,
  author  = {Bartlett, Peter L and Foster, Dylan J and Telgarsky, Matus J},
  journal = {Advances in Neural Information Processing Systems},
  title   = {Spectrally-normalized margin bounds for neural networks},
  year    = {2017},
  volume  = {30},
}

@Article{Gou2021,
  author    = {Gou, Jianping and Yu, Baosheng and Maybank, Stephen J and Tao, Dacheng},
  journal   = {International Journal of Computer Vision},
  title     = {Knowledge distillation: A survey},
  year      = {2021},
  number    = {6},
  pages     = {1789--1819},
  volume    = {129},
  publisher = {Springer},
}

@incollection{Grohs2015,
author="Grohs, Philipp",
title="Optimally Sparse Data Representations",
booktitle="Harmonic and Applied Analysis: From Groups to Signals",
year="2015",
publisher="Springer",
pages="199--248",
}

\end{document}